\documentclass[Afour,sageh,times]{sagej}

\usepackage{moreverb,url}
\usepackage{amsmath,amsfonts}
\usepackage{amsthm}            
\usepackage{siunitx}  
\usepackage{tikz}
\usepackage{makecell}
\usepackage[super]{nth}
\usepackage{subfig}
\usepackage{algorithm}
\usepackage[noend]{algpseudocode}
\usepackage{pgfplots}
\usepackage{multirow}
\usepackage[flushleft]{threeparttable}
\usepackage{sectsty}
\usepgfplotslibrary{groupplots}
\usepackage{enumerate}
\usepackage{csquotes}
\newif\ifshowmarkup
\showmarkupfalse   

\newcommand{\revise}[1]{\ifshowmarkup\textcolor{blue}{#1}\else#1\fi}
\newcommand{\rrevise}[1]{\ifshowmarkup\textcolor{green}{#1}\else#1\fi}

\usepackage[normalem]{ulem} 
\usepackage{xcolor}

\newif\ifrevision
\revisionfalse    

\ifrevision

  \newcommand{\added}[1]{\textcolor{blue}{#1}}
  \newcommand{\deleted}[1]{\textcolor{red}{\sout{#1}}}
  \newcommand{\ccomment}[1]{\textcolor{magenta}{\textbf{[#1]}}}

\else

  \newcommand{\added}[1]{#1}
  \newcommand{\deleted}[1]{}
  \newcommand{\ccomment}[1]{}

\fi

\theoremstyle{definition}
\newtheorem{definition}{Definition}
\newtheorem{theorem}{Theorem}
\newtheorem{proposition}{Proposition} 
\newtheorem{remark}{Remark}

\newcommand\norm[1]{\left\lVert#1\right\rVert}

\usetikzlibrary{positioning,fit,calc}
\tikzset{
block/.style={
draw,
thick,
text width=1.1cm,minimum height=1.5cm,align=center},
line/.style={-latex},
font={\fontsize{8pt}{12}\selectfont}
}
\usepackage[colorlinks,bookmarksopen,bookmarksnumbered,citecolor=red,urlcolor=red]{hyperref}

\newcommand\BibTeX{{\rmfamily B\kern-.05em \textsc{i\kern-.025em b}\kern-.08em
T\kern-.1667em\lower.7ex\hbox{E}\kern-.125emX}}

\def\volumeyear{2016}

\begin{document}

\runninghead{Phiquepal and Toussaint}

\title{Optimizing Trajectory-Trees in Belief~Space: An Application from Model Predictive Control to Task and Motion Planning}

\author{Camille Phiquepal\affilnum{1} and Marc Toussaint\affilnum{1}}

\affiliation{
\affilnum{1}Learning and Intelligent Systems Lab (LIS), TU Berlin, Germany}

\corrauth{Camille Phiquepal, Learning and Intelligent Systems Lab (LIS), TU Berlin, Germany.}

\email{camille.phiquepal@gmail.com}

\begin{abstract}
This paper explores the benefits of computing arborescent trajectories (trajectory-trees) instead of commonly used sequential trajectories for partially observable robotic planning problems. In such environments, a robot infers knowledge from observations, and the optimal course of action depends on these observations. \revise{Trajectory-trees, optimized in belief space, naturally capture this dependency by branching where the belief state is expected to evolve into multiple distinct scenarios, such as upon receiving an observation. Unlike sequential trajectories, which model a single forward evolution of the system, trajectory-trees capture multiple possible contingencies.}
First, we focus on Model Predictive Control (MPC) and demonstrate the benefits of planning tree-like trajectories. We formulate the control problem as the optimization of a tree with a single  branching (PO-MPC). This improves performance by reducing control costs through more informed planning. To satisfy the real-time constraints of MPC, we develop an optimization algorithm called Distributed Augmented Lagrangian (D-AuLa), which leverages the decomposability of the PO-MPC formulation to parallelize and accelerate the optimization. We apply the method to both linear and non-linear MPC problems using autonomous driving examples. 
Second, we address Task And Motion Planning (TAMP), and introduce a planner (PO-LGP) reasoning on decision trees at task level, and trajectory-trees at motion-planning level. This approach builds upon the Logic-Geometric-Programming Framework (LGP) and extends it to partially observable problems.
The experiments show the method's applicability to problems with a small belief state size, and scales to larger problems by optimizing explorative policies, which are used as macro-actions in an overarching task plan.  
\end{abstract}


\maketitle
\section{Introduction}
Robots often operate in environments that are partially observable. This paper focuses on cases where partial observability is multimodal, and concerns essential and critical aspects of the environment, such that it has a profound impact on the actions that the robot should take. For example, when a robot must stack colored blocks in a given color order (see Fig.~\ref{fig:example_problem_tamp}), and does not know initially each block's color, the optimal robot actions are strongly impacted each time the robot discovers the color of a block.
Similarly, when a car drives along a street near pedestrians (see Fig.~\ref{fig:example_problem_mpc}) whose intentions to cross or not are only partially observable, the car’s safe and appropriate behavior must adapt as each pedestrian’s intent becomes clear.

This contrasts with a unimodal and continuous form of partial observability, such as uncertainties in localization or object positions, which are often addressed heuristically (e.g., safety margins or potential fields) or by assuming a probabilistic model (e.g., Gaussian) as in Stochastic Model Predictive Control. 

\begin{figure}[ht]
    \centering
                  \subfloat[Possible start state\label{fig:example_problem_tamp:start}]{%
       \includegraphics[width=0.32\linewidth]{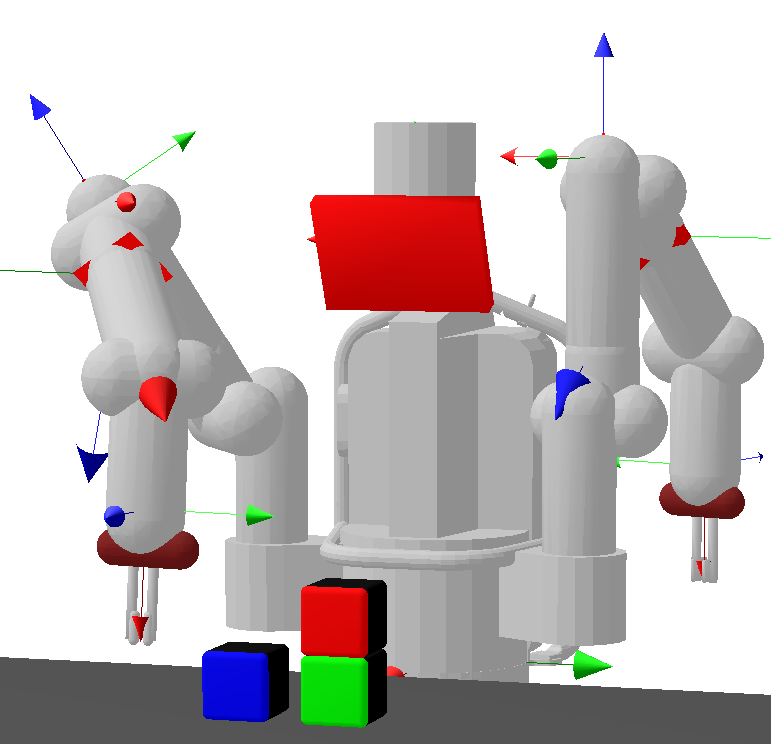}}
           \hfill
       \subfloat[Look action\label{fig:example_problem_tamp:look}]{%
       \includegraphics[width=0.32\linewidth]{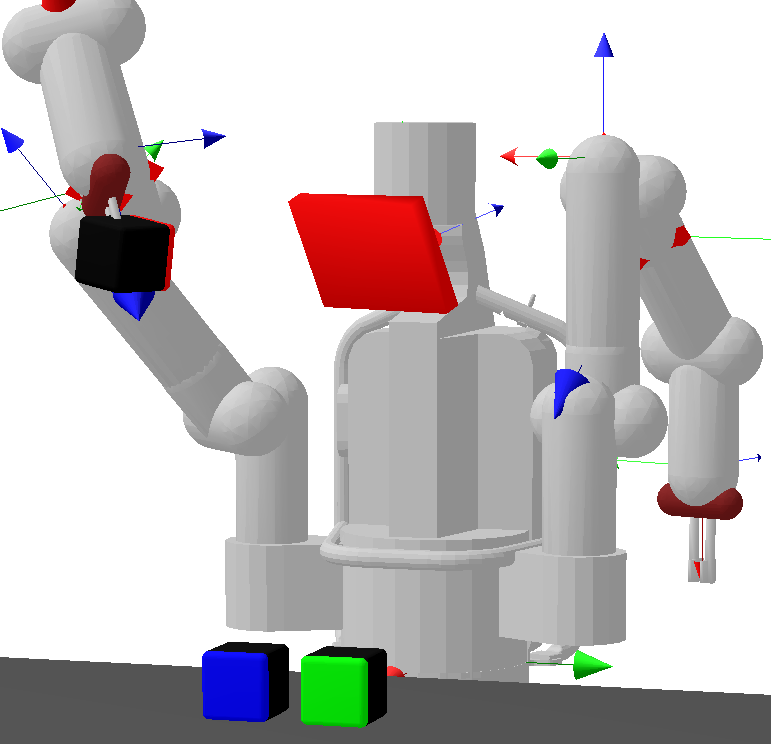}}
           \hfill
       \subfloat[Goal state\label{fig:example_problem_tamp:goal}]{%
       \includegraphics[width=0.32\linewidth]{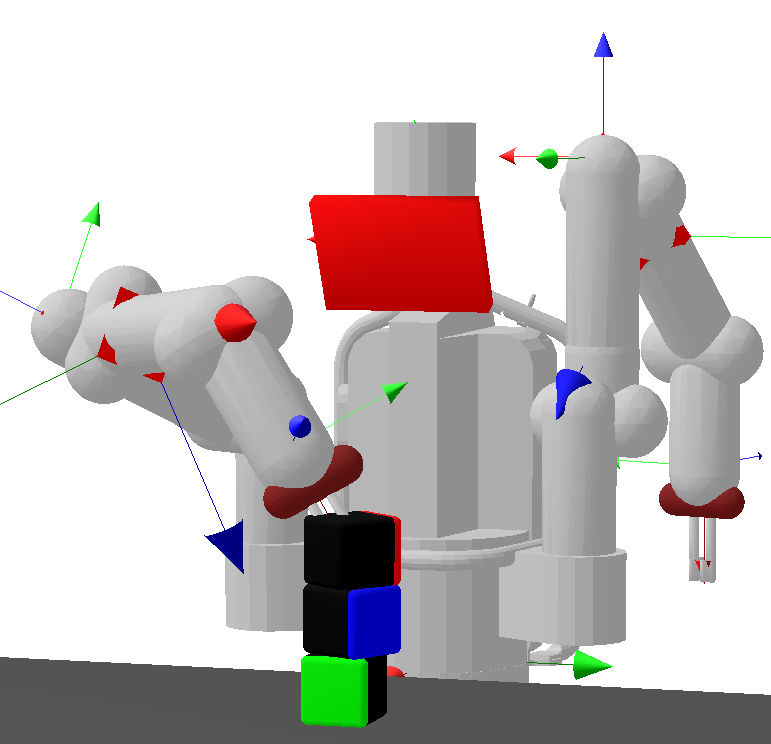}}
  \caption{Example of partially observable TAMP problem: The blocks' colors are initially not visible. The robot must look (b) and react to observations to reach the goal state (c), defined by a given color order.}
  \label{fig:example_problem_tamp} 
\end{figure}

\begin{figure}[ht]
    \centering
       \subfloat[Uncertain intention\label{2a}]{%
       \includegraphics[width=0.3\linewidth]{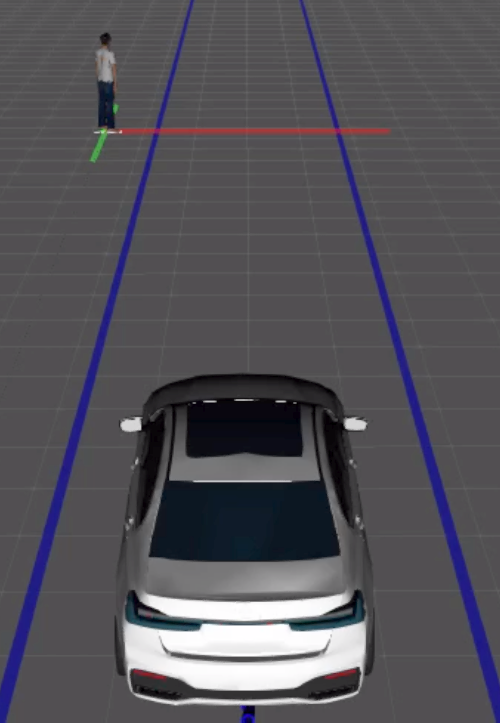}}
       \hfill
       \subfloat[Pedestrian stays on the walkway\label{2b}]{%
       \includegraphics[width=0.3\linewidth]{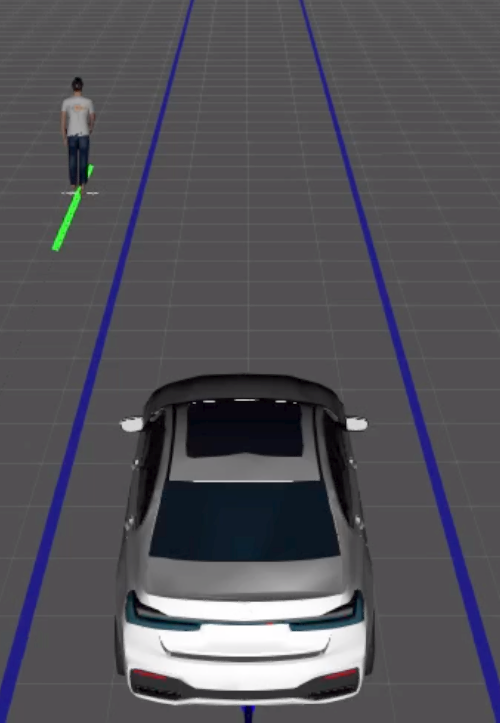}}
       \hfill
  \subfloat[Pedestrian crosses\label{2c}]{%
       \includegraphics[width=0.3\linewidth]{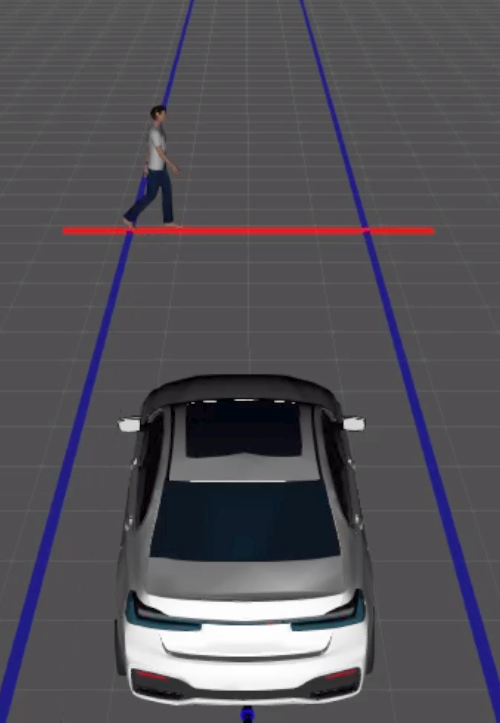}}
  \caption{
 Example of partially observable MPC Problem: a pedestrian is detected (a), whose intention is uncertain. The control policy must account for 2 cases: The pedestrian may walk along the street (b) or cross (c).} 
  \label{fig:example_problem_mpc} 
\end{figure}

In problems with multimodal partial observability, planning based on a single modality, or state hypothesis leads to incomplete plans, or can be overly conservative. For instance, in Fig.~\ref{fig:example_problem_tamp}, a complete action plan needs to account for several initial block configurations; the blocks' colors need to be observed. In Fig.~\ref{fig:example_problem_mpc} planning considering that every pedestrian would cross would be overly conservative. Conversely, ignoring the eventuality that a pedestrian might cross is unsafe.
 
To overcome those limitations, we consider multiple modalities or state hypotheses, and compute trajectory-trees in belief space. Within the planning horizon, the belief state may evolve according to different scenarios, introducing branches in the optimal motion paths, as illustrated in Fig.\ref{fig:branching_point}.

\begin{figure}[!htb]
 \center{\includegraphics[width=0.47\textwidth]{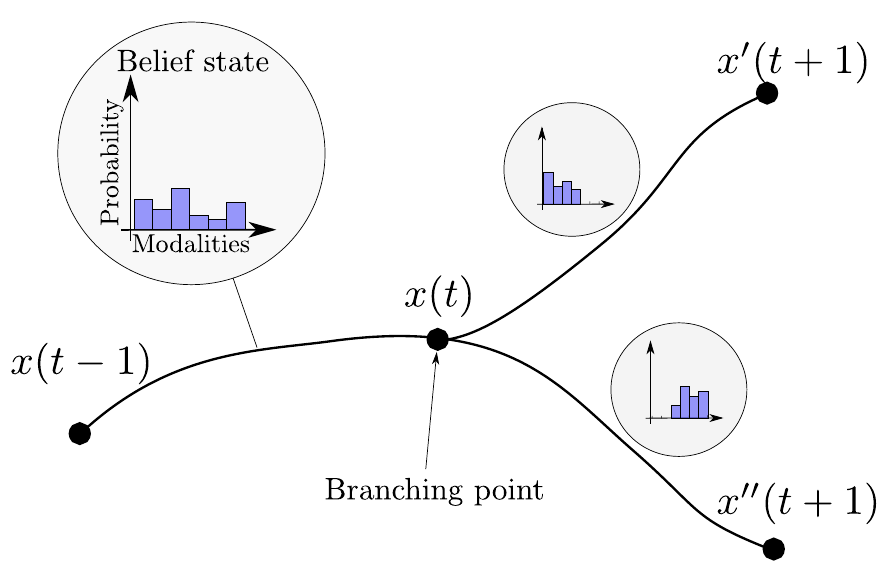}}
 \caption{Branching point: The trajectory-tree branches where the belief state evolution is anticipated to diverge into multiple possible outcomes.}
 \label{fig:branching_point}
\end{figure}



Optimizing trajectory-trees has different implications depending on the use-case:

In Model Predictive Control (MPC), a key requirement is that optimization must be sufficiently rapid to be executed in real time. To maintain tractability, we adopt a trajectory-tree structure which assumes that the state becomes fully observable after a given horizon (PO-MPC), similar to the \deleted{Q-MDP}\added{QMDP} algorithm for POMDPs. We leverage the decomposability of such trajectory-trees by introducing a new optimization algorithm, which we call Distributed Augmented Lagrangian (D-AuLa). This method combines aspects of the Augmented Lagrangian and the Alternating Direction Method of Multipliers (ADMM) methods. This is described in Section~\ref{sec:mpc} and suppported by autonomous driving examples.

In Task and Motion Planning (TAMP), the challenge is not only to optimize motions, but also to devise a symbolic plan solving the problem. In Section~\ref{sec:tamp}, we present an integrated planner (PO-LGP) for discrete partially observable problems, which is based on the Logic Geometric Programming framework (LGP). This planner produces symbolic policies as trees of actions, which are implemented at the motion level by trajectory-trees. The policies include exploratory actions through which the robot acquires information about the environment via observations, leading to belief state updates and corresponding branching points on the trajectory-tree. The optimization of the trajectory-trees utilizes a generalization of the K-Order Motion Optimization (KOMO) transcription and the optimization algorithm method described in the MPC section.

Accordingly, our contributions are:
\begin{itemize}
\item Partially Observable MPC (PO-MPC): an MPC formulation for optimizing control policies under multimodal partially observability, which leverages the probabilistic information of the belief state to reduce conservativeness without compromising safety.
\item Distributed Augmented Lagrangian (D-AuLa): an optimization procedure for accelerating the optimization of loosely coupled constrained optimization problems. It is applicable beyond trajectory optimization, and a proof of convergence and optimality is provided in Appendix~\ref{sec:proof}.
\item Partially Observable LGP (PO-LGP): a planner for partially observable TAMP problems which reasons on reactive action trees both at task and motion planning level. The planned policies account for all contingencies without reliance on replanning.
\item Tree K-Order Motion Optmization (T-KOMO): an extension of the K-Order Motion Optimization transcription for optimizing arborescent trajectories in configuration space. This transcription is used in the TAMP examples as well as in one of the MPC examples. 
\end{itemize}


This article builds upon initial work on TAMP \citep{tamp_0_phiquepal2019combined}; the conceptual limitations of the graph-based task planner are overcome by introducing a tree-based Monte-Carlo method for task planning (more details are provided in Section~\ref{sec:tree_vs_graph}). Furthermore, trajectory-trees are optimized jointly using a new method, referred to as T-KOMO instead of the heuristic approach used in \citep{tamp_0_phiquepal2019combined}. In addition, experiments are included to provide insights on how the approach scales to large belief states. The section on MPC is motivated by the simplified autonomous driving examples of \citep{tamp_0_phiquepal2019combined}. It advances previous work on MPC \citep{mpc_0_phiquepal2021control} by adding more details about the approach scalability, and provides theoretical foundations with a convergence proof for D-AuLa.


\section{Related Work}

Planning under uncertainty requires reasoning about alternative future evolutions of the system. Different robotics domains adopt different representations of such alternatives, which shape how motions are planned and executed. In the following, related work is reviewed through this lens, first considering belief-space motion planning approaches close to the POMDP framework, and then Model Predictive Control (MPC) and Task and Motion Planning (TAMP), two robotics subfields in which we instantiate the trajectory-tree abstraction.

\subsection{Belief-space Motion Planning}

Planning under partial observability can generally be formulated as a Partially Observable Markov Decision Process (POMDP). A comprehensive review of POMDP methods is beyond the scope of this paper, and we refer the reader to~\citep{pomdp_lauri2022partially, pomdp_kurniawati2022partially} for surveys of POMDPs in robotics.

Traditionally, POMDPs have often been solved offline by computing a policy that maps belief states to actions and is used at execution time to select the next action. \added{Classical approaches include point-based value iteration methods such as PBVI~\citep{pineau2003point} and SARSOP~\citep{kurniawati2008sarsop}, which approximate the optimal value function over a selected subset of reachable beliefs.} \added{Building on point-based offline POMDP solvers, Kurniawati et al.~\citep{kurniawati2011motion} propose Milestone Guided Sampling (MiGS), which exploits the spatial structure of the state space through a sampled roadmap to guide belief-space sampling and improve the scalability of offline policy computation for motion planning problems.} \deleted{For example, in\mbox{~\citep{kurniawati2011motion}} policies are computed for discretized environments (grid worlds). To improve scalability, milestones are sampled in the state space and connected into a roadmap that guides belief-space planning, allowing the method to scale to robotic systems with up to five degrees of freedom.}

Such policies represent the most general form of solution to POMDPs. However, their computation quickly becomes intractable in high-dimensional domains.

Approaches have also been developed to compute offline policies directly in continuous spaces for specific problem classes. For instance, in the context of navigation under motion and sensing uncertainty, the Feedback-based Information RoadMap (FIRM) framework~\citep{agha2014firm, agha2014health} constructs a graph in belief space whose edges are associated with local feedback controllers. The resulting solution is a policy over this graph that determines how to switch between controllers based on the current belief state, enabling feedback-driven navigation under uncertainty. This approach exploits structural properties typical of navigation problems, including low-dimensional state spaces and Gaussian belief representations, and is therefore less suited to settings involving high-dimensional systems or multi-modal belief distributions.

\deleted{To overcome the curse of dimensionality of offline policy computation, many recent POMDP-based motion planning approaches rely on online planning combined with open-loop macro-actions.\\}
\added{To overcome the curse of dimensionality associated with offline policy computation, online POMDP planners such as POMCP~\citep{pomcp_silver2010monte} and DESPOT~\citep{somani2013despot} perform online belief-space search from the current belief state, relying on  replanning instead of computing a complete policy offline. They approximate belief-space planning by reasoning over sampled state hypotheses (POMCP) or sampled scenarios (DESPOT). However, directly applying such search to continuous, high-dimensional robotic control spaces remains computationally challenging.}

\added{To improve scalability in continuous robotic domains, more recent approaches combine online belief-space planning with macro-actions.}
Planning is performed iteratively from the current belief state, while macro-actions represent motion segments executed without feedback during their duration. This drastically reduces the set of reachable belief states considered during planning.

\deleted{In\mbox{~\citep{liang2024scaling}}, macro-actions are generated efficiently using parallelized sampling-based motion planning, and planning over these actions is performed using the Reference-based POMDP approach\mbox{~\citep{kim2023reference}}, enabling scalability to domains with a combined up to 15 dimensions.\\}
\added{In~\citep{liang2026think}, the authors present ROP-RAS3, where macro-actions are generated efficiently using GPU-accelerated sampling-based motion planning, and planning over these actions is performed using the Reference-based POMDP approach~\citep{kim2023reference}, enabling scalability to domains with up to 35 dimensions.}
In~\citep{lee2020magic}, a macro-action generator is learned offline and used online to produce motion primitives represented as parameterized curves, which are then evaluated using the DESPOT algorithm\deleted{\mbox{~\citep{somani2013despot}}}.

The combination of online planning and open-loop macro-actions can be interpreted as implicitly constructing trajectory-trees in belief space, where nodes correspond to reachable belief states and edges correspond to motion segments defined by macro-actions.

\deleted{
Like the aforementioned approaches, this work represents planning contingencies using trajectory-trees. However, the motion computation differs fundamentally. Instead of selecting motion primitives generated through sampling or learning, we formulate motion generation as a gradient-based optimal control problem with differentiable costs and constraints. The trajectory-tree structure in belief space induces a coupled optimization problem in which the trajectories of all branches are optimized jointly, rather than computed independently for each branch.\\}
\added{Like the aforementioned approaches, the proposed framework also represents planning contingencies as trajectory-trees in belief space. However, it follows a different research direction. Rather than scaling POMDP planning to continuous robotic domains through macro-actions, the proposed framework extends trajectory optimization to partially observable domains.}

\added{Instead of generating motions through sampling or learned motion generators, we formulate motion generation as a gradient-based optimal control problem with differentiable costs and constraints. The trajectory-tree structure in belief space induces an optimization problem in which the trajectory-tree is optimized jointly, as one consistent optimization problem, rather than computed independently for each motion segment, yielding globally consistent and smooth motions across contingencies and, in TAMP, across action-mode switches. This formulation emphasizes trajectory optimality and is therefore well suited to applications in which the quality of the continuous motion, including smoothness, dynamic consistency, and constraint satisfaction, is of primary importance.\footnote{\added{Additionally, it is well suited to domains with complex physical object interactions, as the underlying physical laws can be expressed as differentiable equations and incorporated in an optimization problem, as shown in \citep{tamp_lgp_3_18-toussaint-RSS}.}}}

\added{
A second notable distinction is that, in its current formulation, the proposed framework optimizes a trajectory-tree, encompassing all contingencies explicitly represented by the model. This contrasts with the aforementioned macro-action-based POMDP planners which perform a selective exploration of the belief space through belief sampling and macro-action generation, thereby enabling them to address general POMDPs with continuous observations, stochastic transitions, and much larger belief spaces. Possible extensions toward trajectory-trees with partial contingency coverage are discussed in the dedicated discussion~in~Section~\ref{sec:towards_partial_contingencies}.
}

\added{
Consequently, the proposed framework targets an operating regime complementary to that of the aforementioned macro-action-based POMDP planners. In its current formulation, it does not primarily aim to scale to the large belief spaces addressed by general online POMDP solvers. Instead, it extends trajectory optimization to multi-modal partially observable domains, making it well suited to applications in which leveraging the strengths of trajectory optimization is desirable, and where partial observability can be captured by a finite and moderate number of explicit hypotheses.
}
\subsection{Model Predictive Control}
%
Model Predictive Control (MPC) typically optimizes a sequence of control inputs, or trajectory, in a receding-horizon fashion. In its original formulation, uncertainty is not explicitly modeled. Frequent replanning nevertheless provides a certain degree of robustness to disturbances.
%
\paragraph{Stochastic and Robust MPC.} Stochastic and robust variants of MPC explicitly incorporate uncertainty in the planning assumptions (e.g., unmodeled system dynamics, external disturbances, or noise). Surveys of stochastic and robust MPC can be found in~\citep{soc-1, soc-2}. Robust MPC typically seeks constraint satisfaction under worst-case realizations of the uncertainty. In contrast, stochastic MPC exploits information about the probability distribution of the uncertainty and interprets constraints probabilistically, requiring that violations occur with probability below a specified threshold. Like robust MPC, the approach presented in this work seeks constraint satisfaction even in worst-case realizations. Similar to stochastic MPC, it leverages probabilistic information about the uncertainty to reduce conservativeness.

A common robust MPC formulation is Min-Max MPC~\citep{robust-1, min-max-1} where a control sequence is optimized with respect to the worst-case realization of the uncertainty. This guarantees constraint satisfaction but can lead to overly conservative behavior. Although multiple uncertainty realizations are considered during optimization, the final solution remains a single sequential trajectory optimized against the worst case.

Another widely used approach is tube-based MPC~\citep{langson2004robust}. In this formulation, a nominal trajectory is optimized while a feedback controller ensures that the system state remains within a bounded tube around this nominal trajectory despite disturbances. This allows bounded uncertainty to be handled efficiently. However, the method assumes bounded disturbances and enforces that the system state remains within a tube around a nominal trajectory, and therefore does not represent branching multi-modal futures.

\paragraph{Tree-based approaches.} The idea of explicitly representing alternative future evolutions as a tree of scenarios within the planning horizon originates in the control literature~\citep{scokaert1998min}. Early applications were mainly in the optimal control of chemical processes. In \citep{multi-stage-luca-1-LUCIA201269, multi-stage-luca-1-LUCIA20131306, multi-stage-luca-1}, Lucia et al. introduced such an approach and termed it Multi-Stage MPC. Subsequent work further developed and optimized this formulation \citep{multi-stage-klintberg-1, multi-stage-leidereiter-1}, in particular by exploiting the tree structure to accelerate the optimization of the underlying QP problems. In this prior work, the tree formulation aims to improve controller performance in the presence of general continuous disturbances. In contrast, in our approach the semantics of the tree structure differ: it models partial observability, with branching occurring at observation points where the belief state is updated.

In robotics, tree-based approaches have been applied to collision avoidance in \citep{multi-stage-robotics-1, mpc_105} for mobile robots. In the context of autonomous driving, this has been applied to highway lane changes, traffic light, and intersection scenarios for autonomous driving \citep{mpc_100, mpc_104, mpc_105} as well as evasion maneuver in \citep{multi-stage-klintberg-2}.

Among these robotics related studies, \citep{multi-stage-klintberg-2} is the only one detailing a parallelizable solving scheme, albeit for linear MPC only. The other studies do not explicitly tackle the challenges of optimization time and scalability, and use off-the-shelf solvers like IPOPT~\citep{wachter2006implementation_ipopt}  for \citep{mpc_100, mpc_104},  and OSQP~\citep{stellato2020_osqp} in \citep{mpc_105}. IPOPT is also used in \citep{multi-stage-robotics-1} via the Do-MPC toolbox~\citep{lucia2017rapid_do_mpc}.

In contrast, our work focuses not only on the benefit of the tree formulation (PO-MPC), but also on the efficiency of the optimization method. We decompose the control problem as a QMDP and leverage the low-coupling between the optimal control problem of each branch of the trajectory-tree to optimize it in a distributed fashion (D-AuLa). The presented scheme is applicable both to linear and non-linear MPC.

\subsection{Task and Motion Planning}
Task and Motion Planning (TAMP) is the robotic planning subfield concerned with solving hybrid planning problems that combine symbolic task planning, which determines a sequence of high-level actions, with motion planning, which computes the continuous motions required to execute these actions. While many robotic problems can fall into this category, the methods developed in the TAMP literature are particularly relevant when the coupling between task and motion planning is strong, meaning that symbolic decisions cannot be evaluated reliably without considering the geometric feasibility and cost of the associated motions. This arises commonly in manipulation problems, where the combined complexity of the robot and environment makes it difficult to assess the feasibility of symbolic actions without feedback from a motion planner. Surveys of recent developments in TAMP are available in~\citep{tamp_tamp_1_garrett2021integrated, tamp_tamp_2_guo2023recent}.

\paragraph{Optimization-based TAMP and LGP.} The approach proposed in this paper is rooted in the Logic–Geometric Programming (LGP) framework~\citep{tamp_lgp_1_toussaint2015logic, tamp_lgp_3_18-toussaint-RSS}. A distinctive feature of LGP is its emphasis on global optimality through trajectory optimization, where the motions associated with a task plan are optimized jointly across task modes and kinematic switches.
This contrasts with many TAMP approaches that prioritize feasibility and rely on sampling-based motion planning, typically computing motions independently for each action. Optimizing trajectories globally naturally produces smooth motions and ensures consistency across action transitions.

Building on this perspective, the proposed method extends trajectory optimization to trajectory-trees under partial observability. The optimization is performed not only across action and kinematic switches but also across observation branchings, ensuring global consistency of the resulting trajectory-tree. In contrast to approaches based on sequential convex optimization~\citep{motion_planning_Schulman2013FindingLO}, we adopt an Augmented Lagrangian method that avoids repeated linearization of the problem.

\paragraph{Determinization under partial observability.} Most TAMP approaches focus on deterministic problems, where a solution is a sequence of actions along with associated motions, and no replanning is expected. When partial observability is considered, a common strategy is to approximate the problem as deterministic, resulting in solutions that remain in the form of action sequences.

This is, for example, the case in \citep{tamp_po_0_kaelbling2013integrated}, where planning is performed in belief space using a determinized model of the belief space dynamics. At execution time, the actual trajectory (in belief space) of the robot may leave the plan's envelope, requiring replanning.

A similar approach is adopted in  \citep{tamp_po_3_adu2021probabilistic, tamp_po_3bis_adu2023long} with optional replanning once an observation is received, which is beyond the plan's expectations.
In \citep{tamp_po_4_hadfield2015modular} the belief space dynamics are determinized, using Maximum Likelihood Observation (MLO)~\citep{determinization_platt2010belief}.
In \citep{tamp_po_1_garrett2020online} self-loop determinization \citep{determinization_keyder2008hmdpp} is used, and replanning is performed after the execution of each action.

An advantage of using a determinized model, is that it enables the usage of deterministic planner for task planning. For example, Fast-Forward~\citep{task_planner_hoffmann2001ff} is used in \citep{tamp_po_4_hadfield2015modular}, and PPDL-Stream solvers \citep{task_planner_garrett2020pddlstream} are used in \citep{tamp_po_1_garrett2020online}. 

However, planning under such determinization schemes presents two limitations. First, it explicitly relies on replanning to account for cases where the outcomes encountered at execution time differ from those predicted by the determinization. Second, it implicitly assumes that some outcomes are much more likely than others, and that the planning goal is reachable considering only these likely outcomes. The efficiency of such an approach may degrade drastically in problems where partial observability is multimodal, and where selecting a single outcome at each step may fail to reach the goal, or lead to plans having a low probability of success, requiring frequent replanning. For example, in Fig.~\ref{fig:example_problem_tamp}, a \textit{Look} action only has a $\frac{1}{6}$ probability of identifying the block, compared to a $\frac{5}{6}$ probability that it is inconclusive. Solving the stacking task, however, necessitates to account for the unlikely outcome that the colored side is discovered.

\paragraph{Tree-based approaches.} To overcome these limitations, our approach does not determinize the belief state dynamics. Instead, the optimized trajectory-trees account for all contingencies foreseen under the planning assumptions. Provided those assumptions hold, the robot's state is expected to remain within the scope of the planned trajectory-tree during execution, without relying on replanning.

This shares similarities with \citep{shah2020anytime}, where tree-like plans are generated using the SSP solver LAO*~\citep{task_planner_hansen2001lao}. However, the nature of the uncertainty addressed differs: \citep{shah2020anytime} tackles action stochasticity, whereas our approach addresses partial observability.
Additionally, \citep{shah2020anytime} plans motions for actions independently (piecewise paths). Our method performs piecewise trajectory optimization during the policy search phase, but it also includes a final phase of joint trajectory-tree optimization, in which the entire tree is optimized across action switches and across observation branchings which significantly reduces trajectory costs. Finally, the two works differ in scope: \citep{shah2020anytime} focuses on the high-level planning loop, i.e. how to orchestrate the calls to a SSP solver and motion planner, whereas this paper explores the specifics of the low-level trajectory optimization for arborescent plans, in addition to the high-level aspects.

\section{Trajectory-Trees in Belief Space} \label{sec:intro_trajectory_trees}
In this section, we introduce the trajectory-tree abstraction, which provides a common foundation for the PO-MPC and PO-LGP methods presented in the following sections.
\subsection{State Representation} \label{sec:intro_trajectory_tree_state}
We consider a hybrid state comprising both a continuous geometric component and a symbolic component:
\begin{itemize}
\item $x \in \mathcal{X}$, where $\mathcal{X}$ is the continuous state space of the robot and its environment.
\item $s \in \mathcal{S}$, where $\mathcal{S}$ is a set of symbolic states.
\end{itemize}
In the TAMP example shown in Fig.~\ref{fig:example_problem_tamp}, the continuous state $x$ corresponds to the geometric configurations of the robot and the objects, while the symbolic state $s$ encodes the actual color of a block’s colored side. In the  MPC example shown in Fig.~\ref{fig:example_problem_mpc}, the continuous state $x$ models the car’s state (position, velocity), while the symbolic state $s$ captures the pedestrian’s intention to cross.

\revise{
\rrevise{
\subsection{Mixed Observability}
We plan in a context of mixed observability, where the hybrid state $(x, s)$ is decomposed into observable and latent components. \rrevise{Observable components are directly accessible at each step.
Latent components are not directly observable and must be inferred through the belief state. 
While latent variables are not directly observed, their values may become fully determined 
over time through belief updates.}
Specifically, we assume a fixed partition of the state space,
\[
x = (x^{\text{observable}}, x^{\text{latent}}), \quad
s = (s^{\text{observable}}, s^{\text{latent}}),
\]
where the observable and latent components are distinct state variables. This decomposition is structural and does not change over time.
\footnote{Observability is a fixed property of the state variables; observable and latent components do not correspond to values of the same variable.}}}

\revise{A belief state encapsulates both the observable components of the state and a probability distribution over the \rrevise{latent} components. Formally, it is expressed as
\[
((x^{\text{observable}}, s^{\text{observable}}),\; \beta(x^{\text{latent}}, s^{\text{latent}})),
\]
where \( \beta(x^{\text{latent}}, s^{\text{latent}}) \) denotes a probability distribution over the latent state variables. We refer to \( \beta \) as the belief distribution, distinguishing it from the full belief state.}

\medskip
\noindent\revise{\textbf{Multi-hypothesis belief assumption:}  \rrevise{We assume the belief distribution $\beta$ to be defined over} a finite set of hypotheses (or modalities), denoted as \( \{(x_m^{\text{latent}}, s_m^{\text{latent}})\}_{m \in \mathcal{H}} \), where \( \mathcal{H} \) is the index set of hypotheses.} 
\rrevise{ During planning, the likelihood and underlying state associated with each hypothesis may evolve. In particular, some observation outcomes may result in zero probability being assigned to certain hypotheses, but the index set $\mathcal{H}$ remains fixed.}

\rrevise{This representation is related to particle-based belief models, where each hypothesis (or particle) carries a state and an associated probability. However, in contrast to particle filters, where particles are samples without persistent semantic identity, we assume a fixed set of hypotheses whose identities are preserved throughout planning and which explicitly represent the possible modalities of the belief.
}

\medskip
\revise{
In Fig.~\ref{fig:example_problem_tamp}, the positions of the blocks are observable, but both the orientation of the colored side relative to the parent block (continuous) and its color (symbolic) are partially observable, resulting in a finite set of hypotheses. In Fig.~\ref{fig:example_problem_mpc}, the pedestrian's intention is partially observable (symbolic).
}
\revise{
\subsection{Trajectory-Trees} \label{sec:intro_trajectory_tree_definition}
Before detailing the planning structure, we first state an assumption on the system dynamics used throughout.
}

\medskip\noindent\revise{\textbf{Deterministic continuous dynamics assumption:} The continuous systems dynamics are deterministic, i.e., applying a given control input \( u \in \mathcal{U} \), where $\mathcal{U}$ denotes the continuous control space, from a given continuous state \( x \in \mathcal{X} \) results in a unique successor state \(x' \). This enables the representation of motion segments between branching points as open-loop trajectories, rather than feedback control policies.
}

\revise{
A \textit{trajectory-tree} is a tree in belief space. Each node in the tree is associated with a belief state. The tree alternates between two types of stages, as illustrated in Fig.~\ref{fig:trajectory_tree}:
\begin{itemize}
\item \textbf{Trajectory stage}: This stage represents a segment of system evolution. Starting from a node, a control trajectory is executed---potentially in conjunction with a symbolic action in the TAMP setting---which deterministically updates the observable state and the underlying state hypotheses of the belief state. The belief distribution, however, remains unchanged.
\item \textbf{Belief update stage}: This stage models a probabilistic branching into multiple possible belief states, each represented by a child node and an associated branching probability. The belief distribution is updated, but the underlying hypotheses and the observable state remain unchanged. This stage is represented in blue in Fig.~\ref{1b}. 
\end{itemize}
}

\begin{figure}[ht]
    \centering 
    \subfloat[Belief space view\label{1b}]{%
    \includegraphics[width=0.49\linewidth]{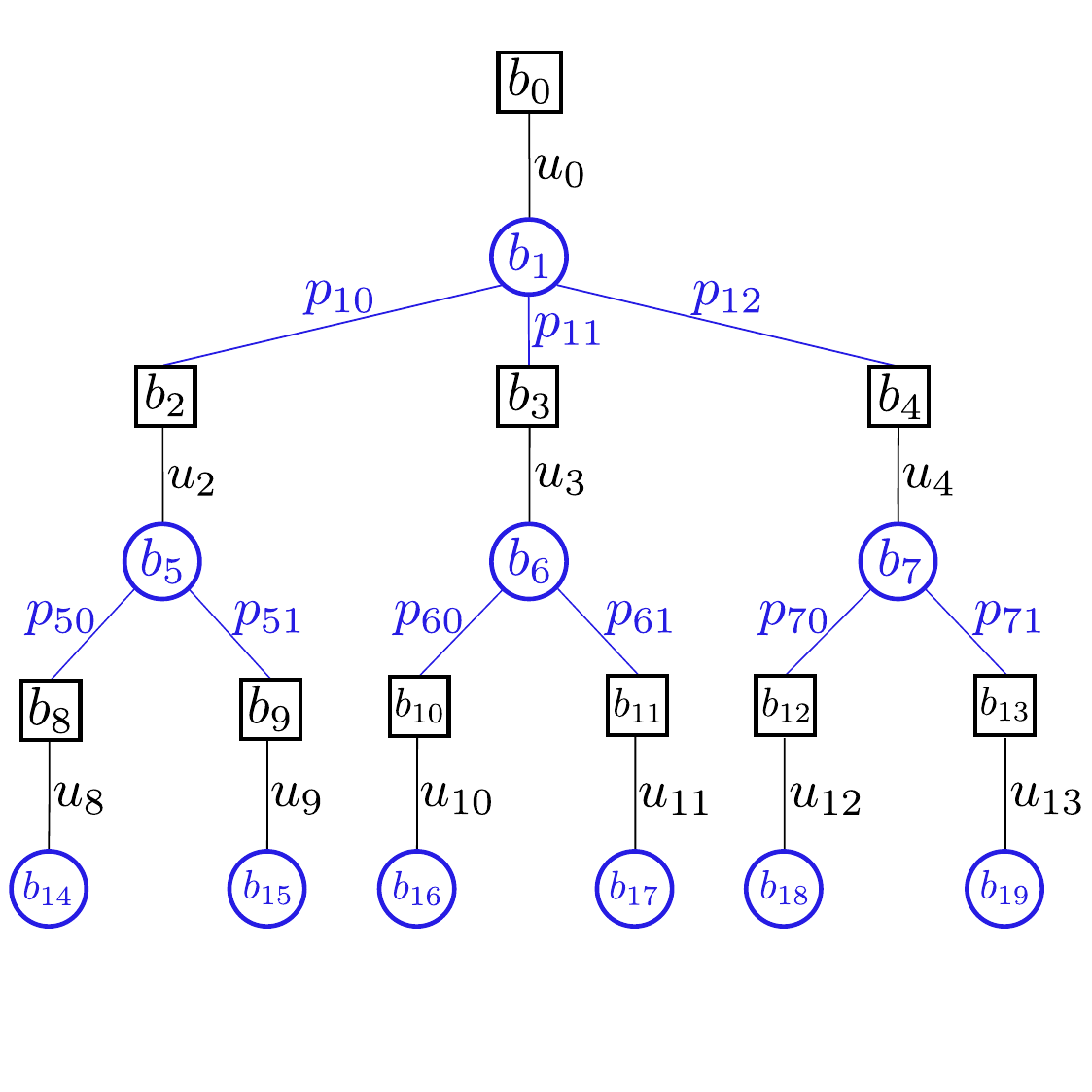}}
                  \hfill
    \subfloat[Geometric view\label{1a}]{%
    \includegraphics[width=0.49\linewidth]{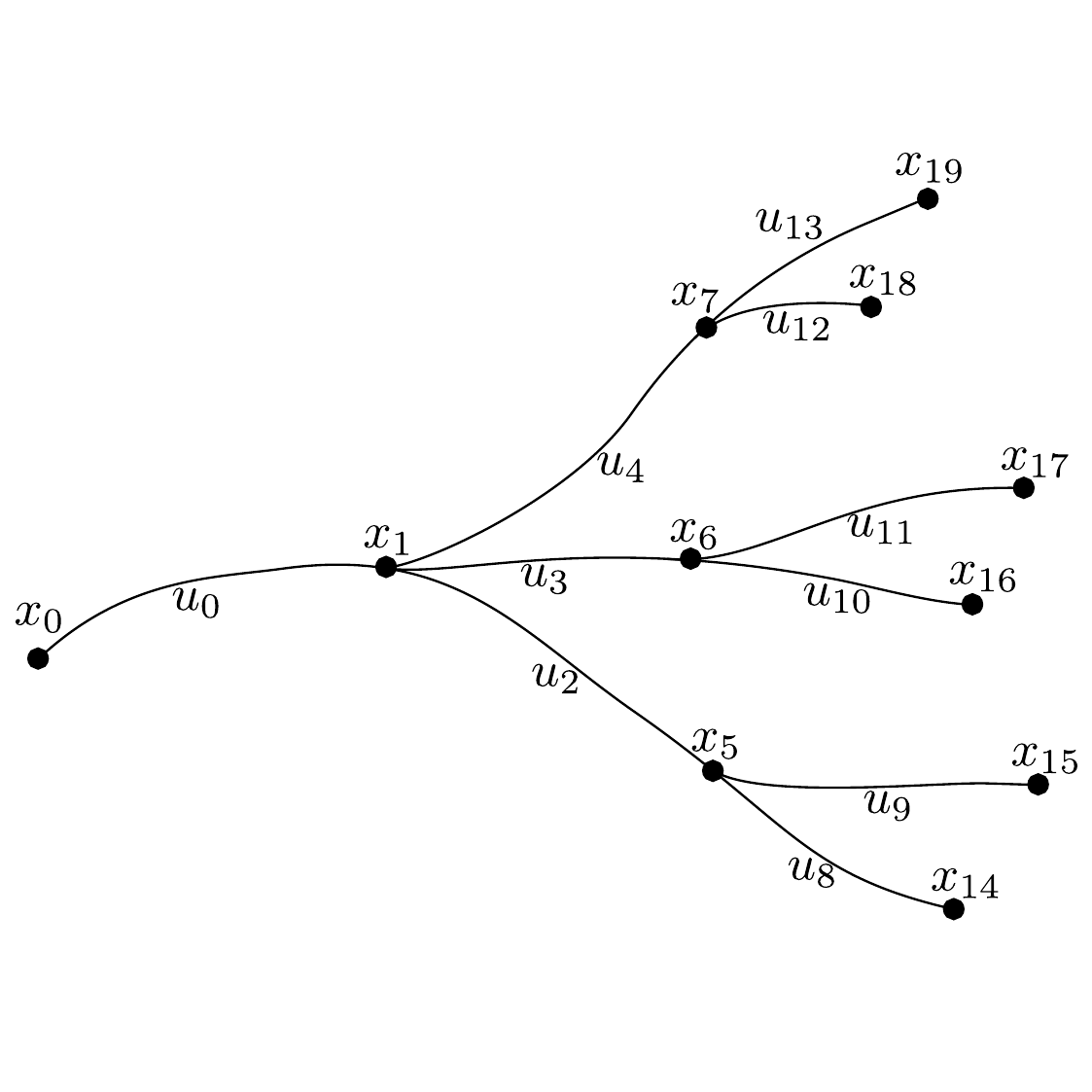}}
  \caption{\revise{Illustrative example of trajectory-tree in belief space: The trajectory stages (in black) represent the system evolution under an applied control $u$. The belief update stages (in blue) correspond to probabilistic branching of the belief state. From a geometrical persective, the motion forms a trajectory-tree.}} 
  \label{fig:trajectory_tree} 
\end{figure}

\revise{
\noindent Trajectory-trees represent the evolution of the hybrid belief state with probabilistic branching occuring at fixed time points. In the PO-MPC method employed for MPC, branching arises from an assumed determinization of the partially observable state after a fixed horizon. In contrast, in the PO-LGP used for TAMP problems, branching points correspond to the reception of observations that inform the robot about its environment. These differing mechanisms give rise to structurally distinct trees, as illustrated in Fig.~\ref{fig:sparse_trajectory_trees}.
}

\begin{figure}[ht]
    \centering 
    \subfloat[PO-MPC: Early branching for fast computations\label{subfig:qmdp}]{%
    \includegraphics[width=0.49\linewidth]{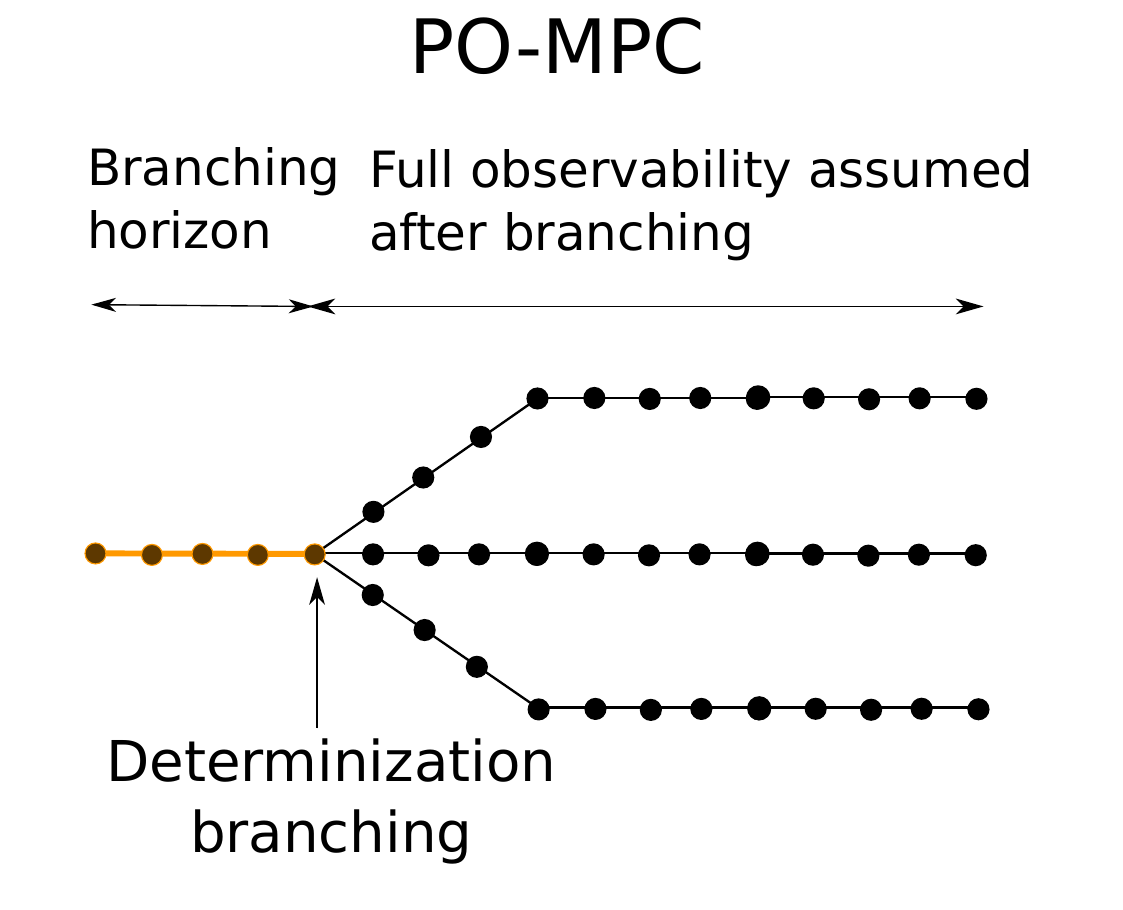}}
                  \hfill
    \subfloat[PO-LGP: Look actions to enable exploratory behavior\label{subfig:tamp}]{%
    \includegraphics[width=0.49\linewidth]{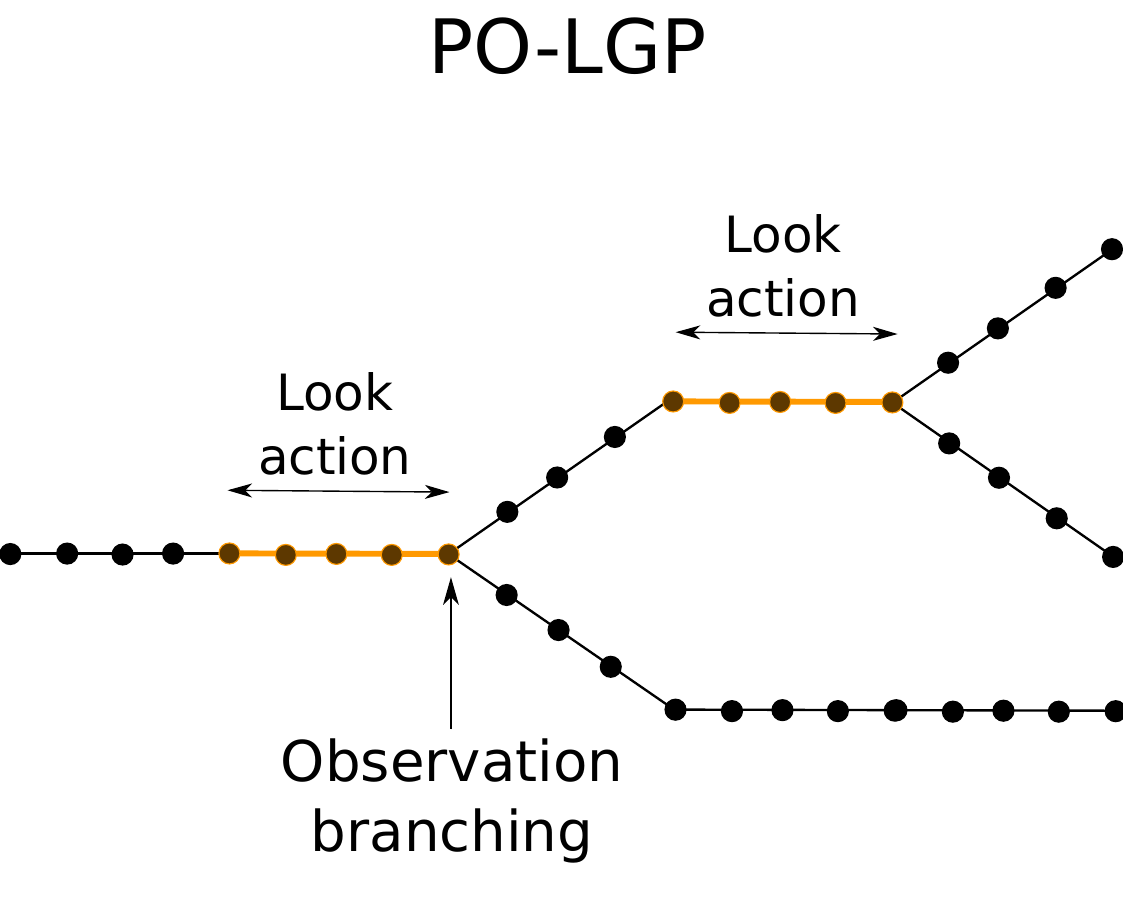}}
  \caption{Example of tree structures: In PO-MPC \ref{subfig:qmdp} the trajectory-tree assumes full observability after the first branching (similarly to \deleted{Q-MDP}\added{QMDP}). In PO-LGP for TAMP \ref{subfig:tamp}, Look actions provide observations, resulting in a belief state update.} 
  \label{fig:sparse_trajectory_trees} 
\end{figure}

\revise{
In the following, we use the term \textit{continuous trajectory-tree} to denote the continuous motions that compose a trajectory-tree. When the context is clear, we may refer to the continuous part simply as the trajectory-tree, omitting the qualifier for brevity. In TAMP, where planning involves symbolic decision-making, we use the term \textit{policy} to refer to the symbolic part of the trajectory-tree, governing the high-level sequencing of actions and observations.
}

\revise{
\subsubsection{Trajectory-Tree Optimization} \label{sec:intro_trajectory_tree_opt}
The specific procedures for optimizing the continuous components of the trajectory-trees in the MPC and TAMP contexts will be described in their respective sections. However, we first highlight key elements that are common to both formulations:
\begin{itemize}
\item \textbf{Symbol-induced optimization problem}:
The trajectory-tree is computed by solving a constrained optimization problem, where the cost and constraint functions are defined by the symbolic components of the tree. In MPC, these functions are associated with the symbolic state $s$. In TAMP, where symbolic actions can actively modify the symbolic state, the cost and constraint functions are instead tied to the symbolic actions.
\item \textbf{Minimization of the expected the trajectory cost}: Since branching in the trajectory-tree is probabilistic, the optimization objective is expressed as the minimization of the expected trajectory cost.
\item \textbf{Robust constraint satisfaction}: Constraints are enforced uniformly across all edges of the trajectory-tree, irrespective of the probability of reaching a given edge. This aligns with the principle of robust constraint satisfaction in the MPC literature, as opposed to formulations based on chance constraints. 
\item \textbf{Augmented Lagrangian-based optimization:} 
Constraints on the trajectory-tree are enforced using the Augmented Lagrangian method (AuLa). In TAMP, the standard formulation of the algorithm is applied. For the MPC, the method is extended to support distributed optimization (D-AuLa).
\end{itemize}
}

%

\section{Model Predictive Control (MPC)} \label{sec:mpc}
We develop trajectory-tree optimization for the MPC use case.
We first describe the PO-MPC problem formulation and state the trajectory-tree optimization objective in Section~\ref{sec:mpc_problem_formulation}. \revise{The main goal of the PO-MPC formulation is to improve the control performance by leveraging the multi-hypothesis representation of the trajectory-tree}. Next, we detail the optimization procedure D-AuLa used to optimize the trajectory-trees. \revise{D-AuLa significantly reduces computation time by exploiting parallelism}. Experiments are provided in Section~\ref{sec:mpc_acc}~and~\ref{sec:mpc_slalom}. The results are discussed in Section~\ref{sec:mpc_discussion}.
\subsection{The PO-MPC Problem Formulation}
\subsubsection{State Representation}
\revise{
We use the hybrid, partially observable state representation introduced in Section~\ref{sec:intro_trajectory_trees}.
}

\medskip
\noindent \revise{\textbf{Continuous and symbolic states}: The continuous state $x$ corresponds to the usual notion of state in the MPC literature, while $s$ represents discrete aspects of the planning problem. In the example shown in Fig.~\ref{fig:example_problem_mpc}, the continuous state captures the car's position and velocity and is fully observable. The symbolic state represents the pedestrian’s intention, which is partially observable and therefore modeled as a latent variable.
}

\medskip
\noindent \revise{\textbf{Relation between the continuous and symbolic states}:
The two state components represent different aspects of the planning problem, and no further relation is assumed between them. However, as outlined  in Section~\ref{sec:intro_trajectory_tree_opt}, and formalized further in the upcoming Section~\ref{sec:mpc_problem_formulation}, the cost and constraint functions defining the optimization problem are linked to the symbolic state. Therefore, if two hypotheses are qualitatively different to the extent that they involve different planning objectives, this difference must be encoded as different values of a symbolic state variable.
}

\subsubsection{Assumed Control Pipeline}
\revise{The PO-MPC formulation assumes a modular architecture where a perception module tracks and provides the multiple hypotheses $\mathcal{H}$ and the belief distribution, as shown in Fig.~\ref{fig:control-loop}.}

\begin{figure}[!htb]
 \center{\includegraphics[width=0.47\textwidth]{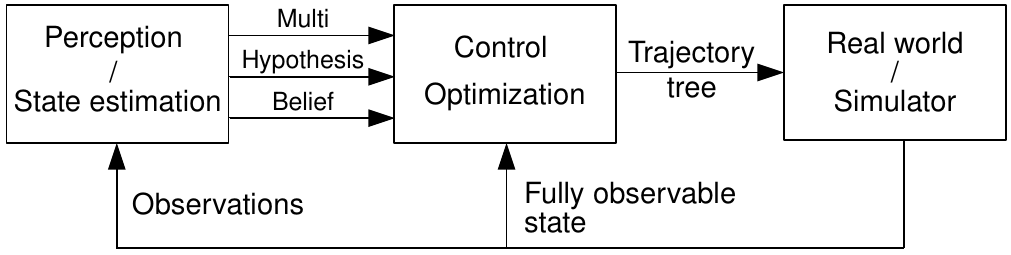}}
 \caption{PO-MPC Control loop: trajectory-trees are optimized with respect to a belief distribution over multiple hypotheses provided by a perception module.}
 \label{fig:control-loop}
\end{figure}
\revise{
Belief space inference is therefore decoupled from planning. At each planning cycle, a trajectory-tree is optimized based on the current belief state estimate.
}

\revise{
Such architecture in which perception components output multiple hypotheses or modalities, each associated with a confidence score or probability, is commonplace in the autonomous driving domain, as for example in~\citep{zyner2019naturalistic, phan2020covernet}.
}

\added{
This architecture is well suited to domains in which the belief over the symbolic state hypotheses remains approximately independent of the robot's continuous behavior over the planning horizon, for example obstacle hypotheses or pedestrian intentions over short time horizons. Conversely, it does not capture interactions in which this belief evolves in response to the robot's actions, such as a pedestrian becoming more inclined to cross after observing the vehicle slow down. Capturing such interactions would require explicitly reasoning about future belief evolution, which is beyond the present formulation and is discussed as a possible extension in Section~\ref{sec:mpc_discussion}. When this assumption is violated, the planned policy may fail to anticipate such interactions, potentially leading to suboptimal decisions. This limitation is partially mitigated in practice by the receding-horizon nature of MPC, since the planner is repeatedly updated using newly observed belief states as execution progresses.
}

\revise{
\subsubsection{Trajectory-Tree Structure}
The trajectory-tree structure is directly determined by the number of hypotheses $|\mathcal{H}|$ of the belief state, and is not co-optimized as it will be the case in TAMP (see Section~\ref{sec:tamp}).
}

It comprises a first common trunk (see orange part in Fig.~\ref{fig:control-tree}) which is the part executed by the controller until the next planning cycle happens. This trunk spans a time interval which we call the \revise{\textit{branching horizon}.}


\revise{Beyond the branching horizon, the trajectory-tree evolves into $|\mathcal{H}|$ branches, \rrevise{each corresponding directly to a specific state hypothesis $m \in \mathcal{H}$.}
} The full trajectory-tree spans a total time interval, which we refer to as the \textit{prediction horizon}, following standard terminology in MPC. \revise{\rrevise{
In the case of a single pedestrian in the scene, there are two state hypotheses, resulting in a trajectory-tree with two branches, as shown in Fig.~\ref{fig:control-tree}.
}
}

\begin{figure}[!htb]
  \center{\includegraphics[width=0.45\textwidth]{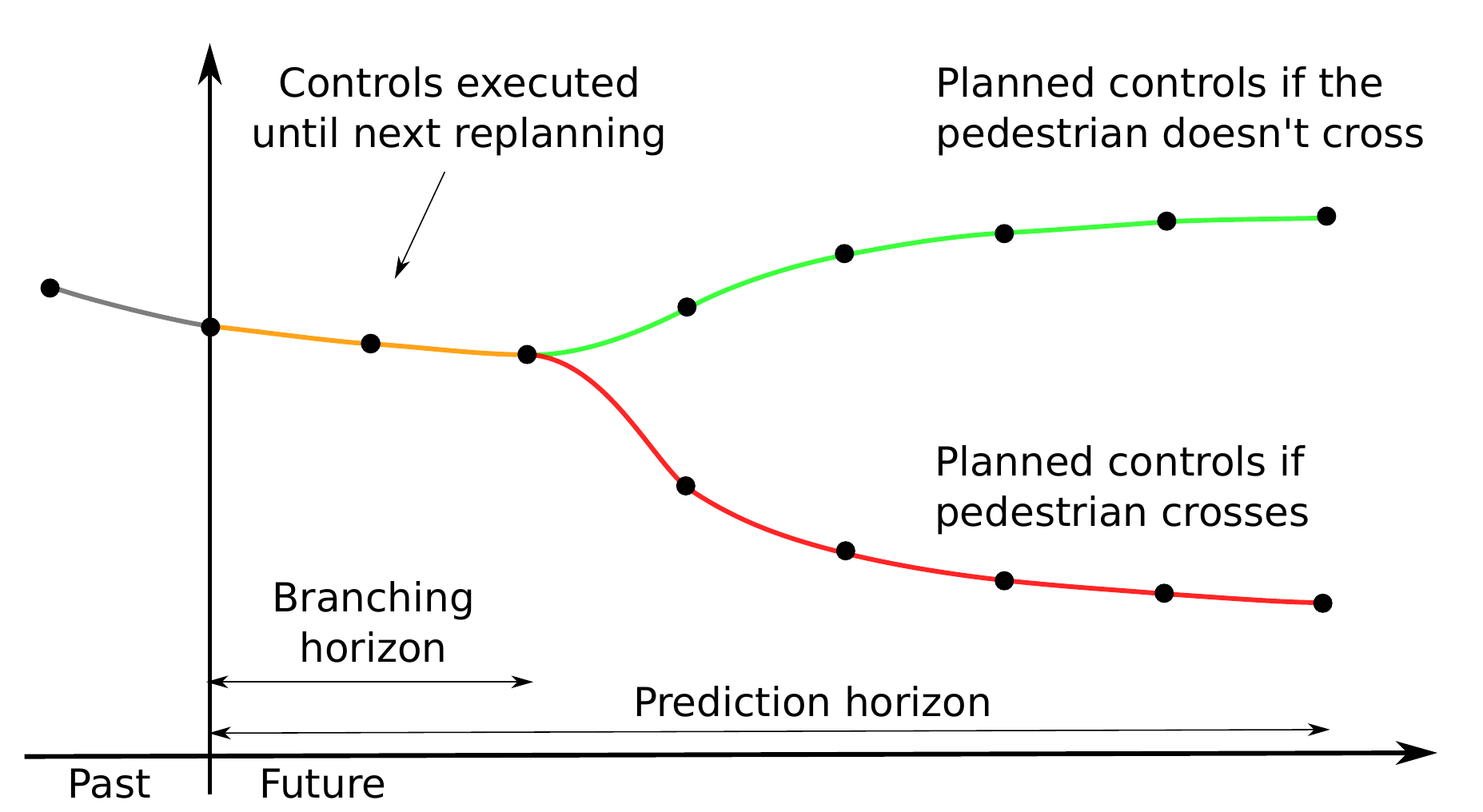}}
\caption{\label{fig:control-tree} Schematic trajectory-tree applied to the pedestrian example: the part before the branching horizon (orange) is executed until the next planning cycle. Beyond the branching horizon, the tree evolves in two branches corresponding to the two hypothetical states: pedestrian crossing (red) or not (green).}
\end{figure}

The trajectory-tree elements before the branching horizon are associated with the current belief state, while the elements beyond the branching horizon have fully determined belief states. This amounts to a form of determinization, where it is idealized that an observation would be received that fully reveals the hidden state. This type of determinization is also employed in the QMDP method~\citep{qmdp} for solving POMDPs. 

The numerical value of the branching horizon is directly motivated by a QMDP analogy.
While QMDP assumes that full observability is achieved after a single action, this assumption is transposed here to full observability one planning cycle. Accordingly, the branching horizon is set to the estimated time between two planning events.\footnote{As the planning cycle time may not be fully deterministic, we recommend using a worst-case estimate.}
By construction, this ensures that only controls preceding the branching point are executed, and reflects the hypothetical nature of the determinization event represented by the branching point.

Importantly, the belief space dynamics within this tree structure do not depend on either explicit observations or an observation model. The branching probabilities are solely determined by the current belief state. This corresponds to the QMDP-style determinization underlying the trajectory-tree structure with a single branching, and is consistent with the control pipeline described in the previous section, where belief state tracking is separated from planning. This modeling choice, however, entails a limitation discussed below, and contrasts with belief space dynamics where belief state inference and planning are integrated and thereby produce trajectory-trees with multiple branching stages, contingent on observations, as introduced in the TAMP section.

This decoupled structure offers several advantages:
\begin{enumerate}[(1)]
\item The trajectory-tree optimization does not rely on an explicit observation model, making the approach modular and easy to integrate with various perception modules, as it only requires an input in the form of multiple hypotheses with associated probabilities.
\item The total number of trajectory elements is kept lower compared to a structure with more branching points.
\item The trajectory-tree optimization can be decomposed into nearly independent subproblems, thereby improving scalability. \rrevise{This is aspect is illustrated in Fig.~\ref{fig:control-tree-subproblems} and is the subject of Section~\ref{sec:d-aula-solver}.}
\end{enumerate}

Despite these advantages, one limitation is that, similar to the \deleted{Q-MDP}\added{QMDP} algorithm, it cannot plan actions that actively seek to gain information. In the second experiment with uncertain obstacles found in Section~\ref{sec:mpc_slalom}, such an action would be to intentionally come closer to the obstacle to observe it better. It is not a desired behavior in that case, but could be relevant in other planning problems. In TAMP problems for instance, exploring the environment is an important skill, such that the PO-LGP approach that we detail in Section~\ref{sec:tamp} adopts the alternative approach, with multi-stage branching based on observations, as examplified in Fig.~\ref{fig:sparse_trajectory_trees}, which enables exploratory behavior.
%
\subsubsection{Optimization Problem Formulation}\label{sec:mpc_problem_formulation}
Let $L$ denote the number of steps in the branching horizon, and $T$ the total number of steps on each branch ($L \ll T$). \revise{We note $d_u$ and $d_x$ the dimension of the control vector at each time step, and $d_x$ the state dimension.} 

\noindent \revise{Let $u_m \in \mathbb{R}^{d_u \times T}$ and $x_m \in \mathbb{R}^{d_x \times T}$ be the control and continuous state sequences on the branch corresponding to the state hypothesis $m \in \mathcal{H}$. The state sequence $x_m$ comprises both the partially and fully observable components.} Let $\tilde{u} \in \mathbb{R}^{d_u \times L}$ be the controls in the branching horizon, and $p(m)$ be the probability of state hypothesis $m$ as indicated by the belief state. We note $s_m$ the discrete state associated with the state hypothesis $m$.

\noindent We formulate the optimal control problem as follows:
\begin{subequations} \label{eq:optimization-problem}
\begin{align}  
  \min_{\tilde{u}, u_m, x_m}&{
  \sum_{m \in \mathcal{H}}
  {p(m)\sum_{t=0}^{T-1}
  {
  c_{s_m}(x_m(t), u_m(t)) \label{eq:cost-term}
  }
  }
  },\\
  \text{s.t.} \quad
  & g_{s_m}(x_m(t), u_m(t)) \leq 0,&&\forall{m}, \forall{t}\label{eq:constraint-term},\\
  & x_m(t+1) = f(x_m(t), u_m(t)),&&\forall{m}, \forall{t}\label{eq:mpc_model},\\
  & u_m(t) = \tilde{u}(t),&&\forall{m}, \forall{t < L},\label{eq:non-anticipativity}
\end{align}
\end{subequations}
where $c_{s_m} : \mathbb{R}^{d_x + d_u} \rightarrow \mathbb{R}$ is a scalar cost function, $g_{s_m} : \mathbb{R}^{d_x + d_u} \rightarrow \mathbb{R}^{d_g}$ defines $d_g$ inequality constraints functions. Those functions are indexed by $s_m$ to reflect the fact that both the planning goal and constraints may depend on the symbolic state. The function $f : \mathbb{R}^{d_x + d_u} \rightarrow \mathbb{R}^{d_x}$ are equality constraints modeling the system dynamic. 

\revise{Although each pair $(u_m, x_m)$ corresponds to a specific hypothesis $m \in \mathcal{H}$, the optimization variables across different $m$ are not fully independent from each other}. Indeed, in the branching horizon, they all correspond to the common trunk and must therefore be equal. This is captured by the Eq.~\eqref{eq:non-anticipativity} and is usually called non-anticipativity constraint in multi-stage MPC~\citep{multi-stage-luca-1-LUCIA201269}. In the following we call $\tilde{u}$ the consensus variable.

The common trunk of the tree is constrained by the active constraints of all states, regardless of the state likelihood. This is for guaranteeing the robustness of the constraint satisfaction.

In contrast, the minimized costs Eq.~\eqref{eq:cost-term} are weighted by the belief state for optimizing more with respect to likely states than unlikely ones.
\subsubsection{Transcription to Generic Solver Format}
\label{sec:problem_reduction}
Here we rewrite the optimization problem in a generic format that is the input to our solver. In Eq.~\eqref{eq:optimization-problem} both the controls $u$ and states $x$ are optimization variables. In many cases, it is possible to eliminate either the controls or the configurations and obtain a more compact formulation:

\begin{itemize}
\item Optimization in control space: This is achieved by eliminating the variable $x$, and is known as single shooting. It is well described in the MPC literature, particularly for linear MPC \citep{book_underactuated, book_Diehl2013}. We apply this approach in the first MPC experiment presented in Section~\ref{sec:mpc_acc}.
\item 
Optimization in configuration space: In that case, the controls $u$ are eliminated. This typically requires adding additional constraints to ensure the existence of controls implementing the configuration transitions e.g. for non-holonomic robots, or to impose control bounds. The LGP formulation \citep{tamp_lgp_1_toussaint2015logic}, and the KOMO solver \citep{komo_solver} follow this approach. We use this transcription for the second MPC experiment described in Section~\ref{sec:mpc_slalom} and the TAMP examples of Section~\ref{sec:tamp}.
\end{itemize}

The problem can be rewritten in a generic format:

\begin{subequations} \label{eq:gen-optimization-problem}
\begin{align}  
  \min_{\tilde{z}, z_m} \quad &{
  \sum_{m \in \mathcal{H}}
  {p(m) c_{s_m}(z_m) \label{eq:gen-cost-term}
  }
  },\\
  \text{s.t.} \quad
  & g_{s_m}(z_m) \leq 0,&&\forall{m}\label{eq:gen-constraint-term},\\
  & h_{s_m}(z_m) = 0,&&\forall{m}\label{eq:gen-mpc_model},\\
  & z_m(t) = \tilde{z}(t),&& \forall{m}, \forall{t < L},\label{eq:gen-non-anticipativity}
\end{align}
\end{subequations}
where $\tilde{z}$ and $z_m, m \in \mathcal{H}$ are the optimization variables, which can be in control space, configuration space, or a combination of both.
We note $d$ the dimensionality of the optimized parameters at each time step. The functions $c_{s_m}:\mathbb{R}^{T \times d} \rightarrow \mathbb{R}$ and $g_{s_m}:\mathbb{R}^{T \times d} \rightarrow \mathbb{R}^{d_{g_{s_m}}}$ correspond directly to the original cost and constraints functions ~\eqref{eq:gen-cost-term} and ~\eqref{eq:gen-constraint-term} respectively. The function $h_{s_m}:\mathbb{R}^{T \times d} \rightarrow \mathbb{R}^{d_{h_m}}$ is an equality constraint which can capture the system dynamics (see ~\eqref{eq:gen-mpc_model}). It is optional, since it is not needed in case of direct shooting transcription. Eq.~\eqref{eq:gen-non-anticipativity} is the non-anticipativity constraint and $\tilde{z}$ is the consensus variable. We generally assume the functions $c_{s_m}$, $g_{s_m}$, and $h_{s_m}$ to be smooth, but not necessarily convex or unimodal.

\subsection{Distributed~Augmented~Lagrangian solver (D-AuLa)} \label{sec:d-aula-solver}
\rrevise{The global optimization problem \eqref{eq:gen-optimization-problem} can naturally be decomposed into $|\mathcal{H}|$ loosely coupled optimization subproblems, each corresponding to one branch of the trajectory tree and its associated state hypothesis, as illustrated in Fig.~\ref{fig:control-tree-subproblems}}.
\begin{figure}[!htb]
  \center{\includegraphics[width=0.45\textwidth]{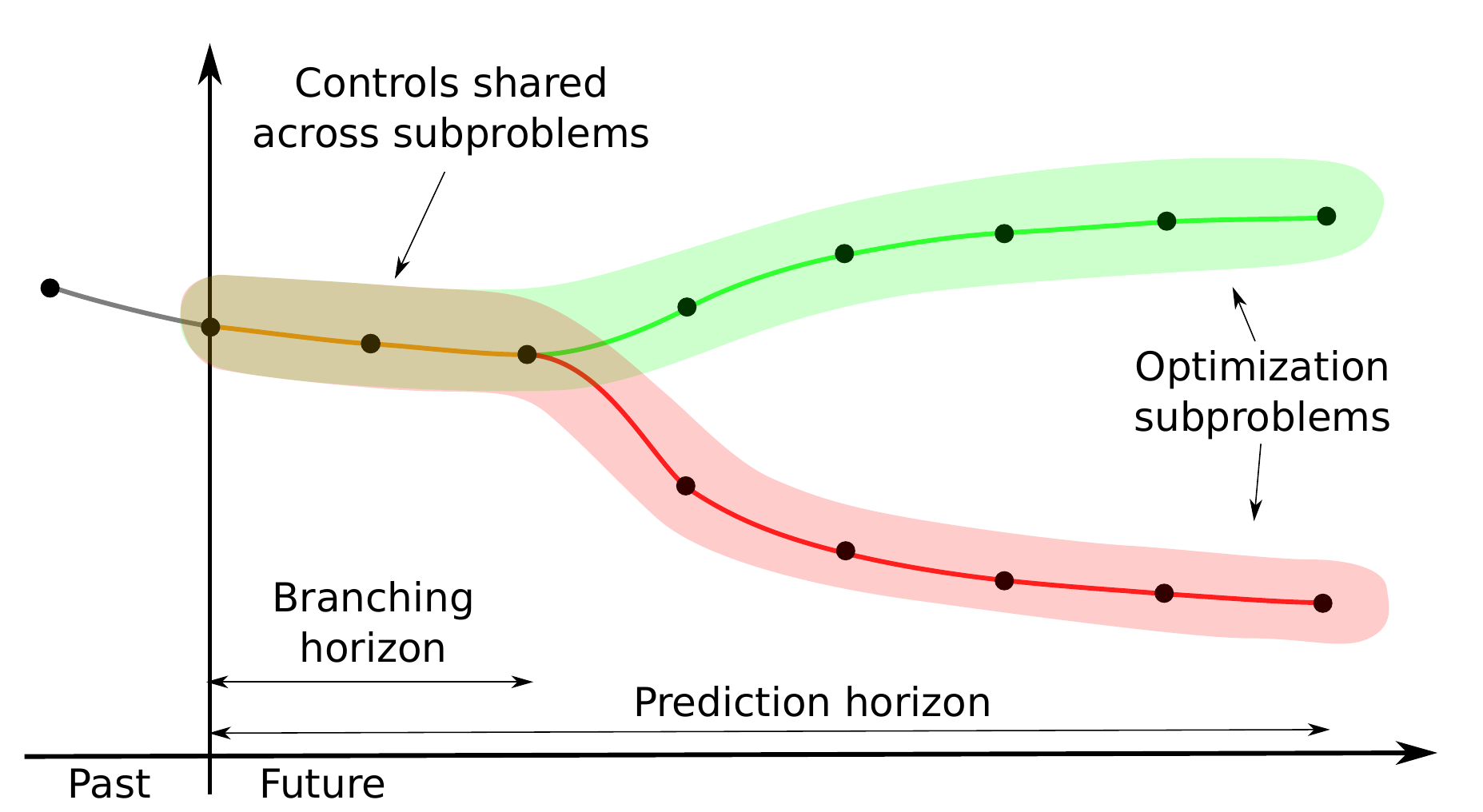}}
\caption{\rrevise{\label{fig:control-tree-subproblems} Decomposition into subproblems: each branch defines an optimization subproblem. The controls before the branching horizon are shared across branches, inducing coupling.}}
\end{figure}
\rrevise{The subproblem associated with a state hypothesis $m \in \mathcal{H}$ follows from the global objective~\eqref{eq:gen-optimization-problem} and is given by:
\begin{subequations}
\begin{align}
\min_{z_m} \quad & p(m)\, c_{s_m}(z_m) \label{eq:subproblem_cost} \\
\text{s.t.} \quad
& g_{s_m}(z_m) \le 0, \label{eq:subproblem_ineq}\\
& h_{s_m}(z_m) = 0, \label{eq:subproblem_eq}\\
& z_m(t) = \tilde{z}(t), \quad \forall t < L, \label{eq:subproblem_non_anticipativity}
\end{align}
\end{subequations}
where the cost term~\eqref{eq:subproblem_cost} corresponds to the term associated with hypothesis $m$ in the global objective~\eqref{eq:gen-cost-term}, and the constraints of the global problem are retained. In particular, the non-anticipativity constraint~\eqref{eq:subproblem_non_anticipativity} induces a coupling between subproblems through the shared consensus variable~$\tilde{z}$. Consequently, the subproblems cannot be optimized independently. However, this coupling is weak, as the non-anticipativity constraint only applies over the branching horizon, which is a small portion of the full trajectory-tree ($L \ll T$).}
%
%

The core idea of the D-AuLa algorithm is to take advantage of this decomposition. Each subproblem is smaller and can be optimized faster. In addition, some parts of the optimization can be parallelized. 
It performs multiple iterations, each consisting of two phases:
\begin{itemize} 
\item A distributed phase where a relaxed version of each subproblem is optimized.
\item A centralized phase at which the consensus variable $\tilde{z}$ is updated.
\end{itemize}
The proposed method integrates the Augmented Lagrangian method (AuLa) with Alternating Direction Method of Multipliers (ADMM) to address both subproblem-specific constraints and the non-anticipativity constraint at the same abstraction level. This is in contrast to a hierarchical approach, where ADMM handles decomposition in an outer loop, and subproblems are solved to high accuracy using standard constrained optimization in an inner loop. By enforcing all constraints simultaneously and incrementally, our method avoids the inefficiency of strictly satisfying subproblem constraints before achieving consensus, ensuring a more cohesive and efficient optimization process. 

\revise{
In the literature, the ADMM algorithm usually refers to a decomposition into two subproblems which are solved and updated in a sequential fashion. Our method builds upon the variation called \textit{consensus optimization}, as described in \citep{ADMM}, for a \textit{N}-fold decomposition, and where the optimizations of the subproblems are parallelizable. We refer the reader to Appendices~\ref{sec:background_aula} and ~\ref{sec:background_admm}  for background on AuLa and consensus-ADMM respectively.
}

\subsubsection{Distributed Augmented Lagrangian (D-AuLa)}
The optimization problem of each \revise{subproblem} is similar to a standard constrained optimization problem, but the coupling constraint \rrevise{\eqref{eq:subproblem_non_anticipativity} makes it peculiar, as the global consensus variable $\tilde{z}$ is shared across all subproblems and must be jointly optimized. In other words, the branches of the trajectory tree cannot be optimized independently, since they share a common trunk that couples the subproblems.} 

\rrevise{To orchestrate this joint optimization, while still exploiting the decomposition into subproblems, we form $|\mathcal{H}|$ unconstrained objectives, one for each subproblem, by combining the augmentations from both the AuLa and the ADMM methods. }

We call this unconstrained objective the \textit{Distributed Augmented Lagrangian}:
\begin{subequations} \label{eq:lagrangian-problem}
\begin{alignat}{3}
L_m(z_m, \tilde{z}, &\lambda_m, \kappa_m, \eta_m) =\ &&
  p(m)c_{s_m}(z_m) 
  \label{eq:lagrangian-cost-term}\\
&+ \lambda_m \cdot g_{s_m}(z_m) &&+ \frac{\mu}{2} \norm{[g_{s_m}(z_m) > 0] \odot g_{s_m}(z_m)}^2 \label{eq:lagrangian-ineq}\\
&+ \kappa_m \cdot h_{s_m}(z_m) &&+ \frac{\nu}{2} \norm{h_{s_m}(z_m)}^2 \label{eq:lagrangian-eq}\\
&+ \eta_m \cdot \Delta z_m &&+ \frac{\rho}{2} \norm{\Delta z_m}^2, \label{eq:admm}
\end{alignat}
\end{subequations}
where $\Delta z_m \in \mathbb{R}^{L\times d}$, with $\Delta z_m(t) = z_m(t) - \tilde{z}(t), \forall t < L$ is the difference between $z_m$ and the consensus $\tilde{z}$ in the branching horizon. 

Equations~\eqref{eq:lagrangian-ineq} and \eqref{eq:lagrangian-eq} are the terms of the Augmented Lagrangian method for handling the inequality and equality constraints intrinsic to each subproblem. The $\odot$ notation in ~\eqref{eq:lagrangian-ineq} is for the element-wise multiplication. In other words, the square penalty applies only on the elements of $g_{s_m}(z_m)$ which are violating the inequality constraint.

Equations~\eqref{eq:admm} are similar ADMM terms for solving the coupling between the subproblems.

 
$\lambda_m \in \mathbb{R}^{d_{g_m}}, \kappa_m \in \mathbb{R}^{d_{h_m}}, \eta_m \in \mathbb{R}^{L \times d}$ are the dual variables, and $\mu, \nu, \rho$ are fixed positive constants.

\subsubsection{Optimization Procedure}
The optimization algorithm consists in executing several iterations of the following steps:
\begin{subequations} \label{eq:algo}
\begin{align}  
{z_{m}^{k+1}} &:= \min_{z_m} L_m(z_m, \tilde{z}^k, \lambda_m^k, \kappa_m^k, \eta_m^k) \label{eq:algo-newton}\\
\lambda_m^{k+1} &:= \max(0, \lambda_m^{k} + \mu g_{s_m}(z_m^{k+1})) \label{eq:algo-dual-ineq} \\
\kappa_m^{k+1} &:= \kappa_m^k + \nu h_{s_m}(z_m^{k+1}) \label{eq:algo-dual-eq} \\
\tilde{z}^{k+1} &:= \frac{1}{|\mathcal{H}|}\sum_{m\in \mathcal{H}}{z_m^{k+1}} \label{eq:z_update} \\
\eta_m^{k+1} &:= \eta_m^k + \rho \Delta z_m^{k+1} . \label{eq:dual-admm}
\end{align}
\end{subequations}
where the subscript $k$ indicates the number of iterations. The lines~\eqref{eq:algo-newton}, \eqref{eq:algo-dual-ineq}, \eqref{eq:algo-dual-eq} and \eqref{eq:dual-admm} are indexed by $m$, and are performed for each subproblem. On the other hand, \eqref{eq:z_update} is a centralized step, it is where the $z_m$ resulting from all sub-optimizations are gathered together. The Fig.~\ref{fig:optim-flow} shows the execution flow. The next subsections explain each line of the procedure one by one.

\begin{figure}[!htb]
 \center{\includegraphics[width=0.475\textwidth]{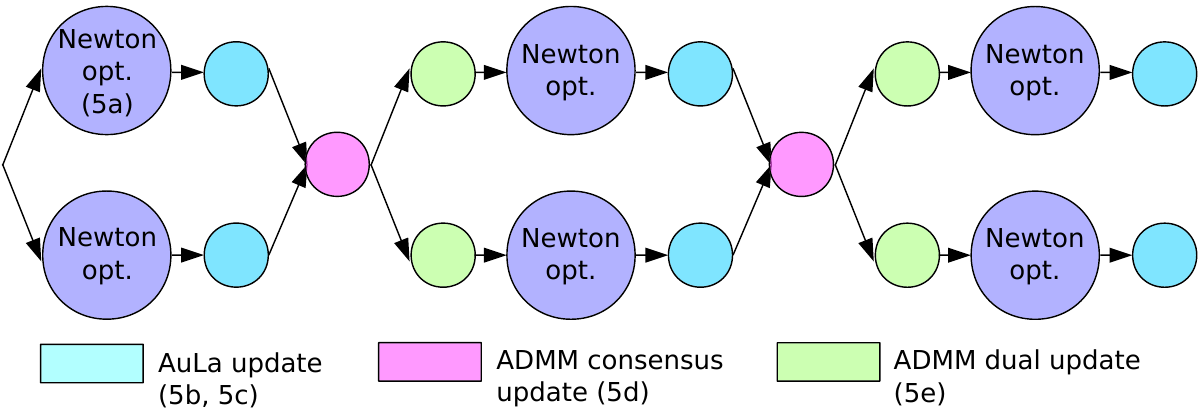}}
 \caption{Execution flow for $|\mathcal{H}|=2$. The costly steps (Newton minimizations) are parallelized.}
 \label{fig:optim-flow}
\end{figure}

\noindent\textbf{Initialization\normalfont{:}}
All dual variables $\lambda_m, \kappa_m, \eta_m$ are initially set to $0.0$.
The optimization variables $z_m$ as well as the initial common consensus variable $\tilde{z}$ can be initialized randomly, or by any better heuristic to speed-up the optimization.

\medskip
\noindent\textbf{Unconstrained minimization\normalfont{:}}
The step~\eqref{eq:algo-newton} is the minimization of $|\mathcal{H}|$ unconstrained optimization problems and is the part which is computationally expensive. These are optimized using a Gauss-Newton procedure. It therefore requires the gradients $\nabla c$, $\nabla g$, $\nabla h$ to be provided. The Hessian, or second order derivatives, are approximated from the gradients.
As Fig.~\ref{fig:optim-flow} shows, these optimizations are independent from each other and can be performed in parallel. This is where the algorithm takes full advantage of the decomposition. 

\medskip
\noindent\textbf{Augmented Lagrangian dual variables update\normalfont{:}}
Equations~\eqref{eq:algo-dual-ineq} and ~\eqref{eq:algo-dual-eq} update the dual variables corresponding to the inequality and equality constraints respectively. As in the pure AuLa method, the Lagrange multipliers are updated to a value that pushes out of constraint violations and ``should'' lead to satisfied constraints in the next iteration.

\medskip
\noindent\textbf{ADMM consensus variables update\normalfont{:}}
Line ~\eqref{eq:z_update} updates the consensus variable. It can be performed only once the computations of all the $z_m^{k+1}$ are finished. Its computation is fairly intuitive, $\tilde{z}^{k+1}$ is the average of the $z_m^{k+1}$ on the branching horizon. 

\medskip
\noindent\textbf{ADMM dual variable update\normalfont{:}}
Line~\eqref{eq:dual-admm} updates the dual variables corresponding to the ADMM equality constraint. It is based on the last results $z^{k+1}_m$ and the consensus variable $\tilde{z}^{k+1}$. It can therefore be performed only after the step~\eqref{eq:z_update}. The ADMM equality constraint ~\eqref{eq:gen-non-anticipativity} is treated similarly to the standard equality constraint~\eqref{eq:gen-mpc_model} throughout the optimization procedure. The only difference, is that the constraint definition is not fixed, since $\tilde{z}$ is a moving target. Both the constraints related to the planning problem and the tree consistency are enforced gradually over the course of the optimization. 

\medskip
\noindent\textbf{Termination criterion\normalfont{:}}
The procedure can be stopped once the constraints of each subproblem are satisfied~\eqref{eq:aula-primal-ineq}, ~\eqref{eq:aula-primal-eq}, once a consensus for $\tilde{z}$ is reached~\eqref{eq:admm-primal}, and once the optimization procedure is stationary ~\eqref{eq:aula-dual},~\eqref{eq:admm-dual}. Formally it means that the AuLa and ADMM residuals are smaller than threshold values ($\epsilon^{pri}, \epsilon^{opt}, \xi^{pri}, \xi^{dual} \in \mathbb{R}_{>0}$). 

\begin{subequations} \label{eq:termination-criterion}
\begin{alignat}{3}
 &\norm{[g_{s_m}(z^k_{m}) > 0] \odot g_{s_m}(z^k_{m})} \leq \epsilon^{pri}, \label{eq:aula-primal-ineq}\\
 &\norm{h_{s_m}(z_{m}^k)} \ \ \ \leq \epsilon^{pri}, \label{eq:aula-primal-eq} \\
   &\norm{z^{k}_{m} - z^{k-1}_{m}} \ \ \leq \epsilon^{opt}, \label{eq:aula-dual} \\
 &\norm{z^k_{m} - \tilde{z}^k} \leq \xi^{pri}, \label{eq:admm-primal} \\
 &\norm{\tilde{z}^{k} - \tilde{z}^{k-1}}\ \ \ \leq \xi^{dual}. \label{eq:admm-dual}
\end{alignat}
\end{subequations}

In the examples of the experimental section, the algorithm typically converges after 10 to 30 iterations.

\subsubsection{Convergence and Optimality}
This algorithm differs from the consensus ADMM algorithm by its additional constraints $g$ and $h$ which lead to additional augmentation terms in the Augmented Lagrangian ~\eqref{eq:lagrangian-ineq}, \eqref{eq:lagrangian-eq} and algorithmic steps \eqref{eq:algo-dual-ineq} and ~\eqref{eq:algo-dual-eq}. These additional constraints do not affect the convergence properties of the ADMM algorithm. For convex problems with equality constraints, but no inequality constraints, a proof is provided in the Appendix~\ref{sec:proof} showing that convergence to the global optimum is guaranteed. The key points of the argument are summarized as follows:
\begin{itemize}
\item The D-AuLa algorithm can be reframed into a form closer to the standard ADMM, with only two subproblems and sequential steps. The additional constraints apply to only one of the subproblems. We call it the Constrained ADMM algorithm (C-ADMM). The relation between those two forms is established in Appendix~\ref{sec:equivalence}
\item Convergence of the C-ADMM algorithm is proven in Appendix~\ref{sec:constrained_admm_proof}, by extending the proof of the standard ADMM given in \citep{ADMM}. 
\end{itemize}
Non-convexity is discussed in Appendix~\ref{sec:nonconvex}, and 
inequality constraints are discussed in Appendix~\ref{sec:inequality_constraints}.


\subsection{Experiments}
We evaluate the proposed approach in two distinct domains, as presented in Sections~\ref{sec:mpc_acc} and~\ref{sec:mpc_slalom}. The performance obtained with PO-MPC, particularly in terms of control costs, is analyzed and discussed independently within each section. Optimization time with D-AuLa is specifically examined and discussed in Section~\ref{sec:benefit_decomposition}. The solver is implemented in C++. The source code and a supplementary video are available for reference\footnotemark \footnotetext{\url{https://github.com/cambyse/trajectory_tree_mpc}}.

\subsubsection{Adaptative Cruise Control among pedestrians} \label{sec:mpc_acc}
We consider the problem briefly introduced in Fig.~\ref{fig:example_problem_mpc}. The car drives along a street in presence of pedestrians on the sides who may cross. The pedestrians' intentions are partially observed through a simulated perception module that outputs, for each pedestrian, the probability that it will cross in front of the car. Eventually pedestrians either cross the street, or walk on the walkway making their intention fully observable. 

We optimize the longitudinal acceleration of the car. The car dynamics are modeled as a linear system~\eqref{eq:mpc_linear_dynamic} and we consider quadratic cost with linear constraints. The problem is reduced to an optimization in control space only (see Section \ref{sec:problem_reduction}). The optimization problem~\eqref{eq:gen-optimization-problem} takes the form of $N$ loosely-coupled QPs. 

Performance is evaluated on randomized scenes in simulation. The car dynamics are simulated in Gazebo. \revise{Planning occurs at a frequency of \SI{10}{\hertz}, with trajectory-trees planned over a prediction horizon of \SI{5}{\second}, and a branching horizon set at \SI{1}{\second}}. The trajectory-tree is discretized at 4 steps per second.

We compare the results obtained with the trajectory-tree optimization versus two baselines. The first baseline (referred to as single-hypothesis) is a classical MPC approach where a sequential trajectory is planned. To make sure that no collision happens with pedestrians, the hypothesis used is the worst case: as long as the pedestrian's intention is uncertain, the car plans to stop in front of it. The second baseline (refered to as full-observability) is an idealized case, possible only in simulation, the perception module has full information: it knows in advance which pedestrians will cross. It outputs crossing probabilities which are either $0.0$ or $1.0$. This second baseline gives a lower-bound of the control costs. Only a sequential trajectory needs to be planned in this case.

\medskip
\noindent\textbf{System dynamics\normalfont{:}}
The system dynamics are described by the linear system:
\begin{align}
\begin{pmatrix} 
x_{t+1} \\
 v_{t+1}
\end{pmatrix} 
= 
\begin{pmatrix}
 1 & dt \\
 0 & 1
\end{pmatrix}
\begin{pmatrix} 
x_{t} \\
 v_{t}
\end{pmatrix} 
+
\begin{pmatrix} 
0 \\
dt
\end{pmatrix}
\begin{pmatrix} 
u_t
\end{pmatrix}\ , \label{eq:mpc_linear_dynamic}
\end{align}
 where $x$ and $v_t$ are respectively the longitudinal coordinate of the vehicle along the road and its velocity. The continuous state is the compound vector $(x, v)$. $u_t$ is the controlled acceleration.

\medskip
\noindent\textbf{Partially observable discrete state\normalfont{:}}
With $N_p$ pedestrians located at $x_0< ..<x_{N_p-1}$ ahead of the vehicle, the discrete state can be described by an integer $m \in [0..N_p]$ indicating the closest pedestrian who crosses. $m=N_p$ is the case where no pedestrian crosses.

We note $p_m$ the output of the simulated prediction module giving the crossing probability of the $m^{th}$ pedestrian. The $m^{th}$ pedestrian is the closest crossing pedestrian if: it crosses, and, the pedestrians before him do not cross, such that:
\begin{align}
p(m) = p_m\prod_{i=0}^{m-1}1-p_i \ .
\end{align}

\medskip
\noindent\textbf{Costs and constraints\normalfont{:}}
The trajectory-tree is optimized w.r.t the following trajectory costs:
\begin{itemize}
\item \textbf{Speed}: The matrix 
$\boldsymbol{Q}=\begin{pmatrix} 
0 & 0\\
0 & k_v \footnotemark
\end{pmatrix}$ penalizes the velocity difference between the vehicle speed and a given desired velocity.
\item \textbf{Acceleration}:
The square acceleration is penalized, $\boldsymbol{R}=\begin{pmatrix} k_u \footnotemark[\value{footnote}] \end{pmatrix}$.
\end{itemize}
In addition, the following constraints are applied:
\begin{itemize}
\item \textbf{Stop before the $i^{th}$ pedestrian}: A state inequality constraint applies to stop and keep a safe distance to the pedestrian, $x \leq x_{i} - d_{saftey} \footnotemark[\value{footnote}]$. This constraint is applied only when planning with respect to a state hypothesis in which the pedestrian intends to cross.
\item \textbf{Control bounds}: Longitudinal acceleration is constrained to stay between bounds $[-8.0, 2.0] m/{s^2}$.
\end{itemize}

\footnotetext{Table \ref{tab:acc_costs} is obtained with $k_u = 5.0$, $k_v = 1.0$, $d_{safety} = 2.5$ m}

\medskip
\noindent\textbf{Example of trajectory-trees\normalfont{:}}
Fig.~\ref{fig:control-tree-vs-sequential} shows a trajectory-tree obtained with a vehicle launched at $48$ km/h ($30$mph) in the presence of 3 pedestrians. Each pedestrian has a $0.15$ probability of crossing. This implies a probability of $0.61$ that the road is free. The trajectory-tree does not brake too hard, but still guarantees that it is possible to come to a stop in the worst case (see red curve in Fig.~\ref{fig:control-tree-vs-sequential}). On the other hand, the single hypothesis approach brakes much stronger.

\begin{figure}[!htb]
 \center{\includegraphics[width=0.45\textwidth]{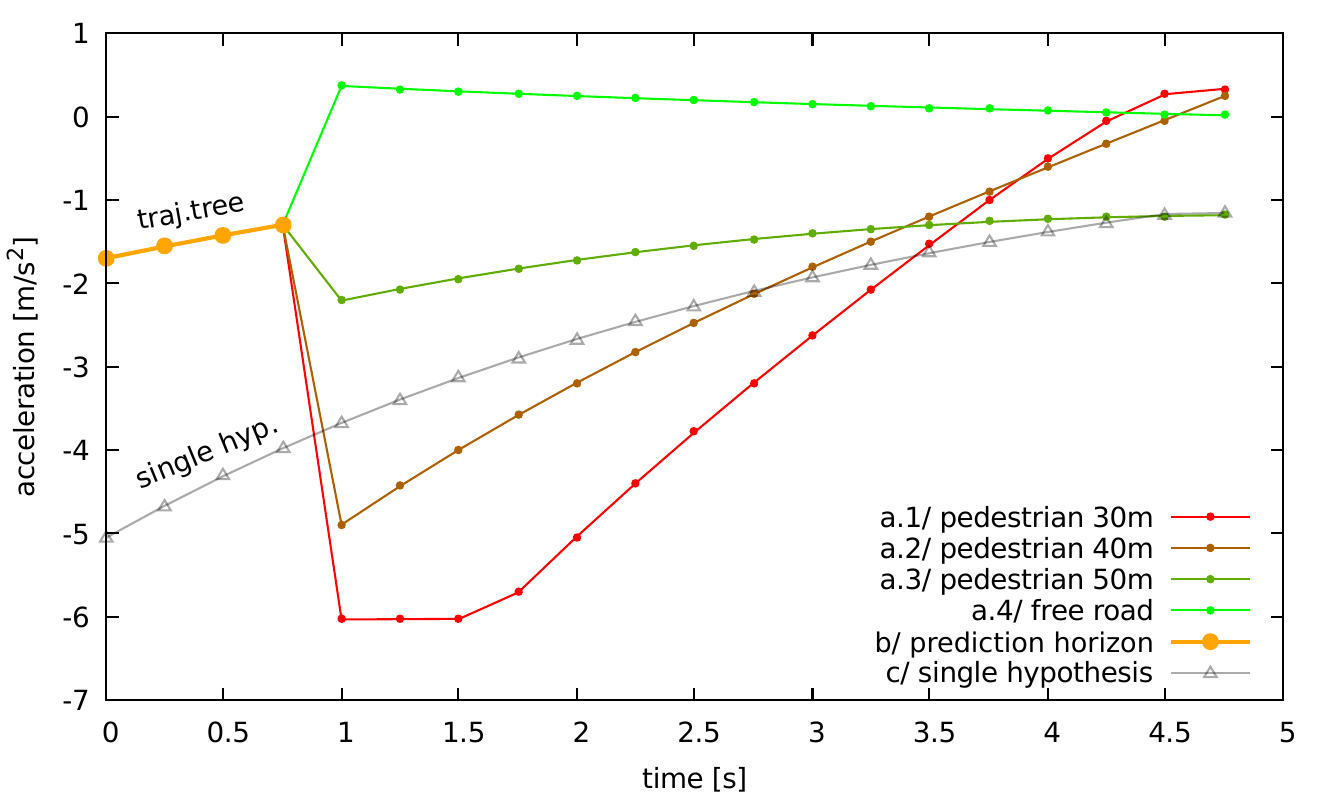}}
  \center{\includegraphics[width=0.45\textwidth]{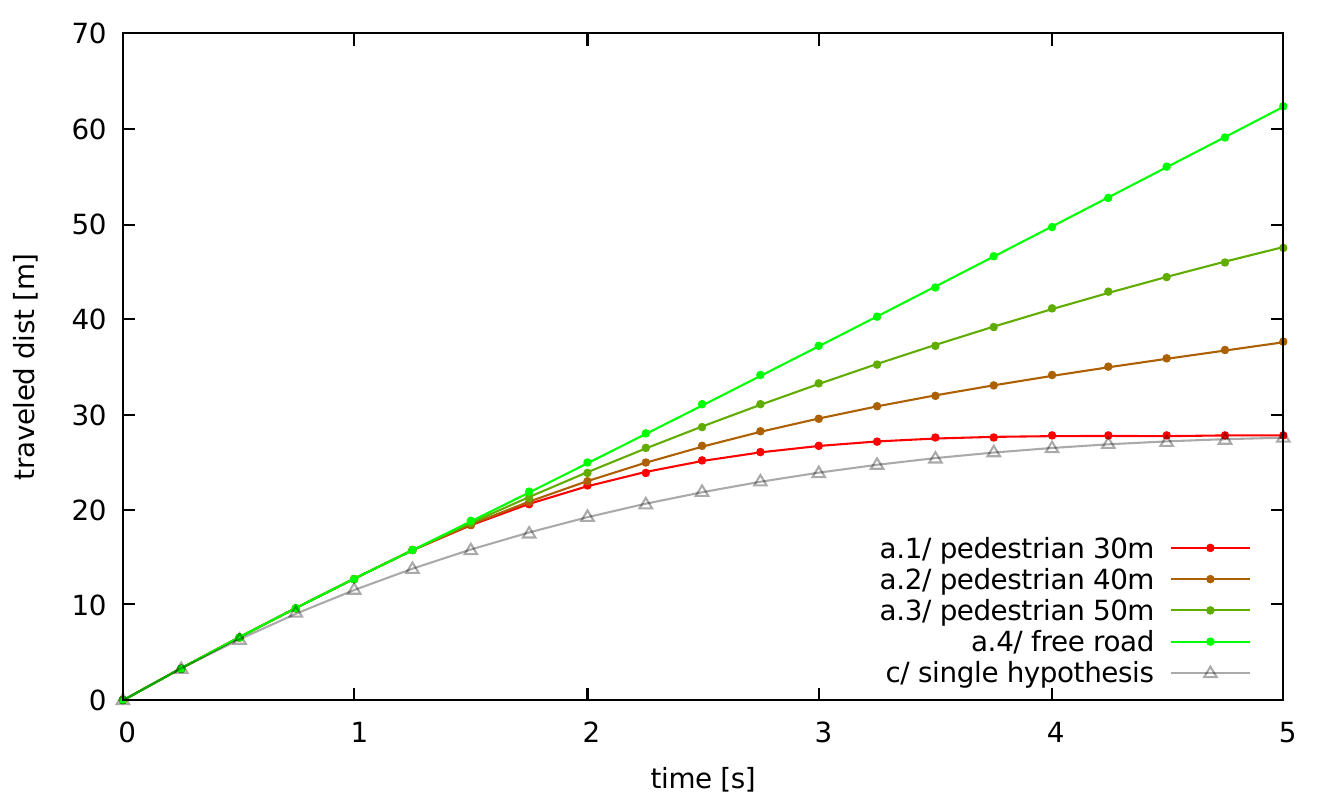}}
  \caption{Example of trajectory-tree when braking: a.1, a.2, and a.3 anticipate that a pedestrian crosses. a.4 corresponds to the free road scenario. c. is obtained with single hypothesis MPC assuming that the closest pedestrian crosses.}
\label{fig:control-tree-vs-sequential}
\end{figure}

\medskip
\noindent\textbf{Influence of the belief state\normalfont{:}}
Fig.~\ref{fig:belief-state} focuses on the branching horizon ($t\leq 1.0s$) and shows the influence of the crossing probability. When the probability is low, the planned control is more optimistic, whereas when this probability increases, the control policy becomes more conservative and tends to the single hypothesis case.

\begin{figure}[!htb]
 \center{\includegraphics[width=0.5\textwidth]{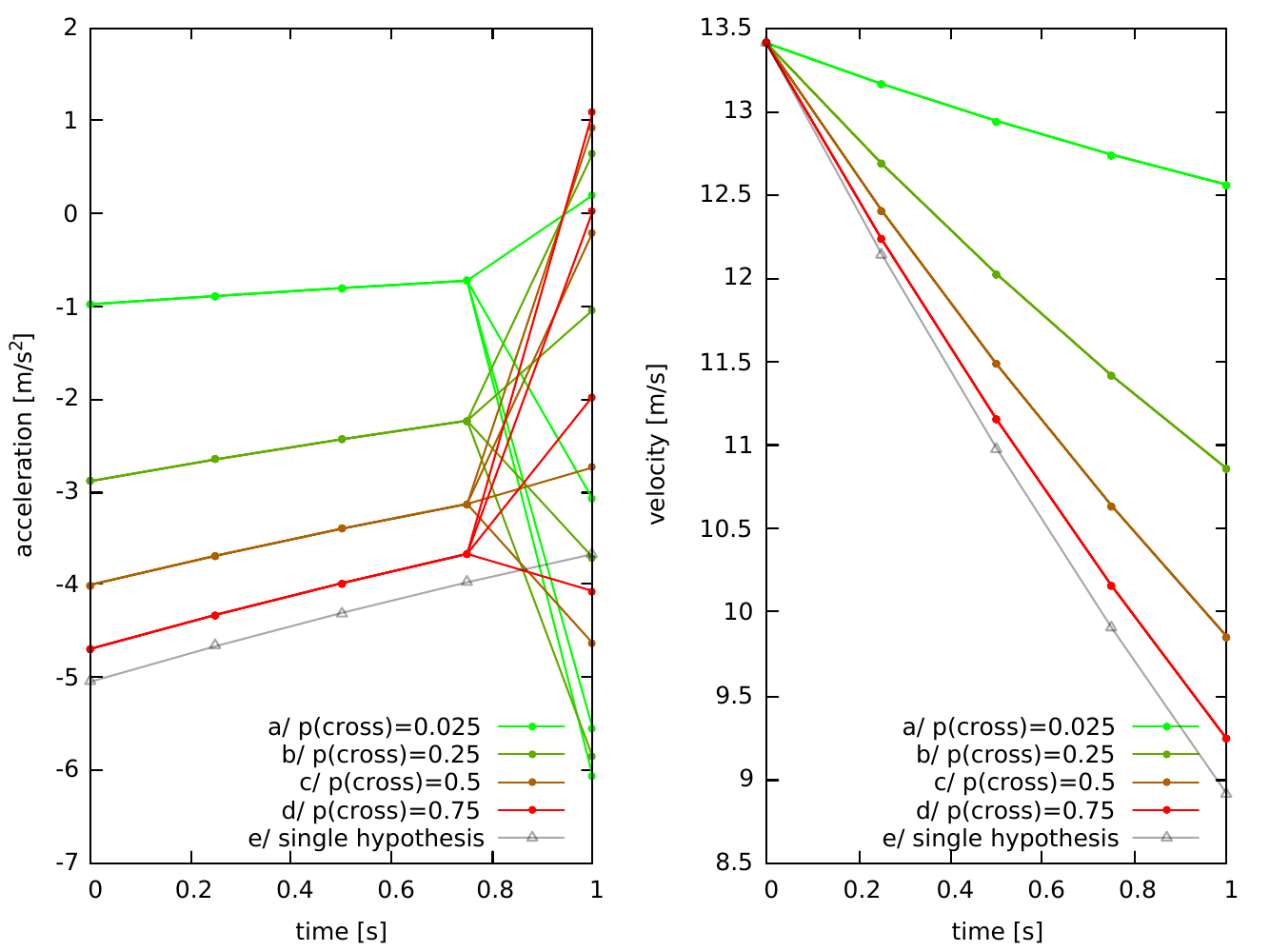}}
  \caption{Influence of the belief state on the braking within the branching horizon: Low crossing probabilities (e.g. a/, b/) lead to a more optimistic trajectory-tree.}
\label{fig:belief-state}
\end{figure}
\medskip
\noindent\textbf{Evaluation on random scenarios\normalfont{:}}
The algorithm is tested under various combinations of pedestrian density and pedestrian behavior (average crossing probability). Each run is performed over 30 minutes of simulated driving. We report in Table~\ref{tab:acc_costs} on the average costs (as defined by the matrices $\boldsymbol{Q}$ and $\boldsymbol{R}$) of the controls which are actually executed, until the next planning cycle occurs. To give a sense of the conservativeness of the car, we indicate the average velocity. 

In Table~\ref{tab:acc_costs}, we also indictate the performance obtained with simplified trajectory-trees. With 80 pedestrians per km, up to 4 pedestrians can enter the planning horizon such that 5 branches are needed (Tree-5). The variations (Tree-4, Tree-3 and Tree-2) are obtained with a simplified trees (having 4, 3 and 2 branches respectively) to evaluate the benefit of having larger trees versus the computation time. 
\begin{table}[h]
\footnotesize
\begin{center}
\begin{tabular}{|c||c|c||c|c||c|}
\hline
 & \thead{Pedestr.\\ per km} & \thead{Percentage\\of pedestr.\\crossing} & \thead{Avg.\\cost} & \thead{Avg.\\speed \\ (m/s)} & \thead{Planning\\time\\(ms)} \\
\hline
Tree-2 & 20 & 5\% & 30.8 & 10.7 & 5.62 \\
Single hyp. & 20 & 5\% & 65.9 & 7.8 & 2.77 \\
Full obs. & 20 & 5\% & 15.5 & 12.6 & 0.92 \\
\hline
Tree-2 & 20 & 25\% & 73.5 & 7.45 & 7.08 \\
Single hyp. & 20 & 25\% & 89.8 & 5.83 & 2.95 \\
Full obs. & 20 & 25\% & 68.3 & 8.25 & 2.02 \\
\hline
Tree-5 & 80 & 1\% & 46.5 & 8.61 & 20.0 \\
Tree-4 & 80 & 1\% & 47.5 & 8.59 & 16.45 \\
Tree-3 & 80 & 1\% & 48.4 & 8.53 & 13.7 \\
Tree-2 & 80 & 1\% & 53.2 & 8.24 & 10.1 \\
Single hyp. & 80 & 1\% & 93.5 & 5.02 & 4.08 \\
Full obs. & 80 & 1\% & 18.7 & 12.4 & 1.23 \\
\hline
Tree-5 & 80 & 5\% & 66.4 & 7.09 & 20.01 \\
Tree-4 & 80 & 5\% & 67.7 & 7.02 & 16.8 \\
Tree-3 & 80 & 5\% & 68.8 & 6.99 & 13.7 \\
Tree-2 & 80 & 5\% & 74.6 & 6.57 & 9.94 \\
Single hyp. & 80 & 5\% & 103.1 & 4.27 & 4.17 \\
Full obs. & 80 & 5\% & 48.3 & 9.90 & 2.23 \\
\hline
Tree-5 & 80 & 25\% & 115.7 & 3.92 & 21.1 \\
Tree-4 & 80 & 25\% & 117.7 & 3.71 & 16.4 \\
Tree-3 & 80 & 25\% & 116.4 & 3.80 & 13.4 \\
Tree-2 & 80 & 25\% & 120.0 & 3.47 & 9.11 \\
Single hyp. & 80 & 25\% & 127.6 & 2.86 & 4.11 \\
Full obs. & 80 & 25\% & 115.2 & 4.24 & 4.02 \\
\hline
\end{tabular}
\end{center}
\caption{Performance comparison: Trajectory-trees lead to lower costs than sequential trajectories (single hyp.). The full obs. baseline is an idealized case assuming perfect information.}
\label{tab:acc_costs}
\end{table}
\medskip
\noindent\textbf{Results interpretation\normalfont{:}}
Trajectory-trees lead to better control costs than the single hypothesis MPC (up to twice lower costs in the case of 20 pedestrians per km, and 5\% of crossing probability). The cost are significantly reduced and come closer to the ideal case with full observability. In particular, the car drives less conservatively, maintains a higher average velocity but still ensures that it is possible to come to a stop safely if a pedestrian would cross. The benefit is more visible when the crossing probability is low (1\% and 5\%). With 25\% of crossing probability and 4 pedestrians in the planning horizon, it is quite likely that at least one pedestrian will cross. This is reflected by the small difference between the single-hypothesis and full observable baselines.  


Trajectory-trees with the highest number of branches perform best. This is consistent since the problem structure is best modeled in that case. However, the cost improvement when adding a branch to the tree diminishes with the number of branches, e.g. the improvement is large between the Single-hypothesis case and Tree-2, but is small between Tree-4 and Tree-5. This suggests that, in this example, a simple tree can be a good trade-off between performance and computation time.

Trajectory-trees take longer to optimize, but computations remain fast and compatible with a real-time application. 

\subsubsection{Slalom among uncertain obstacles} \label{sec:mpc_slalom}
In this example, the ego-vehicle drives along a straight reference trajectory. Obstacles appear randomly. Obstacle detection is imperfect: there are \textit{false-positives} i.e. obstacles are detected although they do not exist. This captures a common problem when working with sensors like radars that can suffer from a high rate of false detections. The simulated detection module also outputs an existence probability for each detected object. The closer the car gets to the obstacle, the more reliable are the observations. Once the distance to obstacle becomes lower than a threshold (randomized in the simulation), the object becomes fully observed: it gets a probability of 1.0 or disappears.

Optimization is performed in configuration space using the \textit{K-Order-Motion-Optimization} (KOMO) formulation \citep{komo_solver, komo_17-toussaint-Newton}. \revise{The prediction horizon is set to \SI{5}{\second} with 4 steps per second, and a branching horizon of \SI{1}{\second}}.
\begin{figure}[ht]
        \centering
  \centering
  \includegraphics[width=0.95\linewidth]{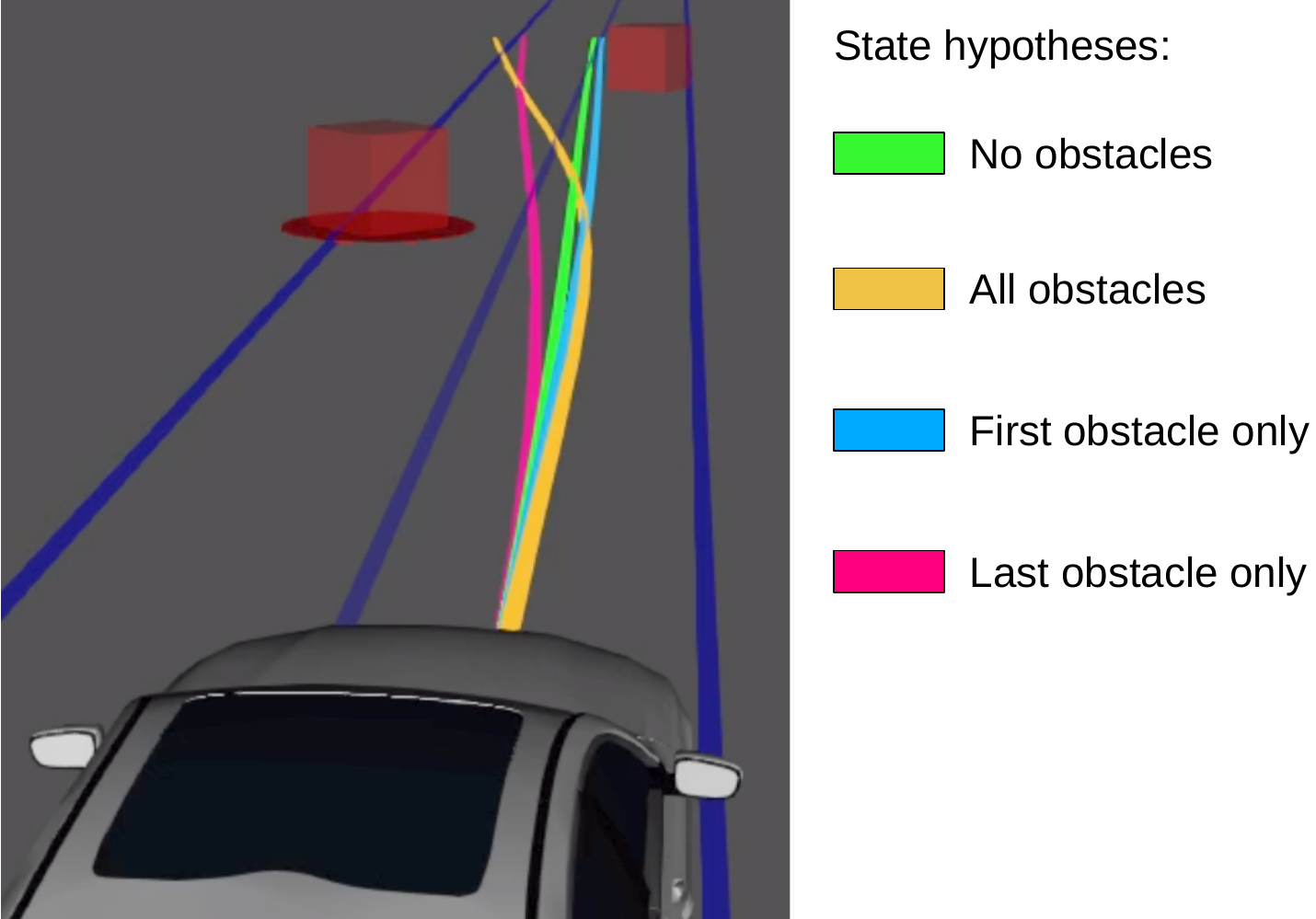}  
\caption{Avoidance of 2 uncertain obstacles: Each trajectory corresponds to a different combination of object presences.}
\label{fig:avoidance_tree_vs_linear}
\end{figure}

\medskip
\noindent\textbf{State space and kinematics\normalfont{:}} \label{section:non-holonomic}
Optimization is performed in $SE(2)$. No slippage is assumed, this is enforced by the following non-holonomic constraint:
\begin{subequations}
\begin{align}
 \dot{x}(t) cos(\theta) - \dot{y}(t) \sin(\theta) = 0, \forall{t}, \label{eq:non-holonomic}
\end{align}
\end{subequations}
where $\begin{pmatrix} x, y, \theta \end{pmatrix} $ is the vehicle pose.

\medskip
\noindent\textbf{Partially observable discrete state\normalfont{:}}
The discrete states represent the possible combinations of obstacle existences. With 2 uncertain obstacles in the planning horizon, there are 4 possible states (see Fig.~\ref{fig:avoidance_tree_vs_linear}).

\medskip
\noindent\textbf{Costs and constraints\normalfont{:}}
The following trajectory cost terms are minimized:
\begin{itemize}
\item \textbf{Acceleration}: The square accelerations $\ddot{x}$ and $\ddot{y}$ as well as the square angular acceleration $\ddot{\theta}$.
\item \textbf{Distance to centerline}: The square distance to a reference line.
\item \textbf{Speed}: The square difference to a desired velocity.
\end{itemize}

\noindent In addition, the following constraints apply:
\begin{itemize}
\item \textbf{Kinematics}: The non-holonomic constraint, see ~\eqref{eq:non-holonomic}.
\item \textbf{Collision Avoidance}: The distance to obstacles must stay greater than a safety distance.
\end{itemize}

Unlike the previous example, the non-holonomic and collision avoidance constraints make the problem \textit{non-convex}.

Gradients are computed analytically while we use the gauss-newton approximation of the \nth{2} order derivatives.

\medskip
\noindent\textbf{Evaluation on random scenarios\normalfont{:}}
Table~\ref{tab:table_obstacle_avoidance} gathers results obtained when simulating 30 minutes of driving with, on average, one potential obstacle every 17 meters resulting overall in 900 uncertain obstacles being encountered. As in the previous example, we compare against two baselines: First, the single-hypothesis MPC which does not use the belief state information and considers all obstacles the same way, regardless of their existence probabilities. Second, the idealized full observability case, without uncertainty. 
\begin{table}[h]
\begin{center}
\footnotesize
\begin{tabular}{|c||c|c||c|c||c|}
\hline
  & \thead{Percentage of \\ false positive} & \thead{Avg.\\cost} & \thead{Planning\\time (ms)} \\
\hline
Tree-4  & 90\% & 0.56 & 36.6 \\
Single hyp. & 90\% & 0.95 & 22.4 \\
Full obs. & 90\% & 0.35 & 11.5 \\
\hline
Tree-4  & 75\% & 1.21 & 49.1 \\
Single hyp. & 75\% & 1.79 & 30.0 \\
Full obs. & 75\% & 1.01 & 19.2 \\
\hline
\end{tabular}
\end{center}
\caption{Performance comparison: Trajectory-trees lead to significantly lower trajectory costs. The fully observable case is an idealized case with full information.}
\label{tab:table_obstacle_avoidance}
\end{table}

\medskip
\noindent\textbf{Results interpretation\normalfont{:}}
Trajectory-trees consistently result in lower costs within the branching horizon compared to the single-hypothesis baseline. When taking the idealized fully-observable case as a base, the cost reduction amounts to 65\% and 74\% in scenarios with false positive probabilities of 90\% and 75\%, respectively. The vehicle exhibits reduced lateral acceleration while successfully avoiding all true obstacles.


Computation time required for trajectory-trees is higher than for sequential trajectories but remains largely compatible with the planning frequency of \SI{10}{\hertz}. Optimization time is higher in the case with 75\% of false positive, reflecting the fact that more obstacles are encountered, which results in a higher number of iterations required by the solver to satisfy the collision avoidance constraint.

\subsubsection{Benefit of Optimization Decomposition (D-AuLa)}\label{sec:benefit_decomposition}
Optimization times are compared against the scenario where trajectory-tree optimization is not decomposed and treated as a single large joint optimization problem. The undecomposed problem is solved using the same solver, which in this case, effectively reduces to the standard Augmented Lagrangian method. Tests were conducted on a Intel\textsuperscript{\textregistered} Core\textsuperscript{\texttrademark} i5-8300H. The results are shown in Fig.~\ref{fig:daula_vs_undecomposed}.


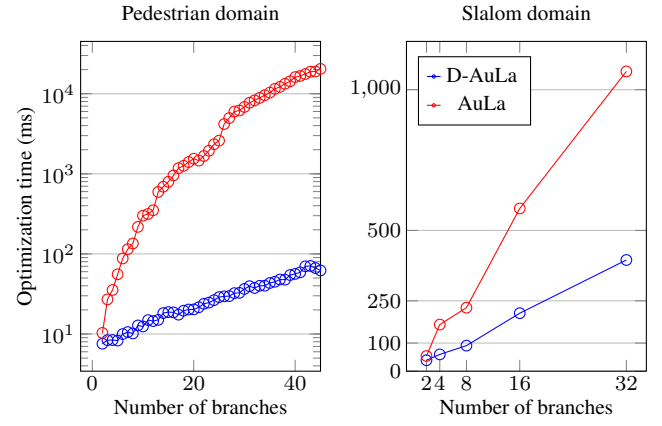
\begin{figure}[ht]
    \centering
\subfloat[Pedestrian domain: The number of branches is proportional to the number of pedestrians\label{fig:pedesrian_runtimes}]{
\begin{tikzpicture}
    \centering
\begin{axis}[
	outer sep=0pt,
	inner sep=2pt,
	ymode=log, 
    width=0.56\linewidth,
    height=0.7\linewidth,
    title={Pedestrian domain},
    xlabel={Number of branches},
    ylabel={Optimization time (ms)},
    ymajorgrids=true,
    x label style={at={(axis description cs:0.5,0.05)},anchor=north},
    ylabel style={at={(0.1,0.5)}, anchor=north},
    xmax=45,
    legend columns=1,
    legend image post style={scale=0.3} 
]
   
        \addplot[blue, mark=o] table [col sep=space]{plots/1_pedestrian/execution_time_dec_qp_100.dat};
        
        \addplot[red, mark=o] table [col sep=space]{plots/1_pedestrian/execution_time_joint_qp_100.dat};

\end{axis}
\end{tikzpicture}
}
\hfill
\subfloat[Slalom domain: The number of branches is $2^n$ ($n$ is the number of obstacles)\label{fig:slalom_runtimes}]{\begin{tikzpicture}
\begin{axis}[
	outer sep=0pt,
	inner sep=2pt,
    width=0.56\linewidth,
    height=0.7\linewidth,
    xtick={2, 4, 8, 16, 32},
    ytick={0, 100, 250, 500, 1000},
    title={Slalom domain},
    xlabel={Number of branches},
    ymajorgrids=true,
    ymin=0,
    x label style={at={(axis description cs:0.5,0.05)},anchor=north},
    ylabel style={at={(0.05,0.5)}, anchor=north},
    legend columns=1,
    legend image post style={scale=0.3}, 
    legend style={at={(0.05,0.85)}, anchor=west, legend columns=1}
]
\addplot[blue, mark=o] table [col sep=space]{plots/2_slalom/execution_time_dec.dat};
\addlegendentry{D-AuLa}
    
\addplot[red, mark=o] table [col sep=space]{plots/2_slalom/execution_time_dec_joint.dat};
\addlegendentry{AuLa}
    
\end{axis}
\end{tikzpicture}}
\caption{\label{fig:daula_vs_undecomposed} Decomposed (D-AuLa) vs. undecomposed (AuLa) optimization: D-AuLa scales exponentially better in the pedestrian case \ref{fig:pedesrian_runtimes}. In the slalom case \ref{fig:slalom_runtimes}, it is faster by a factor 2.75.}
\end{figure}

D-AuLa scales nearly linearly with respect to the number of branches. This can be understood easily; one optimizes as many subproblems as branches in the tree, but the size of each subproblem does not change. 

In the pedestrian example, we use a single shooting transcription. When optimizing sequential trajectories this transcriptions leads to Hessian matrices of the Augmented Lagrangian which are dense. Since the $N$ subroblems are sequential trajectories, the $N$ Hessian matrices are dense. Optimization time is dominated by the computations and Cholesky decompositions these Hessian matrices, which reults in a complexity $\mathcal{O}(N \times T^3)$ which is  cubic with respect to the number of trajectory elements $T$, but linear with respect to the number of subproblems $N$.

When optimizing jointly (AuLa), we use the same solver setting with dense matrix arithmetic, the total number of elements optimized is $L + N \times (T - L)$, therefore leading to a complexity of $\mathcal{O}((T \times N)^3)$, i.e. cubic with respect to the number of branches $N$. It is however, worth noting, that the Hessian of the undecomposed problem becomes sparse for $N>1$, and sparsity increases with $N$, such that one could in principle use sparse matrix arithmetic to improve undecomposed optimization.

When optimizing the undecomposed problem with an off-the-shelf QP solver OSQP~\citep{stellato2020_osqp} specializing in large sparse QPs, we observe that D-AuLa is faster but only for a high number of branches ($N\geq400$). For $N=600$, D-AuLa is faster by a factor $7$ with an optimization time of $\SI{1.8}{\second}$ compared to $\SI{14.3}{\second}$ for OSQP. This updates the measurements obtained with OSQP provided in \citep{mpc_0_phiquepal2021control} which were including expensive matrix manipulations, taking place outside of the OSQP solver as part of the baseline measurement. The lower optimization times obtained with OSQP for smaller values of $N$ can be attributed, at least in part, to the fact that D-AuLa is not specifically tailored for QP problems and does not exploit certain structural properties inherent to QPs (e.g. constancy of Hessian). It therefore requires a high number of branches such that the benefits of the decomposition with D-AuLa overcome a specialized solver.

In the slalom example, D-AuLa is faster by a constant ratio (approximately 2.75), indicating that both exhibit the same complexity. It is consistent with the theoritical complexity of the trajectory-tree optimization with the T-KOMO transcription as will be explained in Section~\ref{sec:tree_komo}. Optimization time scales linearly with respect to the number of trajectory elements, regardless of the linear or tree structure. Decomposing the optimization brings the benefit of parralelization, but does not change the complexity.
\subsection{Discussion} \label{sec:mpc_discussion}
These experiments show that the PO-MPC approach leads to lower control costs than a sequential MPC approach, without compromising the robustness of the constraints satisfaction. This translates into a car driving less conservatively without compromising safety.

The increased optimization time is significantly mitigated by the D-AuLa method which leverages the decomposability of the PO-MPC control policies. Optimization time is compatible with real-time planning for a moderate number of branches in the trajectory-tree (e.g. five branches in the first example, and four in the second example). D-AuLa is a general method applicable both to linear and non-linear MPC.  

The first experiment shows that a large part of the performance gain is already obtained with a trajectory-tree with two branches, and the performance improvement is comparatively smaller with five branches. This suggests that, in general, adopting a simplified trajectory-tree, i.e. clustering modalities, can represent a favorable trade-off between performance and computation time.

An interesting extension of this work could involve specializing D-AuLa to QP problems, which are common in domains like autonomous driving. \revise{Another relevant extension would be to generalize the PO-MPC approach beyond deterministic continuous dynamics. This could be achieved by coupling the presented approach to Tube MPC techniques to account for continuous but statistically bounded disturbances within each modality.}

\added{Finally, the proposed framework could be extended by explicitly reasoning about future belief evolution during planning, rather than relying on a fixed determinization after a single branching stage. Such an extension would allow anticipated future observations as well as the effects of interactions with other agents to be incorporated into the predicted belief evolution, thereby supporting richer behaviors. This would generally lead to increasingly complex trajectory-tree structures. Scalable mechanisms for constructing trajectory-trees with partial contingency coverage would consequently become increasingly important. Possible directions toward such an extension are discussed further in Section~\ref{sec:towards_partial_contingencies}.}

\section{Task and Motion Planning (TAMP)} \label{sec:tamp}
We now develop the trajectory-tree optimization method for partially observable Task and Motion Planning problems. Unlike MPC where the task does not change over the planning horizon, there is here an additional layer of task planning. The goal of the presented approach (PO-LGP), is to determine both the symbolic actions, or policy, along with the corresponding continuous trajectory-tree which optimally solve the partially observable planning problem. 

\subsection{PO-LGP Problem Formulation}

We formulate the problem in terms of a decision tree and, for every action edge in this tree, cost and constraint objectives on the continuous trajectory. The decision tree is in belief space, including action and observation branchings.

To define a problem we first define the symbolic partially observable decision process that spans this tree. Second, we define the cost and constraints functions that define the optimization objectives for the continuous trajectory optimization problem associated with transitioning through this tree. 

\revise{
An optimal trajectory-tree then represents a reactive policy
that transitions the tree depending on observations, smoothly switching into different motion options.
}
\subsubsection{State Representation}
\revise{
We use the hybrid, partially observable state representation introduced in Section~\ref{sec:intro_trajectory_trees}.}

\medskip
\revise{\noindent\textbf{Continuous state\normalfont{:}} As typical in the LGP context, the continuous state is \( x \in \mathcal{X} \), where \( \mathcal{X} \) is the system configuration space. We can write \( \mathcal{X} = \mathbb{R}^n \times SE(3)^m \), for an \( n \)-DOF robot interacting with \( m \) rigid objects. While the continuous state space includes both the robot and environment configurations, only the robot’s joint trajectory is directly optimized. The environment evolves implicitly through kinematic constraints imposed by the symbolic plan structure\footnote{\revise{In scenarios involving underactuated dynamics---such as non-prehensile manipulation---the non-actuated DOFs are likewise included as optimization variables, subject to cost and constraint functions that model the physical interactions, as is standard in LGP, see e.g. \citep{tamp_lgp_3_18-toussaint-RSS}.}}. 
Although, this definition does not explicitly represent dynamic quantities such as velocity or acceleration, these are implicitly accounted for through cost and constraint functions that operate on consecutive configuration tuples along the trajectory, as further detailed in Section~\ref{sec:action_cost}.}

\medskip
\revise{\noindent\textbf{Symbolic state\normalfont{:}} The symbolic state $s$ represents discrete aspects of the planning problem. Unlike in the MPC case where the robot does not control the symbolic state, it is here coupled to a decision process with symbolic actions, as described in the next section.}

\medskip
\noindent \revise{In the example introduced in Fig.~\ref{fig:example_problem_tamp}, the continuous state is the robot's and blocks' configurations. The symbolic state models the color of a block’s colored side, as well as structural relations, such as which blocks are stacked on others. The orientation of the colored side of the block (a continuous variable) and its color (a symbolic variable) are not directly observable. Figure~\ref{fig:possible_initial_configurations} illustrates a state hypotheses in case of a scene with two blocks.
}

\begin{figure}[h]
\centering
\begin{minipage}[c]{0.25\linewidth}
    \centering
    \includegraphics[width=0.78\linewidth]{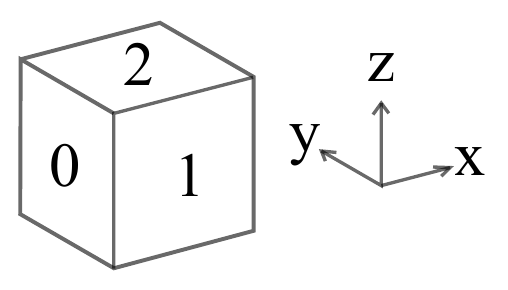}
    \includegraphics[width=0.78\linewidth]{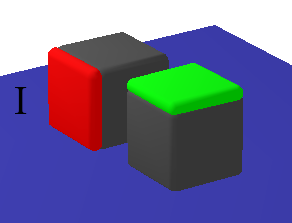}
    \includegraphics[width=0.78\linewidth]{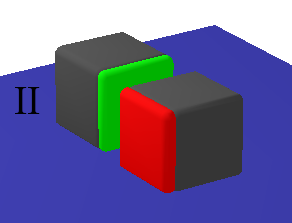}
\end{minipage}
\begin{minipage}[c]{0.7\linewidth}
    \centering
    \footnotesize
    \setlength\tabcolsep{3pt}
    \begin{tabular}{|c|c|c|}
        \hline
        Hyp. & \thead{Continuous \\ $x^{\text{latent}}$} & \thead{Symbolic \\ $s^{\text{latent}}$} \\ 
        \hline
        I & \makecell[l]{$\mathbf{n_0} = (-1, 0, 0)$\\ \\$\mathbf{n_1} = (0, 0, 1)$} & \makecell[l]{(colored block\_0 side\_0)\\(block\_0 red) \\ (colored block\_1 side\_2)\\(block\_1 green)} \\ 
        \hline
		II & \makecell[l]{$\mathbf{n_0} = (0, -1, 0)$\\ \\$\mathbf{n_1} = (-1, 0, 0)$} & \makecell[l]{(colored block\_0 side\_1)\\(block\_0 green) \\ (colored block\_1 side\_0)\\(block\_1 red)} \\
        \hline
    \end{tabular}
\end{minipage}

\caption{\revise{Example of hypotheses: The continuous part represents the normal vector of the colored face. It corresponds to a logical fact in the symbolic part indicating which side is colored. The symbolic state also contains the blocks' colors. There are 72 hypotheses (asuming each side can be the colored side).}}
\label{fig:possible_initial_configurations}
\end{figure}

\medskip
\revise{\noindent\textbf{Relation between the continuous and symbolic latent states\normalfont{:}} We assume that the continuous state hypotheses are linked to symbolic symbols. This linkage is necessary because the observation model of the decision process (described in the next section) is defined over symbolic states. Thus, inference about continuous aspects of the environment is performed through their symbolic counterparts. For example, in Fig.~\ref{fig:possible_initial_configurations}, the orientation of a block’s colored side is encoded symbolically. 
}




\subsubsection{Assumed Decision Process} \label{sec:tamp_decision_process}

We model the decision process by a 7-tuple $(S, \mathcal{X}, A, T, \Omega, O, C, G)$, where,
\begin{itemize}
\item $S$ is a finite set of symbolic states
\item $\mathcal{X}$ is the continuous state space 
\item $A$ is a finite set of symbolic actions
\item $T$ is a transition model giving the next state $s' = T(s, a)$ of $s$ after taking action $a$ 
\item $\Omega$ is a finite set of observations
\item $O$ is an observation model giving the observation $o = O(s, a)$ after taking action $a$ and reaching state $s$
\item \revise{$C$ is a cost model defined by tuples of cost and constraints functions $(c_a, h_a, g_a)$, which evaluate trajectories implementing an action $a \in A$
\item $G$ is a set of goal conditions defined over the symbolic state space $S$. A belief state is terminal if all its hypotheses with non-zero probability are terminal}
\end{itemize}
\revise{
This is a POMDP with the following distinctions:
\begin{itemize}
\item \textbf{Implicit cost-based objective}: The optimization objective is defined in terms of cost minimization rather than reward maximization. Importantly, the action costs are not specified numerically, but are instead defined implicitly through the cost and constraint functions applying on the trajectory-tree. Evaluating an action's cost thus entails solving a constrained optimization problem.
\item \textbf{Deterministic transitions}:
The action transition model $T$, as well as the observation model $O$, map each state-action pair to a next state and observation, respectively. This corresponds to the assumption that the action and observation processes are deterministic, which is less general than the standard POMDP formulation, where transitions and observations are typically stochastic.
\end{itemize}
The assumption of deterministic symbolic actions parallels the earlier assumption of deterministic continuous-state dynamics introduced in Section~\ref{sec:intro_trajectory_tree_definition} but at the symbolic level. The assumption of deterministic observations limits the number of reachable belief states, and thereby the size of trajectory-trees. It also allows the goal conditions $G$ to be defined over the state space directly, instead of being predicates in belief space. We discuss in Section~\ref{sec:tamp_generalizablization} how the framework can be generalized beyond those assumptions.
}

In the experiments, we use a first-order logic language, similar to PDDL to model the decision process. \revise{As is common in PDDL-based representations, actions are associated with explicit preconditions. For instance, a robot’s gripper must be free before a grasp action can be executed. These preconditions do not appear explicitly in the POMDP formalism introduced above, where the actions are assumed to be available in every state. However, preconditions can still be captured, by interpreting actions whose preconditions are not satisfied as incurring infinite cost.}
 
\subsubsection{Decision Tree}
Given an initial belief $b_0$, the decision process spans an AND/OR decision tree with each node corresonding to a belief state. The tree contains two kinds of nodes:
\begin{itemize}
\item In \emph{action nodes} the agent chooses an action 
\item In \emph{observation nodes} the agent receives an observation
\end{itemize}

At action nodes, the probability distribution of the belief state does not change, but the symbolic state of each hypothesis $m \in \mathcal{H}$ is updated by applying the action $a$, \revise{i.e,
\begin{align*}
s'_m = T(s_m, a).
\end{align*}}
\rrevise{At observation nodes, multiple observations may be possible, giving rise to distinct contingencies. This branching occurs despite the observation model being deterministic and reflects the underlying uncertainty in the belief state. Upon reception of an observation $o$,} \revise{the probability of each hypothesis $m$ is updated according to:
\begin{align*}
b'(m) = \frac{p(o \mid s'_{m}, a) b(m)}{\sum_{m \in H} p(o \mid s'_{m}, a) b(m)},
\end{align*}
where $b'(m)$ denotes the updated probability of hypothesis $m$, and $p(o \mid s'_{m}, a)$ is the probability of receiving observation $o$. Since the observation model $O$ is deterministic, $p(o \mid s'_{m}, a)$ is either $0.0$ or $1.0$.} \rrevise{ Hypotheses incompatible with the observation $o$ are assigned zero likelihood. The entropy of the belief state therefore decreases after each observation, and the belief eventually collapses to a single hypothesis with probability $1.0$ once all initially uncertain aspects have been observed.}\footnote{\rrevise{In small problems, a single observation may suffice to fully resolve the belief state, as in Fig.~\ref{fig:example_policy}. In general, however, multiple observation stages are needed to reveal all initially latent aspects (see Fig.~\ref{fig:planned_policies}).}} 

%
\revise{An equivalent view on the problem is to model it as a Markov Decision Process (MDP) in belief space, where observation branching is not formally linked to observations, but tied to the action stochasticity. While formally equivalent, we favor the description using the POMDP formalism (albeit with deterministic transition), as it highlights more clearly the explicit role of observations and information gathering in the decision process}.

\medskip
\noindent\textbf{Example of decision tree\normalfont{:}}
Consider a simplified version of the example introduced in Fig.~\ref{fig:example_problem_tamp} with only the green and blue blocks. The robot can pick up blocks, place them on the top of another block or on the table. The robot also has a \textit{Look} action to observe a block's color. The Fig.~\ref{fig:decision_tree} shows the decision tree at an early stage of its expansion.
At the node~(4), the color of the picked up blocked is observed which leads to two outcoming edges corresponding to the two possibilites.
In this example, observation branching is sparse. Only the observation nodes following a \textit{Look} action have more than one child.
\begin{figure}[ht]
    \centering
       \includegraphics[width=1.0\linewidth]{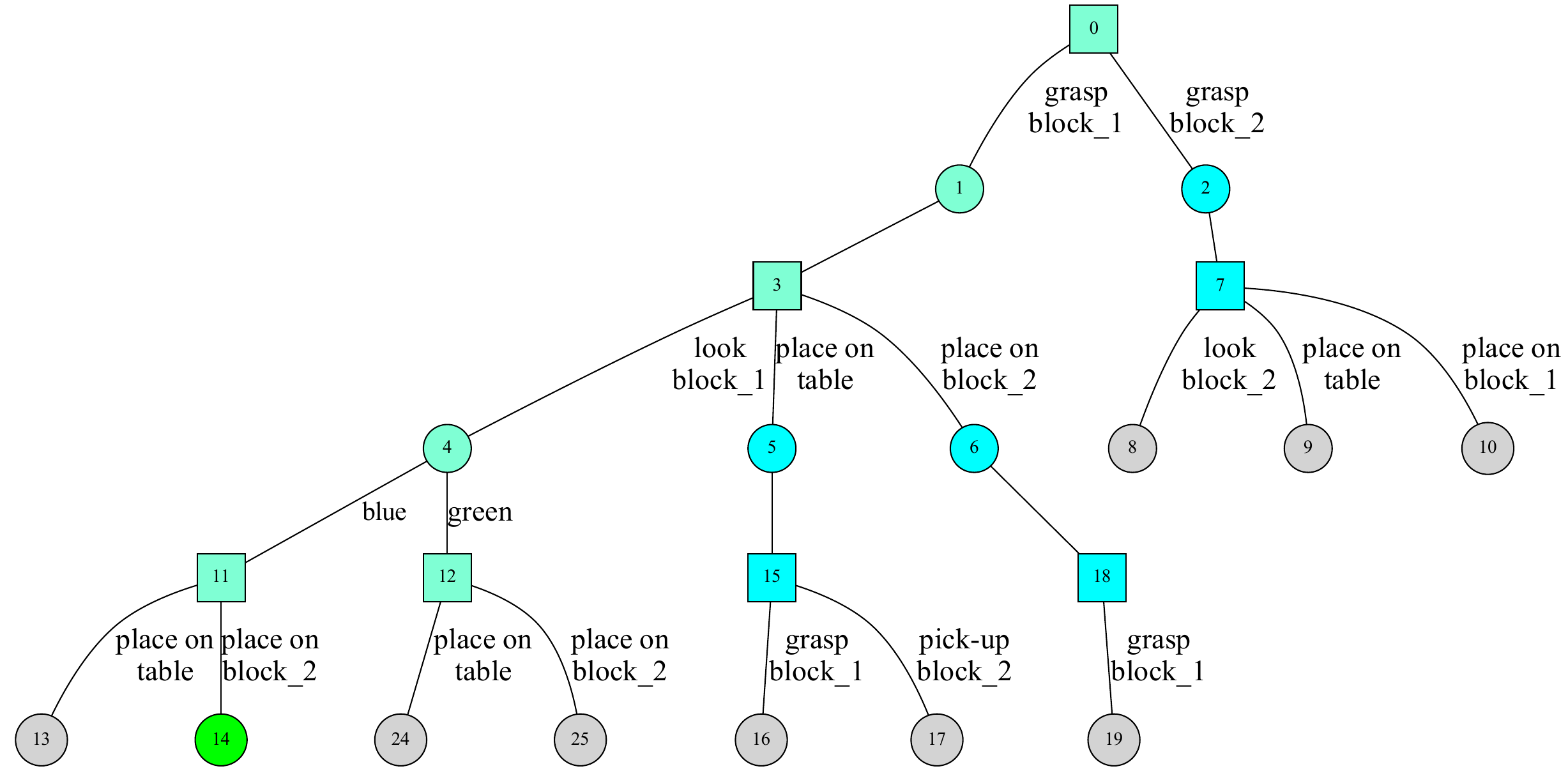}
  \caption{Decision tree for the block stacking problem with 2 blocks: The square nodes are action nodes, circular nodes are observation nodes. At the node (4) the color of the block is observed leading to two contingencies. The node (14) in green fulfills the goal condition. The nodes in gray are not yet expanded.}
  \label{fig:decision_tree} 
\end{figure}

\medskip
\noindent\textbf{Decision tree vs. decision graph\normalfont{:}} \label{sec:tree_vs_graph}
Some nodes in the decision tree may be symbolically equivalent like the root node~(0) and the node~(5) in Fig.~\ref{fig:decision_tree}: the robot picked up a block only to put it back on the table. However, in general, the geometric states are different in these two nodes since the optimal position of the block after placing it back on the table may differ from its initial position. We therefore represent the decision process as a tree, which allows us to attach geometric information to edges (cost and trajectory pieces as described in Section~\ref{sec:dynammic_programming} and \ref{sec:piecewise_optimization}) without loss of generality. This is in contrast to \citep{tamp_0_phiquepal2019combined} where it was represented as a graph, thereby limiting the geometric configurations explored during the search. The tree representation increases the search space for task planning which we mitigate by adopting a Monte-Carlo sampling (see Section~\ref{sec:sample_decision_tree}).

\medskip
\noindent\textbf{Candidate policies\normalfont{:}}
The node~(14) in Fig.~\ref{fig:decision_tree} is a terminal, since the block colors are known, and the blue block is stacked on the green one. Finding a terminal node is however not enough to solve the planning problem, a solution needs to cover all contingences. This implies that the policies will be themselves tree-like.

More formally, a policy $\pi$ is a mapping from a symbolic belief state to an action. Since the belief space dynamics are described by the decision tree, $\pi$ maps decision tree nodes to actions. A policy reaching the goal condition, on the symbolic level, is a sub-set of the decision tree, where each action node has only one outcoming edge (the action has been decided), and which leaves are terminal nodes. We call this a candidate policy. 
\begin{figure}[ht]
    \centering
       \includegraphics[width=1.0\linewidth]{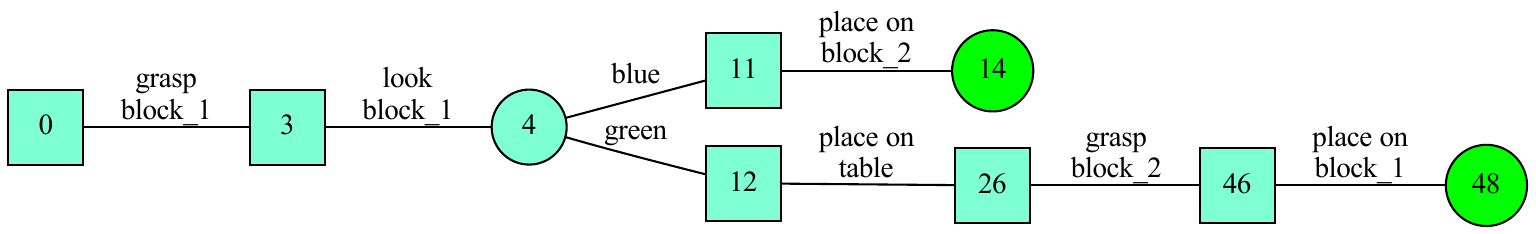}
  \caption{Candicate policy for the stacking problem with 2 blocks: A policy is a subset of the decision tree~\ref{fig:decision_tree}. Observation nodes with only one outcoming edge are omitted for clarity.}
  \label{fig:example_policy} 
\end{figure}
Fig.~\ref{fig:example_policy} shows a candidate policy. \revise{Candidate policies form the symbolic part of a trajectory-tree as defined in Section~\ref{sec:intro_trajectory_trees}. The procedure for extracting candidate policies from the decision tree is detailed in Section~\ref{sec:dynammic_programming}.
}

\subsubsection{Trajectory-Trees}\label{sec:action_cost}
The candidate policies defined above, reach the goal conditions on the symbolic level, but are not yet refined into a full trajectory-tree. \revise{As there are no costs or rewards defined solely at the symbolic level, candidate policies cannot be directly compared. However, at the geometric level, different candidate policies will lead to different trajectory costs, and some may prove infeasible. Here, we define the trajectory-tree optimization problem associated with a candidate policy.}

In typical trajectory optimization, the objective is given as a set of constraints and sum of cost terms along the trajectory. In our setting, we generalize this to a set of constraints and sum of cost terms for each action edge in the tree, weighted by the probability of being in the corresponding belief state.

\revise{Let $a \in A$ be the action taken by the agent in belief state $b$ during the time interval $[t_k, t_{k+1}]$. This action implies cost and constraint functions $c_a$, $g_a$ and $h_a$ on the trajectory.}
%

\revise{Let $x^{\text{known}}(t)$ denote the components of the state whose values are known at time $t$. This includes all observable components $x^{\text{observable}}(t)$, as well as latent components $x^{\text{latent}}(t)$ whose values have been resolved through belief inference. We then define the aggregated vector $\mathbf{x}^{\text{known}}(t)$ as the collection of $x^{\text{known}}(t)$ and its time derivatives up to order $k \in \mathbb{N}$, i.e.,
\begin{align*}
\mathbf{x}^{\text{known}}(t) = \begin{pmatrix}
x^{\text{known}}(t) \\
\dot{x}^{\text{known}}(t) \\
\vdots \\
x^{\text{known}(k)}(t)
\end{pmatrix}.
\end{align*}}

\noindent \revise{We define the cost of action $a$ in belief space $b$ as,
\begin{subequations} \label{eq:attachement_action}
\begin{align}  
 c(a, b) = \int_{t_{k}}^{t_{k+1}} &c_a(\mathbf{x}^{\text{known}}(t)) \ dt \\
 & \notag \\
 \text{if}\ \ & h_a(\mathbf{x}^{\text{known}}(t)) = 0,\label{eq:attachement_dynamic}\\
 & g_a(\mathbf{x}^{\text{known}}(t)) \leq 0,\label{eq:attachement_constraint-term}
\end{align}
\end{subequations}
and $c(a, b) = +\infty$ if any of the constraints are not satisfied.} By defining the costs over $\mathbf{x}^{\text{known}}$ we restrict the robot to act on parts of the system which are either fully observable, or which have been discovered earlier on the trajectory-tree via observations. \revise{The number of derivative $k$ is a global parameter. Actions may be primarily defined over the configurations (i.e., zero-order), as is the case for the observation actions used in the experiments (Section~\ref{sec:observative_actions}). However, all actions typically also involve higher-order cost and constraint terms to ensure smoothness and adherence to physical limits. In the experiments, we use $k=2$, and squared joint accelerations are minimized throughout the entire trajectory-tree, alongside action-specific costs and constraints.}

This definition links the symbolic decision making in belief space to trajectory costs.
A policy $\pi$, which is a tree of actions defines an optimization problem on a trajectory-tree. We use $\psi$ to denote a continuous trajectory-tree.

\revise{The cost and constraints functions are defined by the symbolic action $a$, in contrast to the PO-MPC method, where they are directly linked to the state $s$. This reflects the fact that in PO-MPC, the agent does not take symbolic actions, but rather plans reactively based on the current belief state. In contrast, PO-LGP explicitly incorporates symbolic actions, which serve as the primary drivers of the robot’s motion. The dependency on the state remains, but is implicit via the preconditions necessary to undertake an action $a$ in state $s$.}
\subsubsection{Optimal Trajectory-Tree}
We can now define the problem as finding a symbolic
policy $\pi$ and a continuous trajectory-tree $\psi$ that minimize the expected costs given the initial belief state $b_0$,

\begin{equation} \label{eq:tamp_optimization_objective}
 \Pi^{\star} = \min_{\pi, \psi}\ \sum_{b \in \pi} p(b, b_0) c(\pi(b), \psi(b)).
\end{equation}
Here, the expectation is with respect to the probability $p(b, b_0)$ of visiting a belief node in the decision tree.
Together $(\pi^{\star}, \psi^{\star})$ form the optimal trajectory-tree $\Pi^{\star}$.
\subsection{PO-LPG Planner}
We propose a planner that works in three stages, schematized in Fig.~\ref{fig:algo} and detailed in Algorithm~\ref{alg:tamp_outer_loop}. First, the decision tree is expanded (line~\ref{alg:sample_decision_tree}, and developed in Section~{\ref{sec:sample_decision_tree}). Second, we alternate Dynamic Programming on the decision tree and piecewise trajectory optimization to compute candidate policies $\pi$ along with a set of trajectory pieces $\psi$. This corresponds to the lines~\ref{alg:dynammic_programming} to \ref{alg:stop_criterion} and detailed in Sections~\ref{sec:dynammic_programming} and \ref{sec:piecewise_optimization}. These pieces do not yet form a globally optimal trajectory-tree, but inform the policy optimization about the cost and feasibility associated with actions. In the third stage we fix $\pi^{\star}$ and optimize the full continuous trajectory-tree $\psi^{\star}$ jointly (line \ref{alg:joint_optimization} detailed in Section~\ref{sec:tree_komo}).

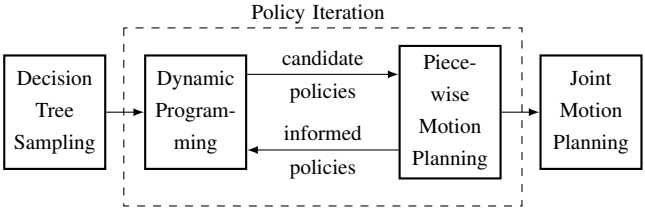
\begin{figure}
\begin{tikzpicture}
  \node[block] (a) {Decision Tree Sampling};
  \node[block, right=0.5cm of a] (b) {Dynamic Programming};
  \node[block, right=2cm of b] (c) {Piece-wise Motion Planning};
  \node[block, right=0.5cm of c] (d) {Joint Motion Planning};
  \draw[line] (a)-- (b);
  \draw[line] ([yshift=0.5cm] b.east)-- ([yshift=0.5cm] c.west) node[pos=0.5,above]{candidate} node[pos=0.5,below]{policies};
  \draw[line] ([yshift=-0.5cm] c.west)-- ([yshift=-0.5cm] b.east) node[pos=0.5,above]{informed} node[pos=0.5,below]{policies};
  \draw[line] (c)-- (d);
  \draw [draw=black,dashed] ([xshift=-0.95cm, yshift=0.35cm] b.north) rectangle ([xshift=0.95cm, yshift=-0.35cm,label=$m$] c.south);
\node[text width=2cm] at (3.6,1.3) {Policy Iteration};
\end{tikzpicture}
\caption{TAMP solver: After the decision tree expansion, candidate policies are generated using dynamic programming and are optimized piecewise to inform about the actual action costs. The final trajectory-tree is re-optimized jointly.}
\label{fig:algo}
\end{figure}

\begin{algorithm}[H]
\caption{TAMP outer loop}
\label{alg:tamp_outer_loop}
\begin{algorithmic}[1]
\Function{Plan}{$b_0$} 
\Comment{Parameters: $c_{mcts}$, $c_0$}
	\State $\mathcal{T} \gets $ \Call{SampleDecisionTree($c_{mcts}$)}{} \label{alg:sample_decision_tree}
	\State $\pi \gets $ \Call{DynamicPrograming}{$\mathcal{T}$, $c_0$} \label{alg:dynammic_programming}
	\Repeat
		 \State $\pi_{prev} \gets \pi$ \Comment{Save $\pi$ for comparison line \ref{alg:stop_criterion}}
		 \State $\psi \gets $ \Call{PieceWiseOptimization}{$\pi$} \label{alg:piecewise_opt}
		 \State $\mathcal{T}.$\Call{IntegrateCosts}{$\psi$}
		 \State $\pi \gets $ \Call{DynamicPrograming}{$\mathcal{T}$}
	\Until{$\pi = \pi_{prev}$} \label{alg:stop_criterion}
	\State $\pi^{\star} \gets \pi$
	\State $\psi^{\star} \gets $ \Call{JointOptimization}{$\pi^{\star}$, $\psi$} \label{alg:joint_optimization}
\EndFunction
\end{algorithmic}
\end{algorithm}
The optimization of trajectory pieces in line \ref{alg:piecewise_opt} raises substantial computational costs, especially since it is performed in a loop for each candidate policy. Our solver involves several mechanisms to decide whether it is worth computing these trajectory pieces. These mechanisms include: (1) controlling the size of the decision tree via a parameter $c_{mcts}$, and (2) a mechanism akin to Rmax~\citep{rmax_brafman2002r}, controlled by a parameter $c_0$, to decide if an action of the decision tree should be “explored”.

\subsubsection{Sampling of the Decision Tree} \label{sec:sample_decision_tree}
\deleted{The decision tree is expanded from the start belief state using a Monte-Carlo based  Partially Observable Upper Confidence Trees algorithm (PO-UCT) described in \mbox{\citep{pomcp_silver2010monte}}.}
\added{
The decision tree is expanded from the start belief state using the Monte-Carlo-based Partially Observable Upper Confidence Trees (PO-UCT) algorithm described in \citep{pomcp_silver2010monte}, which forms the tree-search procedure of the POMCP algorithm.}

A key aspect of PO-UCT is that a state is sampled from the belief state at the begining of each iteration. In our case with deterministic transitions and observation models, this implies that traversing the decision tree  becomes fully deterministic without branching, as in the fully observable case. At action nodes, the action minimizing the Upper Confidence Bound $-\mathcal{C}(b, a) + c_{mcts} \sqrt{\log(\frac{N(b)}{N(b, a)})}$ is selected, where $c_{mcts}$ is an exploration parameter as mentioned in Algorithm~\ref{alg:tamp_outer_loop}. $\mathcal{C}(b, a)$ is the current estimate of the expected costs to go. $N(b)$ and $N(b, a)$ are the visit counters. This is the standard upper-confidence-bound used here for cost minimization instead of reward maximization. Rollouts are performed using random actions down to a given maximal depth (typically 50 actions). At this step, no motions have been planned such that the true trajectory costs, and the feasibility of actions are not known. We assume that all actions have equal costs, which directs the expansion towards solutions minimizing the number of actions. Iterations are stopped once the decision tree contains at least one solution policy, and once a minimum number of iterations is reached.

\deleted{Unlike in a typical usage of PO-UCT where the end result is the solution policy, we are interested in the expanded decision tree, which we will use as search space in the next stages}\added{
Unlike in POMCP, where PO-UCT is used to select the next action online, we use PO-UCT solely to incrementally construct a decision tree, which serves as the search space for the subsequent generation of candidate policies.} The goal is to expand the decision tree enough such that it contains enough candidate policies for the policy improvement phase, without resorting to a systematic uninformed expansion which would result in an unnecessarily large decision tree.

\subsubsection{Dynamic Programming on the Decision Tree} \label{sec:dynammic_programming}
We assume that at any point in time we have cost estimates
$c(a, b)$ for each action $a$ in belief state $b$ in the decision tree. These costs are first initialized with a low, optimistic value $c_0$ as mentioned in Algorithm~\ref{alg:tamp_outer_loop}. Only when an candidate policy makes it likely to
actually visit the edge $(b, a)$ do we compute more precise estimates
using piecewise trajectory optimization, as described below.
Given the current cost estimates, an exact expectation of the costs to goal is computed. This is performed by applying dynamic programming as described in Algorithm~\ref{alg:expected_costs}.

\begin{algorithm}[H]
\caption{Computation of the expected costs}
\label{alg:expected_costs}
\begin{algorithmic}[1]
\Function{ComputeExpectedCost}{$\mathcal{T}$}
	\State \textcolor{cyan}{\footnotesize/* Initialization */}
	\State $Q \gets $ \Call{PriorityQueue()}{}  \label{alg:init_start}
	\For{$n \in$ $\mathcal{T}.nodes$}
		\If{\Call{IsFinal}{$n$}}
			\State $\mathcal{C}[n] \gets 0.0$
			\State $Q$.\Call{Push}{$n, 0.0$}
		\Else
			\State $\mathcal{C}[n] \gets +\infty$
		\EndIf
	\EndFor \label{alg:init_end}
	\State \textcolor{cyan}{\footnotesize/* Apply Bellman updates */}
	\While{$\neg$ \Call{IsEmpty}{$Q$}}
		\State $w \gets$ \Call{Pop}{$Q$} \label{alg:candidate_cost_start}
		\State $v \gets$ \Call{ObservationParent}{$w$}
		\State $u \gets$ \Call{ActionParent}{$v$}
		\State $\mathcal{W} \gets$ \Call{ActionChildren}{$v$}
		\State $\mathcal{C}_{new} \gets$ \Call{\scriptsize{c}}{$u, v$}$ + \sum_{\nu \in \mathcal{W}}{ p(\nu | u) \times \mathcal{C}[w]} $ \label{alg:candidate_cost_end}
		\If{$\mathcal{C}_{new} < \mathcal{C}[u]$} \label{alg:select_start}
			\State $\mathcal{C}[u] \gets \mathcal{C}_{new}$
			\State $Q$.\Call{Push}{$u, \mathcal{C}_{new}$} \label{alg:select_end}
		\EndIf
	\EndWhile
\EndFunction
\end{algorithmic}
\end{algorithm}

First, a queue of action nodes of the decision tree is built, where nodes are ranked by increasing costs to a terminal node. Terminal nodes are placed in the queue. The expected costs to goal are initialized to infinity for non-terminal nodes (lines \ref{alg:init_start} to \ref{alg:init_end}). 

Next, a node $w$ is taken from the queue and a new expected cost is computed for its action parent $u$ (lines \ref{alg:candidate_cost_start}  to \ref{alg:candidate_cost_end}). If the new expected cost improves the current estimate, it is saved, and $u$ is placed on the queue (lines \ref{alg:select_start} to \ref{alg:select_end}). This procedure is repeated until the queue is empty.

The computation of the new costs (line~\ref{alg:candidate_cost_end}), together with the comparison to the current estimate (line~ \ref{alg:select_start}) corresponds to the application of the following Bellman update:
\begin{equation} \label{eq:bellman_update}
\mathcal{C}_{i+1}(b) \leftarrow \min_{a} \left[c(a, b)+\sum_{o \in O(b, a)}{O(o|b, a) \mathcal{C}_{i}(T(b, a, o))} \right],
\end{equation}
where $O(b, a)$ is the set of possible observations received after executing the action $a$. $O(o|b, a)$ is the probability of receiving the observation $o$ after executing $a$ (corresponding to $p(v|u)$ line \ref{alg:candidate_cost_end}). $T(b, a, o)$ 
indicates the successor of $b$ after taking action $a$ and receiving observation $o$. These nodes correspond to $\mathcal{W}$ in line~\ref{alg:candidate_cost_end}.

This algorithm takes advantage of the tree structure to update the expected costs from the leaves to the root, minimizing the number of Bellman updates compared to a standard Value Iteration algorithm as used in \citep{tamp_0_phiquepal2019combined}. 
\revise{Nodes that do not lead to any leaf---thereby constituting dead ends---retain their initial infinite cost throughout the procedure, as no Bellman update assigns them a finite value. The same applies to nodes that lead to a leaf only through an action edge with infinite cost, are assigned an infinite cost by the Bellman update. This case occurs when the piecewise trajectory optimization reported infeasibility of an action, as detailed in the next section. The Bellman updates propagate the infeasibility information throughout the relevant portions of the decision tree.}

A candidate policy $\pi$ is straightforwardly extracted by traversing the decision tree from the root and choosing at each action node the action with the lowest expected costs. 

\revise{An infinite cost at the root node at the end of the procedure indicates that the decision tree does not contain a valid solution. Either the decision tree has not been expanded enough or the planning problem is infeasible. }
\subsubsection{Piecewise Trajectory Optimization}\label{sec:piecewise_optimization}
A given policy $\pi¸$ transitions only a small subset of action
edges of the full decision tree. For this set of action edges we compute $c(a, b)$ (if it was not already computed in previous iterations). For the sake of computational efficiency, \rrevise{$c(a, b)$ is estimated} in two stages:

We first optimize key-frames only (robot pose at each
node). This step is much quicker than optimizing a full trajectory piece. If an action is impossible at this stage, the optimization is not pursued further. Infinite costs are returned, thereby making this node and sub-tree unreachable by the candidate policies. The pose feasibility check is optimistic, it might succeed even if the path itself is infeasible (e.g. if there is no possible trajectory without collision between two key-frames).

If pose optimization reports feasibility, we then optimize the trajectory piece, minimizing (\ref{eq:attachement_action}), to get the cost estimate $c(a, b)$. We use a time discretization of 20 frames per action. Optimization is performed using the \textit{K-Order-Motion-Optimization} method (KOMO) adopted from 
\citep{tamp_lgp_2_toussaint2017multi, tamp_lgp_1_toussaint2015logic}. We save the computed trajectory piece and its cost $c(a, b)$ which will be used in the next round of Dynamic Programming. The final configurations of the trajectory piece are used as initial state for the following action edges. \revise{If piecewise trajectory failed to converge to a feasible solution, the cost $c(a, b)$ is set to infinity, as defined in (\ref{eq:attachement_action}). In the case of observation actions, as described in Section~\ref{sec:observative_actions}, infeasibility arises when the robot is unable to position its sensor within the visibility cone of the target object.}

There may be a strong overlap between candidate policies (same edge in many candidate policies). This is especially the case in the last iterations
of Policy improvement, saving computational costs as computing $c(a, b)$ is performed only once. Intuitively, as we alternate between Dynamic Programming and Piecewise trajectory planning, the decision tree is informed with precise estimates of the costs of actions until convergence to a final policy $\pi^{\star}$.

\subsubsection{Choice of $c_0$ and Exploration vs. Explotation tradeoff}

The use of optimistic initializations $c_0$ of $c(a, b)$ is analogous to the R-Max algorithm~\citep{rmax_brafman2002r} and allows us to control exploration vs. exploitation within the policy optimization. An optimistic initial $c_0$ (e.g., zero costs) encourages exploration, since unexplored actions will appear advantageous compared to actions associated with costs resulting from the piecewise trajectory optimization.

On the other hand, when chosing $c_0$ less optimistically, we lose the guarantees that come with admissible heuristics, but may converge faster to reasonable policies. Our experiments will investigate the influence of the cost initialization on the number of iterations.

Convergence to the best policy contained in the decision tree is guaranteed if $c_0$ is chosen as a lower bound of the piecewise trajectory costs.

\subsubsection{Joint Trajectory-Tree Optimization (T-KOMO)} \label{sec:tree_komo}
In the third stage of the solver, we fix the symbolic policy $\pi^{\star}$ found, as described above, and focus on the joint optimization of the trajectory-tree $\psi$. So far we have only optimized pieces for each action independently. Concatenating these  pieces cannot capture long-term dependencies in the trajectories, e.g. when final actions influence earlier parts of the trajectory. The joint optimization of the trajectory-tree leads to better and smoother motions as the experiments show. It is performed only once for the best symbolic policy $\pi^{\star}$.

\medskip
\noindent\textbf{Joint optimization objective\normalfont{:}}
Let $\psi$ be the trajectory-tree composed of a total of $N$ configurations. $\psi$ implements the policy $\pi$. We index the trajectory elements in a depth-first way. 
We note $p_i$ the probability to reach the $i^{\text{th}}$ configuration. By definition, $p_0 = 1.0$, and this probability does not change between two consecutive configurations on a sequential part of the trajectory. The probability evolves at observation branchings based on the branching probability derived from the observation model $\mathcal{O}$.
Fig.~\ref{fig:t-komo_tree} illustrates the indexing of $\psi$ and how cost and constraints functions apply on consecutive configurations. We note $\alpha(j, n)$ the $n^{\text{th}}$ ancestor of the $j^{\text{th}}$ configuration. 
To keep notations compact, we note $\psi_{j-k:j}$ the $k$-th-tuple composed of the configurations leading to $j$ for a given motion order $k$, i.e. $\psi_{j-k:j} = (\psi_{\alpha(i, k-1)}, ..., \psi_{\alpha(i, 1)}, \psi_i)$. The prefix $\psi_{k-1:0}$ are the configurations before the trajectory-tree; assuming them to be known simplifies the notation, without the need to introduce a special notation for the first $k$ terms. 
\begin{figure}[ht]
    \centering {%
    \includegraphics[width=0.9\linewidth]{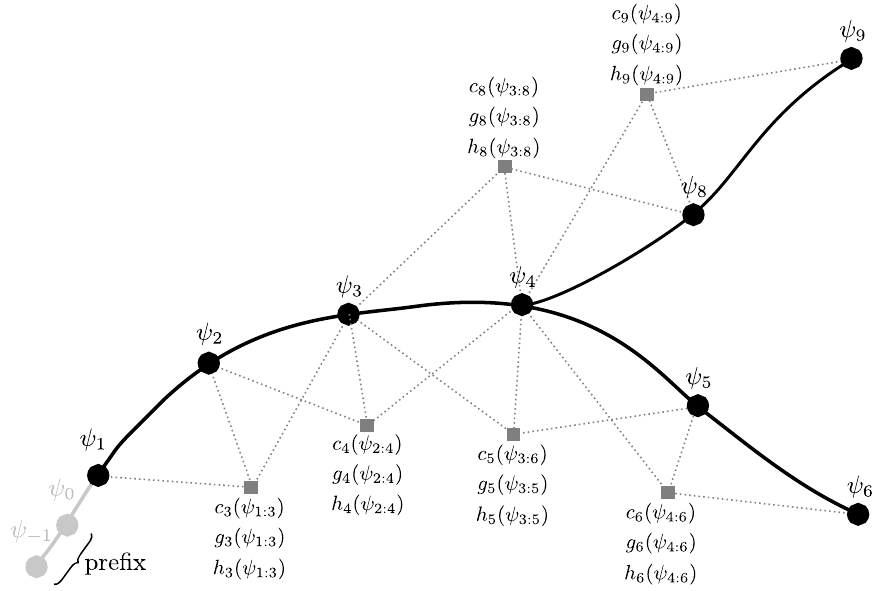}}
  \caption{Illustration of the structure implied by the T-KOMO formulation. $c$, $g$ and $h$ represent $2^{\text{nd}}$ order functions applying on 3 consecutive configurations.}
  \label{fig:t-komo_tree} 
\end{figure}

The $k$-order trajectory-tree optimization problem is formulated as,
\begin{subequations}
\begin{alignat}{2}  
  \min_{\psi} \quad & \sum_{i=1}^{N} {p_i { c_i(\psi_{i-k:i}) } }, \label{eq:ktreekomo}\\
  \text{s.t.} \quad & h_i(\psi_{i-k:i}) = 0\ &\forall i,\\
                    & g_i(\psi_{i-k:i}) \leq 0\ &\forall i,
\end{alignat}
\end{subequations}
where the functions $c$, $g$ and $h$ are the functions defined in~\eqref{eq:attachement_action}. These functions are indexed by $i$ and not by the symbolic action $a$ since at this stage motions are discretized (typically 20 configurations per actions). This optimization objective corresponds to the initial problem formulation~\eqref{eq:tamp_optimization_objective} for a fixed candidate policy $\pi$ and over a discretized representation of the trajectory-tree. The cost terms $c_i(\psi_{i-k:i})$ are weighted by the probability to reach a given configuration of the tree, which leads to trajetory-trees more optimized towards likely configuration than unlikely ones.
The motion order $k$ is the number of consecutive configurations needed by the cost and constraints functions, e.g. $k=2$ for costs defined on the acceleration.

This optimization problem formulation generalizes the K-order Motion Optimization (KOMO) to tree-like trajectories.

\medskip
\noindent\textbf{Optimizing the trajectory-tree\normalfont{:}}
The optimization objective~\eqref{eq:ktreekomo} is a constrained optimization problem that we optimize with the D-AuLa solver introduced in Section~\ref{sec:d-aula-solver} but without decomposition into subproblems (i.e. with a number of subproblems equal to one). \revise{In that case, the D-AuLa algorithm effectively reduces to the standard Augmented Lagrangian method (see Appendix~\ref{sec:background_aula})}. We do not apply decomposition by default since, unlike in the MPC use case, there is no decomposition generally applicable to TAMP problems. Moreover, runtime is less critical in the TAMP use case since the joint optimization is performed only on one policy. A comparison of the optimization time with and without problem decomposition is provided in Table~\ref{table:distrib_vs_joint}.

As mentioned in the solver section, the optimization procedure consists in performing several Newton minimizations of the Augmented Lagrangian using a Gauss-Newton approximation of the Hessian. A computationally intensive part is the Hessian Cholesky decomposition. In the sequential trajectory optimization, the KOMO problem formulation leads to a banded-symmetric Hessian with bandwidth $(2k+1)n$ where $n$ is the number of degress of freedom. This results in a complexity which is only linear in the number of timsteps~$N$ for the Cholesky decomposition. We refer the reader to \citep{komo_17-toussaint-Newton} for more details. When optimizing a trajectory-tree, the Hessian is not banded symmetric anymore, as shown in Fig.~\ref{fig:hessian_matrices}.
\begin{figure}[ht]
    \centering
       \subfloat[Trajectory-tree\label{h_traj_tree}]{%
       \includegraphics[width=0.47\linewidth]{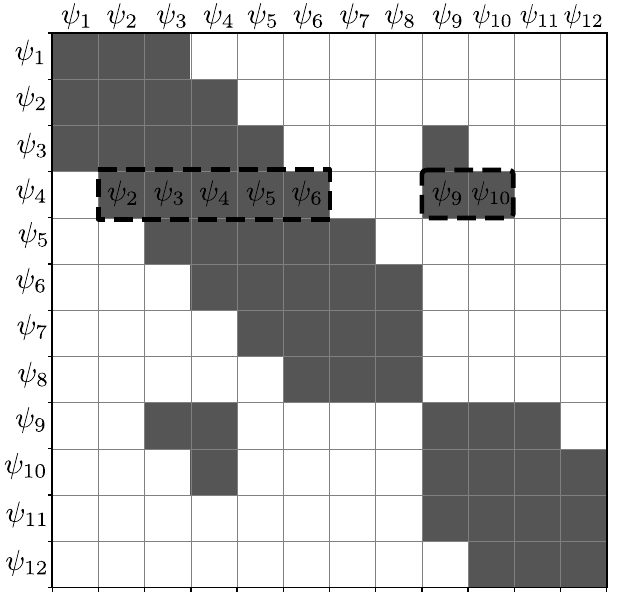}}
           \hfill
       \subfloat[Sequential trajectory\label{h_linear_traj}]{%
       \includegraphics[width=0.47\linewidth]{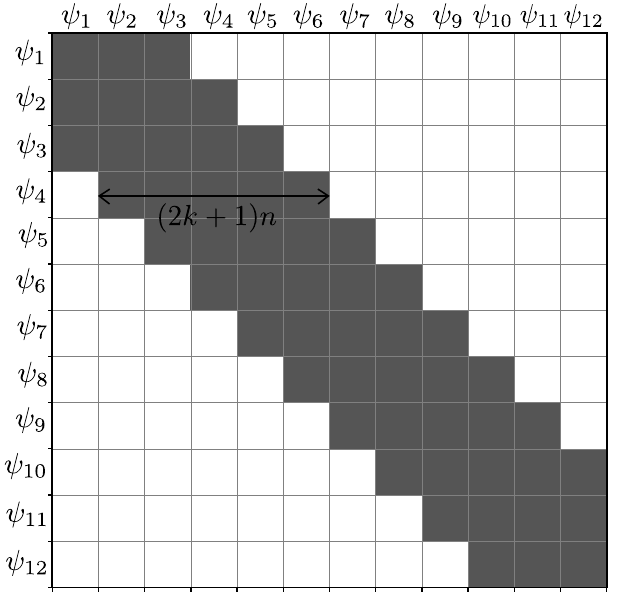}}
  \caption{Comparison of Hessian matrices: When optimizing a trajectory-tree, the Hessian~(a) is not banded-symmetric, unlike in the sequential trajectory case~(b). This reflects the branching illustrated in Fig.~\ref{fig:t-komo_tree}. The trajectory element $\psi_4$ is not only connected to $\psi_{2:6}$ but also to $\psi_8$ and $\psi_9$.}
  \label{fig:hessian_matrices} 
\end{figure}

However, the Hessian is still as sparse as in the sequential trajectory case. We resort to sparse matrix arithmetic for the Hessian decomposition (instead of banded-symmetric) which leads to similar computation time in practice. In other words, optimizing a trajectory-tree of $N$ elements is as complex as optimizing a sequential trajectory with the same number of elements. \revise{The linear complexity with respect to the number of trajectory elements is verified empirically (see Fig.~\ref{fig:scalability:joint}).}

To reduce computation time, the optimization is warm-started with the results of the piecewise optimization.

\subsubsection{Observation Actions} \label{sec:observative_actions}
In the examples we consider in the experiments, not all actions will provide observations. We use a dedicated \textit{Look} action instead. Fig.\ref{fig:look_action_grounding} shows the grounding we use for this action. The cost and constraints functions determine where the sensor should be placed such that, at execution time, knowledge is gained. 
\begin{figure}[htbp]
    \centering
    \begin{minipage}{0.45\linewidth}
        \centering
        \includegraphics[width=\linewidth]{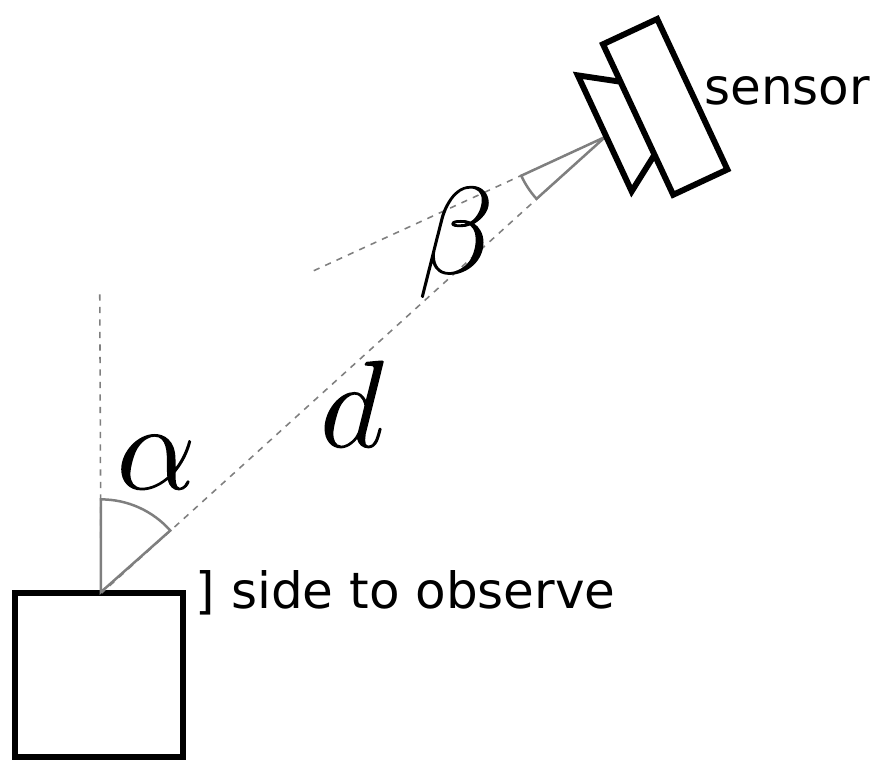}
        \label{fig:example-figure}
    \end{minipage}%
    \hfill 
    \begin{minipage}{0.54\linewidth}
        \centering
\begin{subequations}
\begin{align}  
 c_{look}(x)\; =\; &(d(x) - d_{desired})^2 \notag \\
 \text{s.t.} \; \lvert\alpha(x)\rvert \; \leq \; &\alpha_{max}, \label{eq:incidence_angle}\\
 \beta(x)\; = \; & 0 \label{eq:fov_centering}.
\end{align}
\end{subequations}
    \end{minipage}
\caption{Example of \textit{Look} action: The robot is incentivized to place its sensor at a distance $d_{desired}$ from the object to observe. The angle w.r.t. the surface should remain smaller than $\alpha_{max}$~\eqref{eq:incidence_angle} and the side center shall be centered in the sensor's field of view~\eqref{eq:fov_centering}.\label{fig:look_action_grounding}}
\end{figure}
This grounding is formulated based on the relative pose between the sensor and the object to observe. It does not determine, by itself, if the sensor or the object should be moved during the action. The possible and most advantageous joints to action for fulfilling the action depends on the scene kinematic, and will result from the optimization. In the examples using the Baxter, this leads the robot to move both its head and arm, as shown in Fig.~\ref{fig:example_problem_tamp:look}.

\rrevise{Under this formulation, trajectories that satisfy the geometric constraints associated with observation actions are assumed to lead to an observation according to the observation model $O$.}
\rrevise{In this example, this implicitly assumes that sufficient conditions for successful object detection are known and can be encoded through geometric constraints. In practice, additional factors (e.g., lighting conditions) may affect perception performance, potentially resulting in non-detections even when the sensor viewpoint satisfies the specified constraints. Such events fall outside the planning assumptions and would require replanning, analogously to failures of other symbolic actions such as grasping.}

\subsubsection{Implications of Policy Selection Based on Piecewise Optimization}
\label{sec:implications}

\rrevise{The planning procedure described in the previous sections determines the final symbolic policy based on piecewise trajectory optimization, where actions are optimized independently. This constitutes a trade-off between computational efficiency and completeness or optimality. It is efficient because an action edge of the decision tree shared among multiple candidate policies needs to be optimized only once. This also enables the use of Dynamic Programming for generating candidate policies.}

\rrevise{However, this decomposition implies that trajectory pieces of early actions are fixed before subsequent planning, which may create unfavorable conditions for later actions, potentially resulting in infeasibility or high costs. In such cases, the propagation of infinite or high costs during Dynamic Programming may rule out otherwise viable parts of the decision tree, potentially resulting in a suboptimal final policy or, in the worst case, no feasible policy being found.}

\rrevise{This limitation is partially mitigated by the final joint re-optimization of the selected policy, although this re-optimization does not alter the choice of symbolic policy itself.}

\rrevise{In domains where strong causal dependencies between geometric decisions are critical, several directions may be considered:}
\begin{itemize}
\item \rrevise{Introduce symbolic action variants (e.g., grasp-from-top vs.\ grasp-from-side), thereby exposing geometric alternatives at the logic level. This is already supported within the current framework but only allows predetermined alternatives.}
\item \rrevise{Perform the search directly over complete policies, as in classical LGP in the fully observable case~\citep{tamp_lgp_1_toussaint2015logic}, thereby relying more heavily on joint trajectory-tree optimization.}
\item \rrevise{Introduce mechanisms to reconsider piecewise trajectories upon failure of subsequent actions. This requires varying the optimization problem of an action \eqref{eq:attachement_action}, for example by introducing additional parameters that may be sampled (e.g. goal pose for grasping). Related ideas of interleaving refinement and policy recomputation have been explored in~\citep{shah2020anytime}.}
\end{itemize}
%
%
%
%
\subsection{Experiments} \label{sec:tamp_experiments}
The solver is implemented in C++. The source code and a supplementary video are available for reference.\footnotemark \footnotetext{\url{https://github.com/cambyse/trajectory_tree_tamp}}


\subsubsection{Planning Problems} \label{sec:planning_problems}
We consider the Baxter and Franka robots with the task to stack blocks in a given color order (blue, green and red on top). The blocks only have one side colored. The robot knows where blocks are, but it cannot see the colored side from the initial position, and therefore has to explore to identify the blocks and build the stack.

We plan using using the Baxter's right arm (7 DOF)  and assume a sensor is placed on the Baxter's head which is itself articulated leading to a total of 8 DOF. For Franka, we assume the sensor is mounted on the gripper resulting in 7~DOF.

\noindent There are 3 actions:
\begin{itemize}
\item \textbf{Look} at a block: the robot must align its sensor with the colored side of the block. This is modeled by the cost and constraints functions described in Section~\ref{sec:observative_actions} with a maximum observation angle $\alpha_{max}=\ang{45}$, and a desired observation distance $d_{desired} = \SI{20.0}{\centi\meter}$. \rrevise{These are applied during the final $\SI{200}{\milli\second}$ of the action, ensuring that the colored side is visible at the end of the sensor placement trajectory.} In the case of Franka, the sensor is highly mobile, allowing it to move to various observation points. For Baxter, however, the sensor's movement is more restricted, often requiring the robot to move both its head and arm simultaneously (see Fig.~\ref{fig:example_problem_tamp:look}). An observation is received after this action indicating the block's color.
\item \textbf{Grasp} a block: only the right Baxter's arm can grasp.
\item \textbf{Place} a block at a location: the block is placed on the table, or onto another block.
\end{itemize}
In addition to the cost and constraints specific to each action, the squared joint acceleration is minimized.

The variations of the planning problems as well as the experiments specific purposes are summarized in Table~\ref{table:problems}.
In the subproblems using the Baxter, we assume that the colored side can be only the side opposite the robot, such that the robot knows already which side to observe, thereby largely reducing the dimensionality of the belief state. 
In the subproblems \deleted{with Franka-A and Franka-B}\added{using Franka}, we do not make this assumption and the belief state size quickly becomes large (e.g. Franka-B).

The limits of tractability are quickly reached in such a combinatorial domain, \deleted{as we see with the problem Franka-B }\added{as illustrated by the Franka-B problem, which was specifically designed to expose this scalability limitation}. \deleted{To scale further }\added{As one practical strategy to improve scalability}, we introduce the problem Franka-C$\times$A', where the problem is decomposed by optimizing: (1) an overarching policy (Franka-C) which rearranges blocks without implementing the observations of the blocks's side, and (2) a policy which identifies a block (Franka-A'), which used as a explorative macro-action. Franka-A' is similar to Franka-A but is optimized with a fixed goal configuration, \rrevise{enabling the handover to Franka-C once the block's color has been identified}. Composed together, this gives solution policies for the combined problem (Franka-C$\times$A'). The composition breaks down the complexity of the combined problem, which would be otherwise intractable with a belief state size of 1296.

The problem Baxter-D is to evaluate robustness against motion planning failures. In Baxter-B and C, the \textit{Look} action has a precondition: the robot should have a block in hand before looking at it. This precondition is removed in the variation D. This causes the \textit{Look} action to be symbolically possible more often. However, if the robot does not hold the block, no robot motion allows sensor alignment with the colored side causing the motion planning to fail.

\begin{center}
\footnotesize
\setlength\tabcolsep{3pt} 
\begin{tabular}[ht]{|c|c|c|c|}
\hline
\thead{Problem} & \thead{Belief state\\ size} & \makecell[l]{Description} & \makecell[l]{Purpose} \\
\hline
\hline
Baxter-A                     & 1 & \makecell[l]{3/3 blocks known} & \makecell[l]{Fully observable \\ baseline} \\
\hline
Baster-B                     & 2 & \makecell[l]{1/3 blocks known}  &\makecell[l]{Short policies} \\
\hline
Baxter-C                     & 6 & \makecell[l]{0/3 blocks known}  &\makecell[l]{Larger policies}\\
\hline
Baxter-D                     & 6 & \makecell[l]{0/3 blocks known \\Relaxed logic}  &\makecell[l]{Influence of \\ motion planning \\ failures}\\ 
\hline
\hline
Franka-A                     & 6 & \makecell[l]{1 block\\Colored side unknown}  &\makecell[l]{\rrevise{Policy with}\\\rrevise{cascading structure}}\\
\hline
Franka-B                     & 72 & \makecell[l]{2 blocks\\Colored side unknown}  &\makecell[l]{Show scalability\\issue with large\\ combinatorial space}\\
\hline
\hline
Franka-C                     & \rrevise{6} & \makecell[l]{3 blocks\\Defered observations\\Colored side unknown}  &\makecell[l]{Overarching policy\\for Franka-C$\times$A'}\\
\hline
\rrevise{Franka-A'}                     & 6 & \makecell[l]{Similar as Franka-A\\Fixed goal configuration}  &\makecell[l]{Macro-action\\for Franka-C$\times$A'}\\
\hline
Franka-C$\times$A'                     & 1296 & \makecell[l]{3 blocks\\Colored side unknown} &\makecell[l]{Scalability through \\composed policy}\\
\hline
\end{tabular}
\captionof{table}{\label{table:problems}Overview of the planning problem variations} 
\end{center}
\subsubsection{Planning Results}
\noindent\textbf{Planned policies\normalfont{:}}
Fig.~\ref{fig:planned_policies} shows example of planned policies. The planning problems with the Baxter result in policies with a sparse branching with sequences of several actions between observations. This is visible in Fig.~\ref{fig:policy-baxter-b} for Baxter-B.
This sparse branching reflects the assumption that the colored side is opposite the robot, which limits the number of contingencies. In contrast, the policies for the Franka problems have denser branching, as Fig~\ref{fig:policy-franka-a} shows. The combinatorics quickly becomes large, the policies for Franka-B contains 72 branches, with 305 actions and 43 branching points. For space considerations, the policies for Baxter-C and Franka-C$\times$A' are given in Appendix~\ref{sec:appendix_tamp_policies}, and the corresponding trajectory-tree executions are shown in the accompanying video provided as a multi-media extension.

\begin{figure}[ht]
    \centering
    \subfloat[Sequential policy for problem Baxter-A (full observability)]{
    \includegraphics[width=\linewidth]{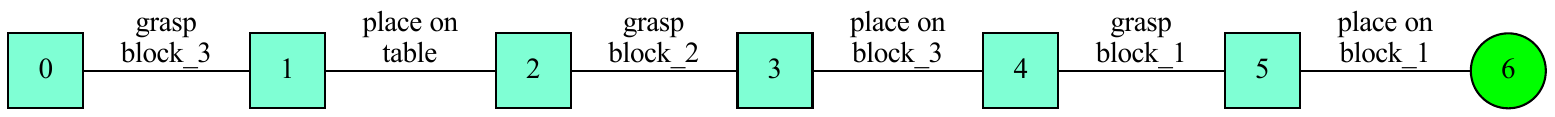}
    }
    \\
    \subfloat[Policy with one observation branching for Baxter-B]{
    \includegraphics[width=\linewidth]{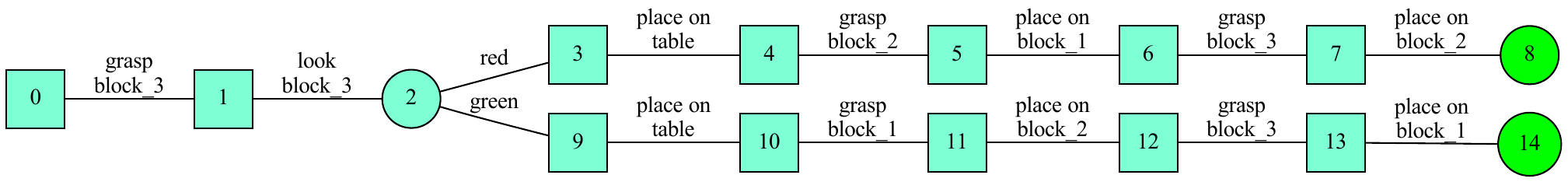}\label{fig:policy-baxter-b}
    }
    \\
    \subfloat[Policy with a cascading structure for Franka-A]{
    \includegraphics[width=\linewidth]{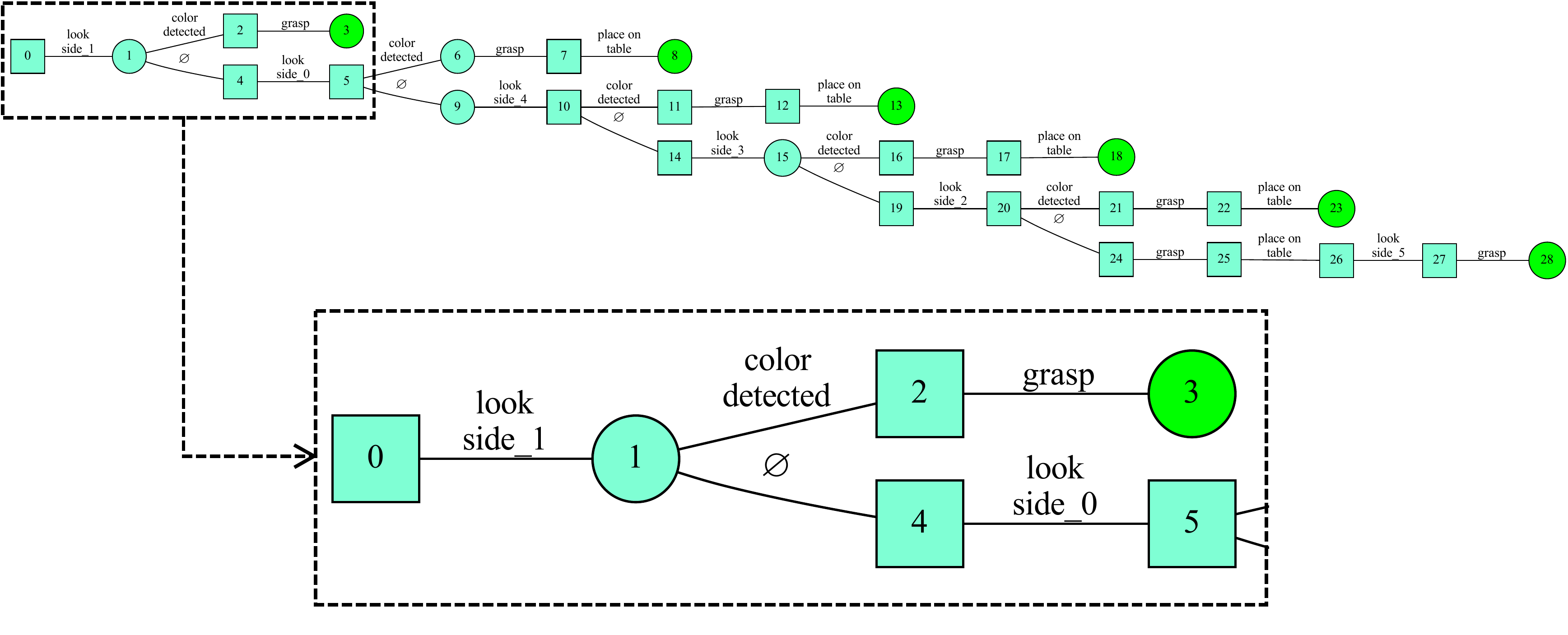}\label{fig:policy-franka-a}
    }
    \caption{Policies obtained for the problems Baxter-A, Baxter-B and Franka-A.\label{fig:planned_policies}}
\end{figure}

Fig.~\ref{fig:example_config_franka} and Fig.~\ref{fig:example_problem_tamp} show robot configurations at different stages of the policy execution for the Baxter and Franka problems respectively. Fig.~\ref{fig:3_successful_look} shows a configuration where the optimized observation pose is close to the boundary of the inequality constraints, the sensor's line of sight has an orientation of approximately \ang{45} with respect to the observed side. Observing the side with a lower relative angle would require the robot to further lower the gripper position which is not advantageous in terms of trajectory-cost.

\begin{figure}[ht]
    \centering
       \subfloat[Start of Franka-C$\times$A'\label{fig:0_start}]{%
       \includegraphics[width=0.3\linewidth]{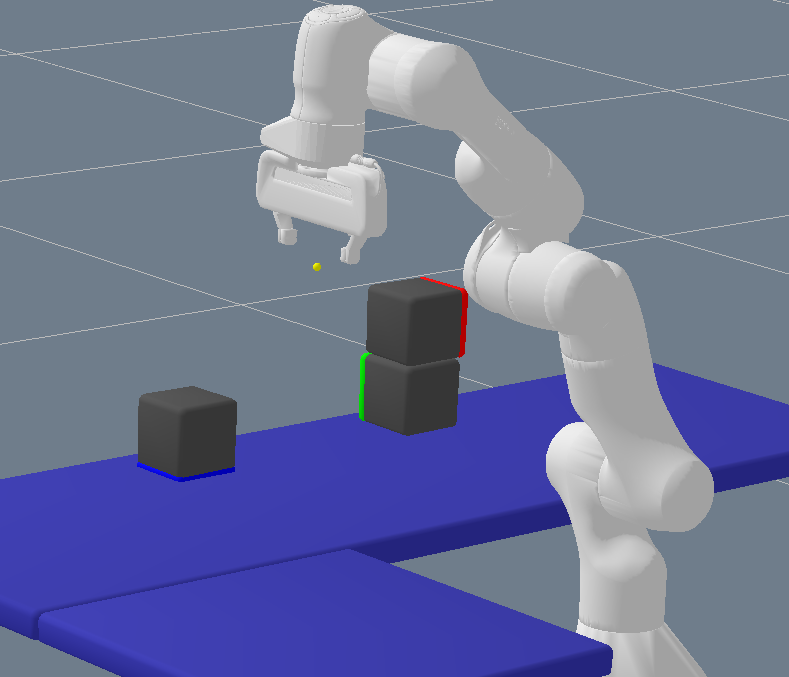}}
           \hfill
       \subfloat[Start pose for explorative policy (Franka-A')\label{fig:1_start_franka_a}]{%
       \includegraphics[width=0.3\linewidth]{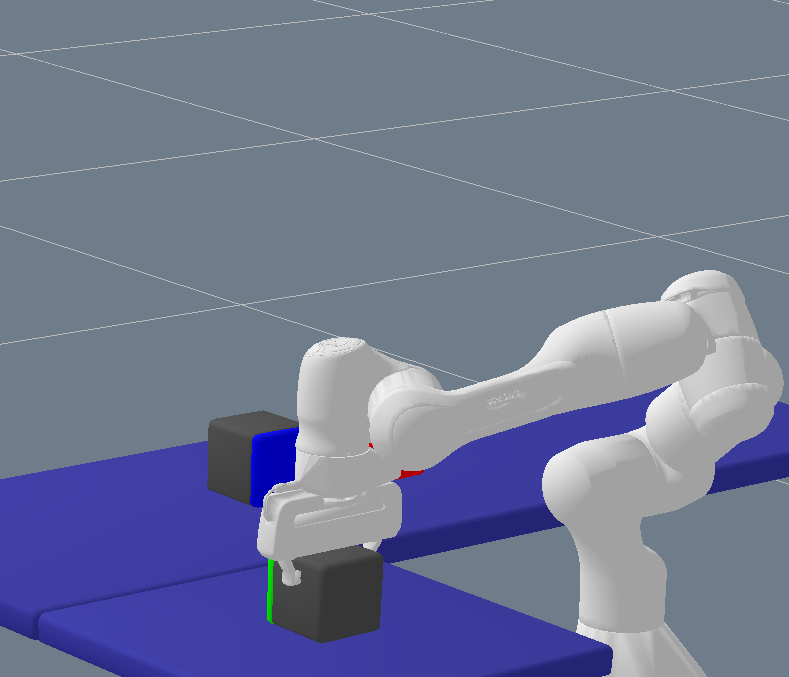}}
           \hfill
       \subfloat[\textit{Look} action detecting no color\label{fig:2_unsuccessful_look}]{%
       \includegraphics[width=0.3\linewidth]{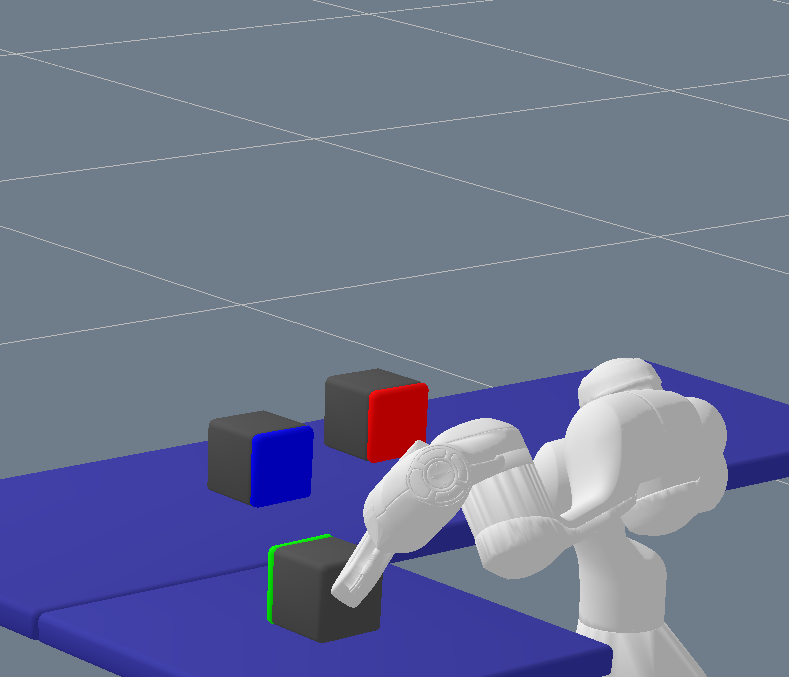}}
       \\
              \subfloat[\textit{Look} action identifying the block\label{fig:3_successful_look}]{%
       \includegraphics[width=0.3\linewidth]{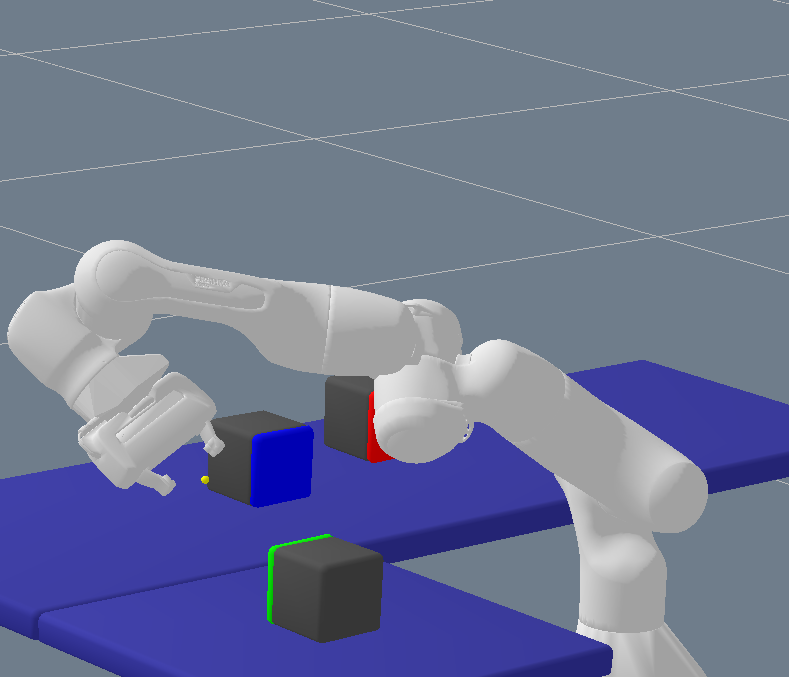}}
           \hfill
       \subfloat[End of Franka-A, block is identified and re-arranged\label{fig:4_end_explo}]{%
       \includegraphics[width=0.3\linewidth]{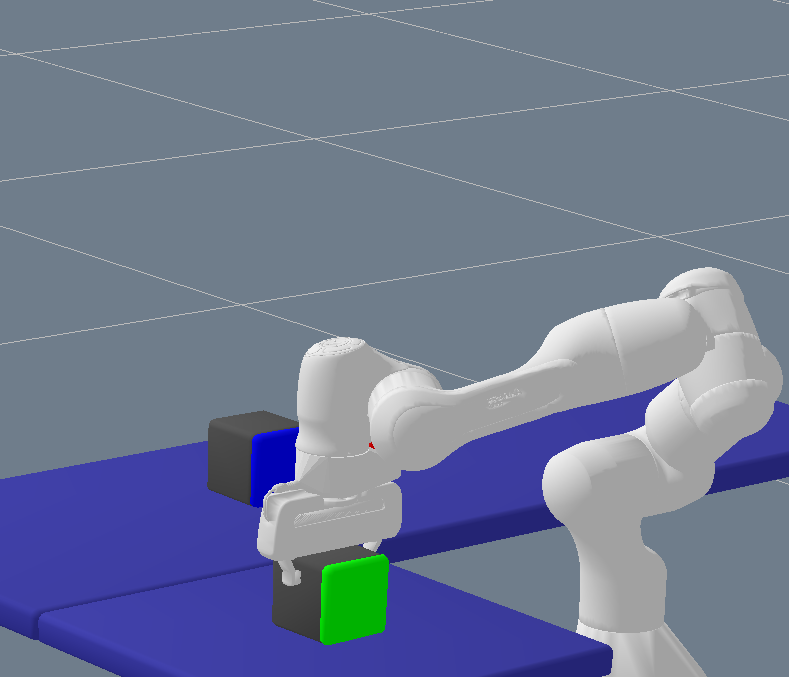}}
           \hfill
       \subfloat[Goal state of Franka A$\times$C\label{fig:5_goal_fanka_a_c}]{%
       \includegraphics[width=0.3\linewidth]{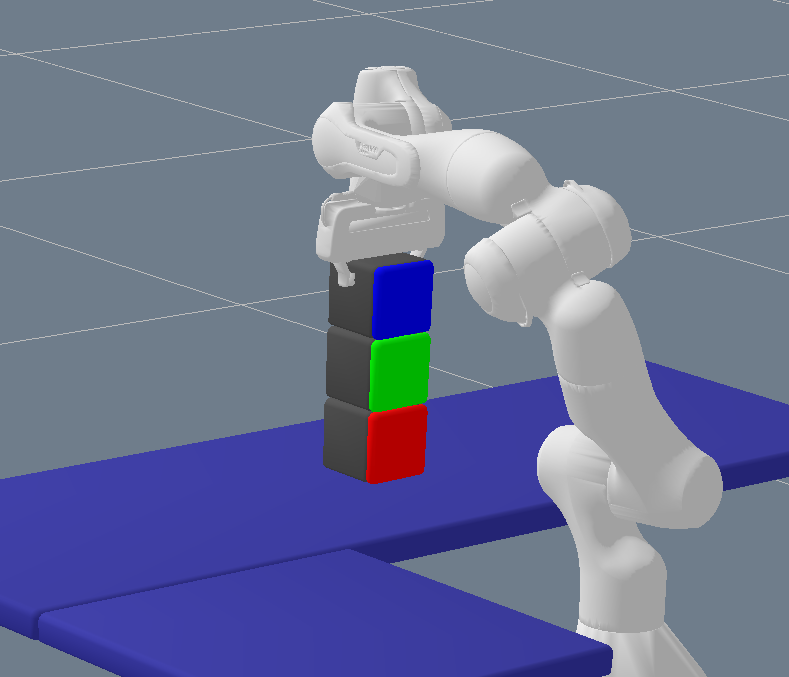}}
  \caption{Examples of configurations for Franka C$\times$A': Blocks colors are unknown at the start \ref{fig:0_start}. Blocks are brought to a fixed position \ref{fig:1_start_franka_a}. From there,
the explorative policy Franka-A' observes the sides \ref{fig:2_unsuccessful_look}, \ref{fig:3_successful_look} and rearranges the block \ref{fig:4_end_explo}. The goal state is \ref{fig:5_goal_fanka_a_c}.} 
  \label{fig:example_config_franka} 
\end{figure}

\noindent\textbf{Planning time\normalfont{:}}
Table~\ref{table:times} gives an overview of the planning time. \rrevise{For each standalone (i.e., non-composed) problem, we consider two variations of $c_0$ to either aim for short planning times or optimal policies. For Franka-C$\times$A', we report the metrics for Franka-C and Franka-A' separately, as well as the total planning time, which amounts to their sum, since planning for Franka-C and Franka-A' needs to be performed only once to obtain a complete composed policy valid in all contingencies.} Planning for each variation is carried out five times and we report on the average planning times and iterations. Fig.~\ref{fig:iteration_cost_scatter_plot} provides an indication of the dispersion around the average.

\begin{center}
\footnotesize
\setlength\tabcolsep{3pt} 
\begin{tabular}{|c|c|c||c|c|c|c|c|c|}
\hline
\thead{Problem} & \thead{$c_0$} & \thead{Iter-\\ations}  & \thead{Decision\\tree} & \thead{Dyn.\\prog.} & \thead{Piecewise\\motion\\planning} & \thead{Joint\\motion\\planning} & \thead{Total\\(s)}\\
\hline
\hline
Baxter & 10.0 & 1 & 0.156 & 0.0002 & 0.52 & 0.56 & 1.24 \\
A      & 0.1 & 4 & 0.157 & 0.0012 & 1.22 & 0.563 & 1.94 \\
\hline
Baxter & 10.0 & 1 & 0.34 & 0.0003 & 1.28 & 1.57 & 3.21 \\
B      & 0.1 & 6 & 0.35 & 0.002 & 3.90 & 1.24 & 5.51 \\
					    
\hline
Baxter & 10.0 & 4 & 1.04 & 0.007 & 8.74 & 5.48 & 15.25 \\
C      & 0.1 & 13 & 1.02 & 0.024 & 27.0 & 5.33 & 33.4 \\
\hline
Baxter & 0.1 & 57 & 12.4 & 5.94 & 33.9 & 5.03 & 57.39 \\
D      & 10.0 & 9 & 12.5 & 0.97 & 6.89 & 5.21 & 25.6 \\
\hline
\hline
Franka & 10.0 & 1 & 0.35 & 0.003 & 1.63 & 1.11 & 3.10 \\
A      & 1.0  & 16 & 0.40 & 0.059 & 20.7 & 1.86 & 23.1 \\
\hline
Franka & 10.0 & 1 & 25.2
 & 0.018 & 62.1 & 129.1 & 216.9 \\
B      & 1.0 & 1 & 25.5 & 0.021 & 64.1 & 193.6 & 283.2 \\
\hline				    
\hline
\thead{\rrevise{Franka C}} 
      & 1.0 & 12 & 0.97 & 0.016 & 15.4 & 5.01 & 21.4 \\
\hline
\thead{\rrevise{Franka A'}}       & 5.0 & 13 & 0.31 & 0.042 & 13.04 & 1.27 & 14.7 \\
\hline
\thead{\rrevise{C$\times$A'}} & - & - & \thead{1.28} & \thead{0.058} & \thead{28.34} & \thead{6.28} & \thead{36.0}\\
\hline
\end{tabular}
\captionof{table}{\label{table:times}Number of iterations and planning times. For Franka~C$\times$A' (denoted C$\times$A' for compactness) the planning time is the sum of Franka-C and Franka-A'.} 
\end{center}
For all problems, planning time is predominantly consumed by motion planning (encompassing both piecewise and joint motion planning), while task planning (comprising decision tree creation and dynamic programming) accounts for a comparatively small portion. However, Task Planning becomes significant in two specific cases: (1) In Baxter-D, this is because the exploration parameter $c_{mcts}$ and the number of iterations for decision tree creation are significantly increased to enable sucessfull planning despite incomplete logic, and is discussed further in Section~\ref{sec:influence_of_motion_planning_failures}; and (2) in Franka-B, this is due to the first order logic engine employed which relies heavily on string manipulations and lacks optimization for speed, leading to notable task planning times in problems with high combinatorial complexity. 

For a very small number of policy iterations (up to approximately 5), the majority of motion planning time is spent on the final step of joint trajectory optimization. However, since joint optimization is performed only once, irrespective of the number of policy iterations, piecewise trajectory optimization becomes the dominant component as the parameter $c_0$ is adjusted to allow for greater exploration, resulting in an increased number of iterations.
\subsubsection{Influence of $c_0$, Exploration vs. Exploitation}
The main parameter influencing planning is $c_0$.
A low (optimistic) value for $c_0$ leads to candidate policies with an \enquote{hypothetized} expected cost underestimating the actual piecewise trajectory costs. Once the piecewise trajectory costs are integrated into the decision tree, the next round of dynammic programming will tend to output a candidate policy largely composed of \enquote{unexplored} actions. This is visible in Fig.~\ref{fig:convergences}  with $c_0 = 0.15$. With more exploration, the search converges to policies with a lower cost. The opposite phenomenon takes place when $c_0$ is an overestimate. Actions with a cost estimate resulting from the piecewise trajectory optimization appear comparatively advantageous and are used for the next candidate policy leading to a fast convergence to a final policy (e.g.  with $c_0 = 1.5$ in Fig.~\ref{fig:convergences}), at the expense of optimality.

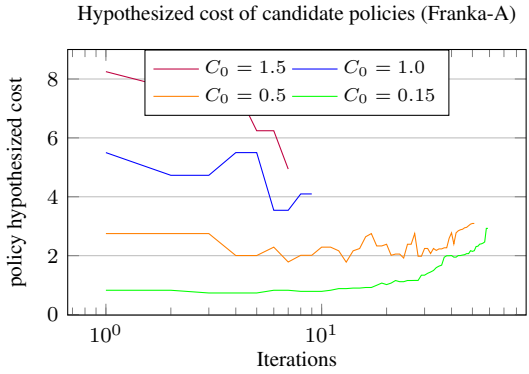
\begin{figure}[ht]
    \centering
\begin{tikzpicture}
\begin{axis}[
    width=0.9\linewidth,
    height=0.6\linewidth,
    title={Hypothesized cost of candidate policies (Franka-A)},
    xlabel={Iterations},
    xmode=log, 
    ylabel={policy hypothesized cost},
    ylabel style={at={(0.05,0.5)}, anchor=north},
    ymajorgrids=true,
    legend pos=south west,
    legend style={font=\fontsize{7}{8}\selectfont},
    legend cell align={left},
    x label style={at={(axis description cs:0.5,0.05)},anchor=north},
    legend style={at={(0.5,1.0)}, anchor=north},
    legend columns=2
]
   
\addplot[ 
    color=purple, 
    ]
    table [col sep=comma]{tamp_data/C0_analysis_1/1.5/policy-candidates.data};
    \addlegendentry{$C_0=1.5$}

\addplot[ 
    color=blue,
    ]
    table [col sep=comma]{tamp_data/C0_analysis_1/1.0/policy-candidates.data};
    \addlegendentry{$C_0=1.0$}  
    
\addplot[ 
    color=orange,
    ]
    table [col sep=comma]{tamp_data/C0_analysis_1/0.5/policy-candidates.data};
    \addlegendentry{$C_0=0.5$}  

\addplot[ 
    color=green,
    ]
    table [col sep=comma]{tamp_data/C0_analysis_1/0.15/policy-candidates.data};
    \addlegendentry{$C_0=0.15$}
          
\end{axis}
\end{tikzpicture}

\caption{\label{fig:convergences} Evolution of the hypothesized cost of the candidate policies: Over the iterations, the cost estimate of the candidate policies is refined by the result of the piecewise optimization. 
}
\end{figure}
Some actions may be infeasible leading to infinite costs, such that any choice of $c_0$ still allows the algorithm to iterate enough to generate candidate policies circumventing infeasible actions, within the limits of decision tree. 
In practice, $c_0$ is chosen empirically. Fig.~\ref{fig:iteration_cost_scatter_plot} shows the relation between $c_0$ the final piecewise trajectory costs, and the number of iterations. Planning is performed 5 times for each value of $c_0$. For low values of $c_0$, the final policy cost exhibits minimal dispersion, although the number of iterations required for convergence shows some variability. Conversely, higher values of $c_0$ lead to reduced variability in the number of iterations needed for convergence but increase the variability of the final policy cost.

\begin{figure}[ht] 
    \centering
\begin{tikzpicture}
\begin{axis}[
        ylabel={Trajectory cost},
        xmode=log, 
        ymin=0, ymax=11, 
        axis y line*=left, 
        axis x line=none, 
 		ylabel style={at={(0.125,0.5)}, anchor=south},
 		width=0.9\linewidth 
     ]
\addlegendimage{only marks, mark=square, color=blue}\addlegendentry{iterations}

\addplot+[ 
    only marks,
    mark=triangle,
    color = orange
    ]
    table [col sep=comma]{tamp_data/C0_analysis/c0_to_piecewise_cost.data};
    \addlegendentry{piecewise cost}
    
\addplot+[ 
    only marks,
    mark=o,
	color=yellow,
    ]
    table [col sep=comma]{tamp_data/C0_analysis/c0_to_joint_cost.data};
    \addlegendentry{joint opt. cost}

\end{axis}
    
\begin{axis}[
    title={Number of iterations and trajectory costs (Franka-A)},
    xlabel={$c_0$ (log scale)},
    ylabel={Number of iterations},
    xmode=log, 
    ymajorgrids=true,
    error bars/y dir=both, 
    error bars/y explicit, 
    ymin=0, ymax=105, 
    axis y line*=right, 
    ylabel style={at={(1.3,0.5)}, anchor=north},
    width=0.9\linewidth,
    legend to name=sharedlegend, 
    legend style={at={(0.5,1.0)}, anchor=north, legend columns=1}, 
]

\addplot+[ 
    only marks,
    mark=square,
    ]
    table [col sep=comma]{tamp_data/C0_analysis/c0_to_it.data};
\end{axis}

\end{tikzpicture}  
\caption{\label{fig:iteration_cost_scatter_plot}Influence of $c_0$ on the trajectory cost of $\pi^{\star}$ and the total number of iterations: A low $c_0$ (optimistic) leads to better trajectory-trees at the expense of the number of iterations.}
\end{figure}
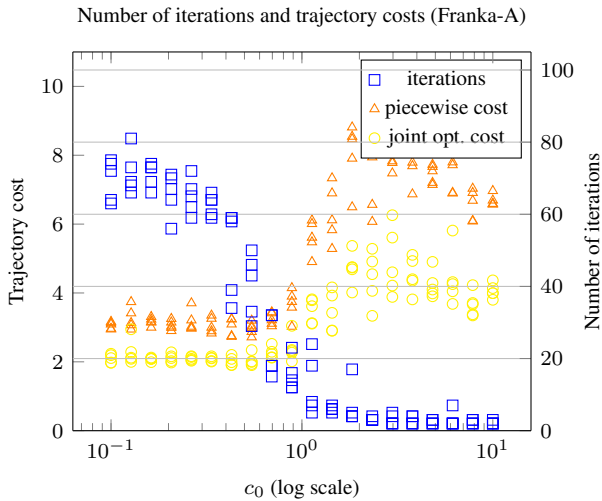
For the problem Franka-A, optimal policies are found when $c_0$ is below $1.0$.

\subsubsection{Influence of Motion Planning Failures} \label{sec:influence_of_motion_planning_failures}
In the problem Baxter-D the logic definition is relaxed to make the \textit{Look} actions available also when the robot does not grasp the block to observe, as described in Section~\ref{sec:planning_problems}. This increases the branching factor. In addition, the large majority of the planned candidate policies do not grasp the block \rrevise{before executing a \textit{Look} action, since it is symbolically not advantageous to do so. Piecewise trajectory optimization for such policies results in infinite costs for the \textit{Look} action, because the constraints~\eqref{eq:incidence_angle} and \eqref{eq:fov_centering} ensuring observability of the colored side cannot be satisfied unless the object has been grasped beforehand. The propagation of infinite costs during dynamic programming effectively rules out large portions of the decision tree, eventually favoring candidate policies that grasp the block before executing a \textit{Look} action, for which trajectory optimization succeeds.} 
The algorithm still converges to a feasible policy as good as in Baxter-C, when $c_{mcts}$ and the number of iterations for the decision tree creation are large enough. 

This is an important quality of the proposed solution. Adding domain specific knowledge in the task planning (to ensure that motion planning will succeed) will, in general, speed up the search.
However, we think that is it not always possible, nor convenient to incorporate geometric reasoning
(reachability of a view point, reachability of an object) in
the logical reasoning.

\subsubsection{Benefits of the Joint Optimization}
Joint optimization results in trajectory-trees being optimized across observation branchings, and are more optimized towards likely contingencies as it is visible in Fig.~\ref{fig:joint-obs-branching}. Furthermore, the optimization spans actions and kinematic switches, like in the fully observable LGP framework. This allows trajectory elements corresponding to a given action to be influenced by subsequent actions, ultimately resulting in lower overall trajectory costs compared to piecewise optimization. This is shown in Fig.~\ref{fig:joint-obs-anticipation}.

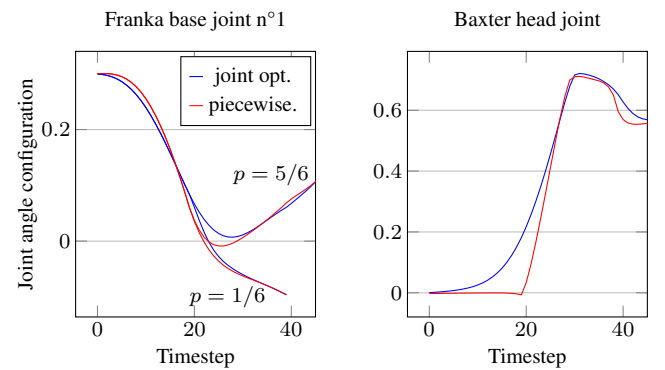
\begin{figure}[ht]
    \centering
\subfloat[The jointly optimized traj.-tree for policy \ref{fig:policy-franka-a} is more optimized towards the likely branch.\label{fig:joint-obs-branching}]{
\begin{tikzpicture}
\begin{axis}[
    width=0.56\linewidth,
    height=0.6\linewidth,
    outer sep=0pt,
	inner sep=2pt,
    title={Franka base joint n°1},
    xlabel={Timestep},
    ylabel={Joint angle configuration},
    xtick={0, 20, 40},
    ymajorgrids=true,
    x label style={at={(axis description cs:0.5,0.05)},anchor=north},
    ylabel style={at={(0.125,0.5)}, anchor=north},
    xmax=45,
    legend columns=1,
    legend image post style={scale=0.3} 
]
   
        \addplot[blue] table [col sep=comma]{tamp_data/joint_vs_markov/franka_joint/joint_1_branch_5_q.data};
        \addlegendentry{joint opt.}
        \addplot[blue, forget plot] table [col sep=comma]{tamp_data/joint_vs_markov/franka_joint/joint_1_branch_4_q.data};

        \addplot[red] table [col sep=comma]{tamp_data/joint_vs_markov/franka_markov/joint_1_branch_5_q.data};
        \addlegendentry{piecewise.}
        \addplot[red, forget plot] table [col sep=comma]{tamp_data/joint_vs_markov/franka_markov/joint_1_branch_4_q.data};
        
        \node[anchor=west] at (axis cs:27.0, 0.12) {$p=5/6$};
        \node[anchor=west] at (axis cs:18.0, -0.1) {$p=1/6$};

\end{axis}
\end{tikzpicture}
}
\hfill
\subfloat[The jointly optimized traj.-tree for \ref{fig:policy-baxter-b} anticipates the \textit{Look} action for t~$\geq$~20.\label{fig:joint-obs-anticipation}]{\begin{tikzpicture}
\begin{axis}[
    width=0.56\linewidth,
    height=0.6\linewidth,
    outer sep=0pt,
	inner sep=2pt,
    title={Baxter head joint},
    xlabel={Timestep},
    ymajorgrids=true,
    x label style={at={(axis description cs:0.5,0.05)},anchor=north},
    xmax=45
]
\addplot[blue] table [col sep=comma]{tamp_data/joint_vs_markov/baxter_joint/joint_0_branch_0_q.data};
    
    \addplot[red] table [col sep=comma]{tamp_data/joint_vs_markov/baxter_markov/joint_0_branch_0_q.data};
    
\end{axis}
\end{tikzpicture}}
\caption{\label{fig:joint_vs_piecewise_optimization} Comparison of joint vs. piecewise optimization: The joint optimization accounts for the branching probability (based on the belief state) and optimizes the trajectory globally over all actions. This leads to smoother motions.}
\end{figure}

Overall, the expected costs of the jointly optimized trajectory-trees are 40\% inferior compared to the piecewise trajectory-trees, as illustrated in Fig.~\ref{fig:iteration_cost_scatter_plot}. Similar cost reduction is observed for the problems using the Baxter. Importantly, we observe that the lowest joint trajectory costs are obtained with the policies having the lowest piecewise trajectory costs. This validates empirically the approach of selecting a policy on the basis of the piecewise cost only, and performing the joint optimization as a final step. Performing the joint optimization appear particularly advantageous when optimizing a macro-action like Franka-A', since it is typically used multiple times by the overarching task plan (Franka-C). Optimizing macro-actions to a high degree is therefore especially meaningful. 
\begin{figure}[htbp]
    \centering
    \begin{minipage}{0.475\linewidth}
\begin{center}
\footnotesize
\setlength\tabcolsep{1pt} 
\begin{tabular}{|c||c|c|c|}
\hline
\thead{} & \thead{No\\decomp.} & \thead{Subtrees} & \thead{Branches}\\
\hline
Baxt.-B & 1.24 & 0.96 & 1.19 \\
\hline
Baxt.-C & 5.28 & 4.81 & 6.81 \\
\hline
Frank.-A  & 1.38 & 2.22 & 2.99 \\
\hline
Frank.-C & 4.74 & 3.20 & 3.8 \\					    
\hline
\end{tabular}
\captionof{table}{Optimization time for different decomposition schemes (in seconds)\label{table:distrib_vs_joint}} 
\end{center}
    \end{minipage}
    \hfill 
    \begin{minipage}{0.475\linewidth}
        \centering
        \includegraphics[width=\textwidth]{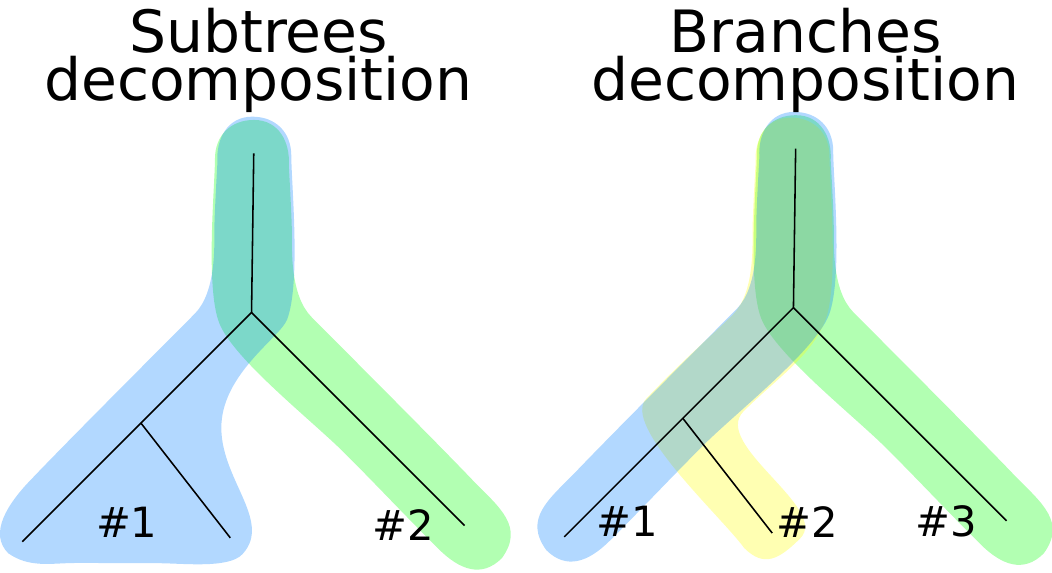} 
        \caption{Illustration of the optimization decomposition schemes\label{fig:decomposition_schemes}}
    \end{minipage}
\end{figure}

With the D-AuLa solver, the joint optimization can be decomposed into sub-optimization problems, as shown in Fig.~\ref{fig:decomposition_schemes}. Optimization time are indicated in Table~\ref{table:distrib_vs_joint}. Unlike the MPC case, the decomposition does not demonstrate a clear advantage. This can be understood by considering that the planned policies are typically less decomposable than in MPC: there are potentially multiple branchings, which can take place at any depth on the tree. Franka-C is the problem which would benefit the most from optimization decomposition (with the Subtree-decomposition), but decomposition can be detrimental, e.g. for Franka-A with dense branching structure (see Fig.~\ref{fig:policy-franka-a}). We therefore do not decompose by default as explained in Section~\ref{sec:tree_komo}.

\revise{\subsubsection{Scalability}
Fig.~\ref{fig:scalability:combined_figure} shows the planning times observed with variations of the Baxter problems. The measurements for the  belief state sizes of one, two, and six  are obtained with the problems Baxter-A, Baxter-B and Baxter-C from the Table.~\ref{table:problems}. Additional data points for belief state sizes of three, four, and five were measured by removing hypotheses from the initial belief state of the Baxter-C problem, thereby enabling a finer-grained analysis of scalability between the Baxter-B and Baxter-C problem variations.}
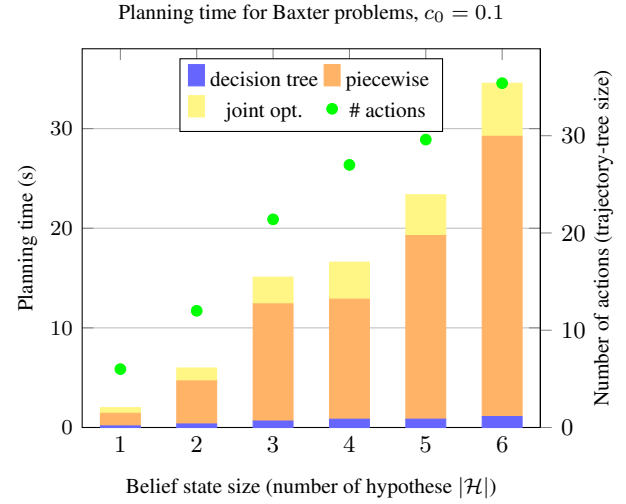
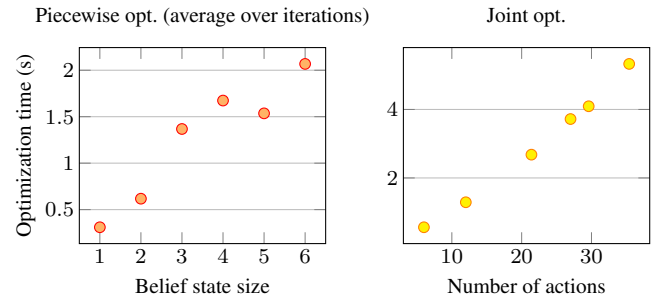
\begin{figure}[ht]
    \centering
    \subfloat[\revise{Cumulated planning time and trajectory-tree size vs. the number of hypotheses~$|\mathcal{H}|$.}\label{fig:scalability:overall}]{%
        \begin{tikzpicture}
        \begin{axis}[
            ybar stacked,
            bar width=15pt,
            enlarge x limits=0.1,
            ymin=0,
            title={Planning time for Baxter problems, $c_0 = 0.1$ },
            ylabel={Planning time (s)},
            xlabel={Belief state size (number of hypothese $|\mathcal{H}|$)},
            xtick={1,2,3,4,5,6},
            xticklabel style={font=\small},
            ylabel style={at={(0.125,0.5)}, anchor=south},
            width=0.9\linewidth,
            legend style={at={(0.5,0.975)}, anchor=north, legend columns=2},
            ymajorgrids=true,
        ]

        \addplot+[ybar, fill=blue!60, color=blue!60] coordinates         {
        (1, 0.1583932)
        (2, 0.3545638)
        (3, 0.647826)
        (4, 0.8311378)
        (5, 0.837203)
        (6, 1.1009498)
        };
        \addlegendentry{decision tree}
        
        \addplot+[ybar, fill=orange!60, color=orange!60] coordinates {
        (1,1.265966)
        (2,4.322138)
        (3,11.762)
        (4,12.05)
        (5,18.42968)
        (6,28.1242)
        };
        \addlegendentry{piecewise}

        \addplot+[ybar, fill=yellow!60, color=yellow!60] coordinates {
        (1, 0.5664732)
        (2, 1.294632)
        (3, 2.684448)
        (4, 3.720284)
        (5, 4.093164)
        (6, 5.335992)
        };
        \addlegendentry{joint opt.}
       
        \addlegendimage{only marks, color=green, legend image post style={yshift=-1.1ex}}
        \addlegendentry{\# actions}

        \end{axis}

        \begin{axis}[
            axis y line*=right,
            axis x line=none,
            ymin=0,
            ylabel={Number of actions (trajectory-tree size)},
            ylabel style={at={(1.3,0.5)}, anchor=north},
            yticklabel style={font=\small},
            tick align=outside,
            width=0.9\linewidth,
        ]
        \addplot+[only marks, mark options={fill=green}, color=green] coordinates {
        (1, 6)
        (2, 12)
        (3, 21.4)
        (4, 27)
        (5, 29.6)
        (6, 35.4)
        };
        \end{axis}
        \end{tikzpicture}
    }
	\hfill
    \subfloat[\revise{Piecewise optimization: Performed once per iteration; the plot shows the average over iterations.}\label{fig:scalability:piecewise}]{%
        \begin{tikzpicture}
    \centering
\begin{axis}[
	outer sep=0pt,
	inner sep=2pt,
    width=0.57\linewidth,
    title={Piecewise opt. (average over iterations)},
    xlabel={Belief state size},
    ylabel={Optimization time (s)},
    xtick={1, 2, 3, 4, 5, 6},
    ymajorgrids=true,
    x label style={at={(axis description cs:0.5,0.05)},anchor=north},
    ylabel style={at={(0.1,0.5)}, anchor=north},
    legend columns=1,
    legend image post style={scale=0.3} 
]
   
      \addplot+[only marks, mark options={fill=orange!60}, color=red] coordinates {
        (1, 0.31)
        (2, 0.617448285714286)
        (3, 1.36767441860465)
        (4, 1.67361111111111)
        (5, 1.53580666666667)
        (6, 2.06795588235294)
        }; 

\end{axis}
\end{tikzpicture}
    }
    \subfloat[\revise{Joint optimization: Scales linearly with the number of actions in the trajectory-tree.}\label{fig:scalability:joint}]{%
       \begin{tikzpicture}
    \centering
\begin{axis}[
	outer sep=0pt,
	inner sep=2pt,
    width=0.57\linewidth,
    title={Joint opt.},
    xlabel={Number of actions},
    ylabel={},
    ymajorgrids=true,
    x label style={at={(axis description cs:0.5,0.05)},anchor=north},
    ylabel style={at={(0.1,0.5)}, anchor=north},
    legend columns=1,
    legend image post style={scale=0.3} 
]
   
      \addplot+[only marks, mark options={fill=yellow}, color=orange] coordinates {
        (6, 0.56)
        (12, 1.29)
        (21.4, 2.68)
        (27, 3.72)
        (29.6, 4.093164)
        (35.4, 5.33)
        }; 

\end{axis}
\end{tikzpicture}
    }

    \caption{\revise{Planning time for variations of the Baxter problems: Planning time is dominated by the iterations of piecewise optimization. Reported values are averaged over five runs.}}
    \label{fig:scalability:combined_figure}
\end{figure}

\revise{The joint optimization time exhibits linear scaling with the number of hypotheses $|\mathcal{H}|$ of the belief state. This linear relationship is consistent with the theoretical scalability analysis discussed in Section~\ref{sec:tree_komo}. It becomes even more apparent in Fig.~\ref{fig:scalability:joint}, where the data is plotted against the number of actions in the trajectory-tree. While the absolute number of actions remains relatively modest, each action comprises 20 configurations, with each configuration represented by an 8-dimensional vector (for the robot's 8 degrees of freedom). As a result, a trajectory-tree with 35 actions leads to a joint optimization problem with 5600 decision variables.}

\revise{The overall planning time is primarily dominated by the piecewise trajectory optimization performed during each iteration of the policy search.  When averaged over the number of iterations, the optimization time exhibits a linear trend with respect to $|\mathcal{H}|$, as illustrated in Fig.~\ref{fig:scalability:piecewise}. It is important to note, that in each iteration, only newly introduced actions are subject to optimization; action edges that are shared with previously evaluated policies are not re-optimized. Consequently, the time required for piecewise optimization is strongly influenced by the degree of overlap between successive candidate policies. This results in a less direct correlation with the number of hypotheses, and accounts for the less monotonic behavior observed in Fig.~\ref{fig:scalability:piecewise}, in contrast to Fig.~\ref{fig:scalability:joint}.}

\revise{While the observed trend with respect to the number of hypotheses is linear, it is important to note that in these block-stacking experiments, the number of hypotheses itself grows factorially with the number of blocks. This is exemplified by the experiment Franka-B, where the belief state is aready of size 72 for two blocks since the colored side can be any side. This underscores the challenge of scalability in highly combinatorial domains. The Franka-C$\times$A' experiment, discussed in the next section, investigates how hierarchical decomposition can mitigate this complexity and enhance scalability.}

\subsubsection{Improved Scalability via Hierarchical Planning} \label{sec:scalability_through_composition}
As the problem Franka-B shows, one reaches the limits of tractability for problems with a large belief state size. This is inherent to the nature of the approach, since we aim for trajectory-trees covering all contingencies. Introducing hierarchy Franka-C$\times$A' breaks down the complexity and enables scaling to larger problems.

Firstly, we plan for the higher-level policy (Franka-C), where the \textit{Look} action constraining the sensor movement is replaced by an \textit{Identify-block} action. This action becomes available once the robot placed a block on the second table, in a fixed configuration which will be the start configuration for Franka-A'. Symbolically, the effect of this action is that the block is identified. Motion planning for this action is not planned at this stage.

Subsequently, we plan the low-level policy (Franka-A') implementing the \textit{Identify-Block} action. It is planned from a specified start configuration and, importantly, a fixed final configuration. Fixing the final configuration is atypical in Logic-Geometric Programming (LGP), where it is generally defined implicitly and determined as part of the optimization process. However, in this case, a fixed final configuration is necessary to enable the higher-level policy (Franka-C) to be planned independently of Franka-A'. This \textit{de-facto} boils down to optimizing a trajectory-graph.

With this decomposition in place, planning is possible in approximately 36 seconds as indicated in Table~\ref{table:times}. The resulting trajectory-tree can handle any of the 1296 possible starting configurations.
\subsection{Discussion}
These experiments demonstrate the feasibility of planning trajectory-trees in belief space, with branchings dependent on the received observations. This enables the resolution of challenging TAMP problems characterized by a strongly multimodal partially observable structure. Trajectory-trees combine explorative actions (primarily sensor trajectories) with exploitative actions (e.g., grasping and placing).
The degree of exploration over the space of all possible manipulation policies is controlled by the cost initialization parameter $c_0$ allowing for flexible trade-offs between exploration and planning time. For problems of moderate size (e.g., belief state sizes up to 6), planning times remain within several seconds, highlighting the method's practical applicability.

One limitation of the current approach is its reliance on a predefined set of world hypotheses, which must be available at the start of planning, along with a suitable observation model. While this assumption is valid for many structured applications, its relaxation would broaden the applicability of the method to more dynamic or unstructured scenarios. 

\revise{The approach optimizes trajectory-trees covering all possible contingencies. This is well-suited to the presented examples, where the different modalities are of equal importance. However, in domains where some contingencies are intrinsically unlikely, this strategy may become computationally inefficient. A potential solution would be to fully optimize only a subset of the trajectory-tree that accounts for a target probability mass, accepting a trade-off in the form of occasional replanning for rare cases, thereby sharing similarities with \citep{shah2020anytime}.}
\added{More generally, this idea can be viewed as a particular instance of trajectory-trees with partial contingency coverage, where only a subset of contingencies is represented or optimized explicitly. Beyond improving scalability, this could naturally extend the framework toward richer classes of POMDPs, as discussed further in Section~\ref{sec:towards_partial_contingencies}.}

We think that \deleted{this}\added{the presented} approach is particularly well-suited for optimizing mid-term policies or macro-actions, which aim to reduce the entropy of the belief state, as exemplified by the Franka-A problem. Having only limited horizon and branching, such macro-actions are fast to optimize (potentially even online), and can be integrated into a higher-level TAMP planning outer loop, effectively decoupling the higher-level planner from the complexities of handling lower-level partial observability. This decoupling enhances scalability and simplifies the higher-level planning process.
The practicality of this concept is illustrated in the Franka-C$\times$A' problem, where a total planning time of 36 seconds was achieved, for a trajectory-tree addressing 1,296 contingencies.

The planning time is primarily dominated by the piecewise trajectory planning step during the policy search. A promising direction for future work could involve learning to predict promising policies, as demonstrated within the LGP framework for the fully observable case in \citep{tamp_lgp_6_driess2021learning}.

\revise{\subsection{Generalizations Beyond Deterministic Dynamics} \label{sec:tamp_generalizablization}
The PO-LGP framework assumes deterministic transitions, of both the continuous and symbolic states, as explained in Section~\ref{sec:intro_trajectory_tree_state}, and Section~\ref{sec:tamp_decision_process}. It can also be extended beyond these assumptions:}

\medskip
\noindent\revise{\textbf{Continuous dynamics\normalfont{:}} In cases where the continuous dynamics are stochastic but stabilizable, due to statistically bounded disturbances, such as Gaussian noise, the trajectory-tree structure can be extended by attaching feedback controllers to its edges (which may be co-optimized) to track the nominal trajectories, thereby sharing similarities with the FIRM approach \citep{agha2014firm}. As long as the uncertainty remains sufficiently concentrated, the underlying multimodal structure and branching of the trajectory-tree remain unaffected.}

\medskip
\noindent\revise{\textbf{Stochastic observation model\normalfont{:}} A stochastic observation model increases the branching factor at observation nodes, resulting in a larger decision tree. It also necessitates defining goal conditions probabilistically in belief space, to account for the possibility that the belief may not concentrate fully, and residual non-zero probabilities may persist across a long tail of states. While this generalization does not alter the structure of the framework, it may raise tractability challenges due to the expanded size of the trajectory-tree.}

\medskip
\noindent\revise{\textbf{Stochastic symbolic actions\normalfont{:}} With stochastic symbolic actions, the system may transition to different modalities after each action. This can be modeled by defining a tuple of cost and constraint functions $(c, g, h)$ for each possible outcome, thereby implicitly specifying the corresponding continuous evolution. While this formulation is  feasible in principle, it may come with practical challenges when action stochasticity represents a transition to failure modes, which are often unpredictable and difficult to model.}

\medskip
\revise{These considerations suggest that extending the framework to accommodate stochastic continuous dynamics and/or stochastic observation models is structurally compatible with the existing trajectory-tree formulation. By comparison, symbolic action stochasticity---particularly when involving transitions to failure modes---introduces modeling and specification challenges, and may be more naturally addressed through replanning mechanisms, rather than direct integration into the tree structure.}
\added{These extensions would generally require representing a larger number of contingencies, leading to larger trajectory-trees and increasing the scalability challenges associated with complete contingency coverage. This naturally motivates the discussion of trajectory-trees with partial contingency coverage in Section~\ref{sec:towards_partial_contingencies}.}

\section{Conclusion}


\added{\subsection{Toward Trajectory-Trees with Partial Contingency Coverage}\label{sec:towards_partial_contingencies}}
\added{In its current formulation, the framework assumes that partial observability can be represented by a finite set of explicit hypotheses. 
Furthermore, the number of contingencies represented in the  trajectory-tree is kept manageable through the QMDP assumption in MPC and the assumption of deterministic symbolic transitions in TAMP.}

\added{Under these assumptions, the presented framework optimizes complete trajectory-trees, in the sense that they explicitly represent all contingencies. While this formulation is well suited to problems involving a moderate number of explicit hypotheses, its computational cost increases rapidly with the number of modeled hypotheses. Furthermore, the simplifying assumptions restrict the class of problems that the framework can address.}

\added{A natural extension would be to construct trajectory-trees with partial contingency coverage, in which only a subset of contingencies is represented explicitly, for example by sampling future belief evolutions. This would partially decouple the computational effort from the size of the underlying belief space. Furthermore, sampling provides a natural discretization mechanism through which potentially continuous observation models, stochastic transitions, or richer belief dynamics can be approximated by a finite set of explicit contingencies amenable to joint trajectory-tree optimization. This would therefore not only improve the scalability of the framework but also extend its applicability to domains with large or continuous belief spaces, or stochastic transitions.}

\added{
In MPC, this could be introduced straightforwardly while retaining the QMDP assumption by constructing trajectory-trees from a fixed number of sampled hypotheses, thereby bounding the computational requirements. More generally, this would facilitate an extension toward richer planning behaviors by explicitly reasoning about anticipated belief evolution, arising, for example, from future observations or interactions with other agents \footnote{\added{To fully leverage richer models of belief evolution, the QMDP approximation could be relaxed, or alternatively applied only after a fixed number of explicitly represented branching stages, allowing a trade-off between computational complexity and planning expressiveness.}}.
}

\added{In TAMP, constructing trajectory-trees with partial contingency coverage would extend the framework toward a broader class of POMDP models.}
\added{However, this would introduce a direct reliance on frequent replanning, since the optimized policies would intentionally represent only a subset of the reachable belief states.}

\added{
Such an extension would bring the proposed framework conceptually closer to recent online macro-action POMDP planners, while retaining its distinctive trajectory-optimization-based formulation of motion generation.}

\subsection{Concluding Remarks}
In this article, we propose a new approach to trajectory optimization for problems with a multimodal partial observability, which consists in optimizing trajectory-trees in belief space.

For MPC, the control problem is formulated with a fixed tree structure having an early branching (PO-MPC). Trajectory-trees are optimized with respect to a belief state provided externally. Given the critical runtime requirements, we developed a specialized optimization algorithm~(D-AuLa) which exploits the decomposability  of these trees to significantly improve computational efficiency. 
We believe these contributions can enable the adoption of tree-like controls (otherwise known as Multi-Stage MPC) to more robotic uses cases, such as partially observable problems.

In TAMP, where the challenge resides in the efficient integration of the symbolic and geometric reasoning, we developed an integrated planner (PO-LGP), which incorporates belief state inference to enable the planning of information gathering actions. The planner fully encompasses the contingent nature of the partially observable problem, by reasoning on trees both at a task and motion planning level. The optimization of the trajectory-trees is performed accross observation branchings and across task modes using a new transcription (T-KOMO). 
These contributions enable long-term planning for TAMP problems where existing approaches are either not applicable, or would rely  heavily on replanning.

Together, these contributions are summarized in Table~\ref{table:contibution_table}. \revise{While MPC and TAMP are addressed with distinct methodologies, certain use cases transcend this boundary. For instance, complex driving scenarios may require symbolic decision-making akin to TAMP, while object manipulation tasks may demand fast, online trajectory-tree replanning similar to MPC. From this perspective, the presented contributions form a complementary set of techniques and provide a unified perspective for addressing the challenges of trajectory optimization under multimodal partial observability.}

\begin{table*}[htbp]
\centering
\footnotesize
\setlength\tabcolsep{1pt} 
\begin{tabular}{|c||c||c|c|c|c||c|}
\hline
\thead{} & \thead{Challenges} & \thead{Belief state \\update} & \thead{Traj. tree\\structure} & \thead{Trajectory\\transcription}& \thead{Importance of \\ joint optimization} & \thead{Contributions}\\
\hline
MPC & \thead{Performance \\vs.\\ Conservativness\\-\\Real-time  planning} & \thead{External \\ (provided by \\perception \\ module)} & Fixed & \thead{Indirect single shooting\\or\\ Direct (KOMO)} &Essential& \thead{PO-MPC (high-level problem formulation) \\ D-AuLa (low-level optimization) }\\
\hline
TAMP & \thead{Long term planning in\\ highly contingent domains\\-\\Scalability} & \thead{Internal \\ (exploration \\ is planned)} & \thead{Determined by\\task planner} & Direct (T-KOMO) &\thead{Trajectory-tree\\improvement}& \thead{PO-LGP (high-level integrated planner) \\ T-KOMO (mid-level traj.-tree-transcription)}\\				    
\hline
\end{tabular}
\captionof{table}{Overview of research challenges and contributions. The rows 3 to 6 outline salient aspects of the method.\label{table:contibution_table}} 
\end{table*}

\bibliographystyle{SageH} 
\bibliography{references} 

@article{soc-1,
author = {Mesbah, Ali},
year = {2016},
month = {12},
pages = {30-44},
title = {\href{https://ieeexplore.ieee.org/document/7740982}{Stochastic Model Predictive Control: An Overview and Perspectives for Future Research}},
volume = {36},
journal = {IEEE Control Systems},
doi = {10.1109/MCS.2016.2602087}
}

@article{soc-2,
  title={\href{https://www.sciencedirect.com/science/article/pii/S0098135417303812}{Stochastic model predictive control - how does it work?}},
  author={Tor Aksel N. Heirung and Joel A. Paulson and Jared O'Leary and Ali Mesbah},
  journal={Computers and Chemical Engineering},
  year={2017},
  volume={114},
  pages={158-170}
}

@article{robust-1,
  title={\href{https://ieeexplore.ieee.org/document/4789462}{Robust Model Predictive Control}},
  author={Peter J. Campo and Manfred Morari},
  journal={1987 American Control Conference},
  year={1987},
  pages={1021-1026}
}

@article{min-max-1,
author = {Löfberg, Johan},
year = {2003},
month = {01},
pages = {},
title = {\href{https://www.researchgate.net/publication/228869742_Min-Max_Approaches_to_Robust_Model_Predictive_Control}{Min-Max Approaches to Robust Model Predictive Control}},
journal = {Linköping Studies in Science and Technology. Dissertations. No}
}

@article{langson2004robust,
  title={Robust model predictive control using tubes},
  author={Langson, Wilbur and Chryssochoos, Ioannis and Rakovi{\'c}, SV and Mayne, David Q},
  journal={Automatica},
  volume={40},
  number={1},
  pages={125--133},
  year={2004},
  publisher={Elsevier}
}

@article{scokaert1998min,
  title={Min-max feedback model predictive control for constrained linear systems},
  author={Scokaert, Pierre OM and Mayne, David Q},
  journal={IEEE Transactions on Automatic control},
  volume={43},
  number={8},
  pages={1136--1142},
  year={1998},
  publisher={IEEE}
}

@article{multi-stage-luca-1-LUCIA201269,
title = "A new Robust NMPC Scheme and its Application to a Semi-batch Reactor Example*",
journal = "IFAC Proceedings Volumes",
volume = "45",
number = "15",
pages = "69 - 74",
year = "2012",
note = "8th IFAC Symposium on Advanced Control of Chemical Processes",
issn = "1474-6670",
doi = "https://doi.org/10.3182/20120710-4-SG-2026.00035",
url = "http://www.sciencedirect.com/science/article/pii/S1474667016304219",
author = "S. Lucia and T. Finkler and D. Basak and S. Engell"
}

@article{multi-stage-luca-1-LUCIA20131306,
title = "Multi-stage nonlinear model predictive control applied to a semi-batch polymerization reactor under uncertainty",
journal = "Journal of Process Control",
volume = "23",
number = "9",
pages = "1306 - 1319",
year = "2013",
issn = "0959-1524",
doi = "https://doi.org/10.1016/j.jprocont.2013.08.008",
url = "http://www.sciencedirect.com/science/article/pii/S0959152413001686",
author = "Sergio Lucia and Tiago Finkler and Sebastian Engell"
}

@article{multi-stage-luca-1,
author = {Lucia, Sergio and Andersson, Joel and Brandt, Heiko and Diehl, Moritz and Engell, Sebastian},
year = {2014},
month = {08},
pages = {},
title = {\href{https://www.sciencedirect.com/science/article/pii/S0959152414001450}{Handling uncertainty in economic nonlinear model predictive control: A comparative case study}},
volume = {24},
journal = {Journal of Process Control},
doi = {10.1016/j.jprocont.2014.05.008},
}

@article{multi-stage-klintberg-1,
author = {Kouzoupis, Dimitris and Klintberg, Emil and Diehl, Moritz and Gros, Sebastien},
year = {2017},
month = {12},
pages = {},
title = {\href{https://www.researchgate.net/publication/321984963_A_dual_Newton_strategy_for_scenario_decomposition_in_robust_multistage_MPC}{A dual Newton strategy for scenario decomposition in robust multistage MPC}},
journal = {International Journal of Robust and Nonlinear Control},
doi = {10.1002/rnc.4019}
}

@ARTICLE{multi-stage-klintberg-2,  author={E. {Klintberg} and S. {Gros}},  journal={IEEE Transactions on Control Systems Technology},   title={A Parallelizable Interior Point Method for Two-Stage Robust MPC},   year={2017},  volume={25},  number={6},  pages={2087-2097},}

@inproceedings{multi-stage-leidereiter-1,
author = {Leidereiter, Conrad and Potschka, Andreas and Bock, Hans},
year = {2015},
month = {07},
pages = {1608-1613},
title = {\href{https://ieeexplore.ieee.org/document/7330767}{Dual decomposition for QPs in scenario tree NMPC}},
doi = {10.1109/ECC.2015.7330767}
}

@INPROCEEDINGS{multi-stage-robotics-1,
author={S. {Subramanian} and S. {Nazari} and M. A. {Alvi} and S. {Engell}},
booktitle={2018 23rd International Conference on Methods   Models in Automation   Robotics (MMAR)},
title={Robust NMPC Schemes for the Control of Mobile Robots in the Presence of Dynamic Obstacles},
year={2018},
volume={},
number={},
pages={768-773},}

@INPROCEEDINGS{mpc_100,
  author={Ulfsjöö, Carl Hynén and Axehill, Daniel},
  booktitle={2022 IEEE Intelligent Vehicles Symposium (IV)}, 
  title={On Integrating POMDP and Scenario MPC for Planning under Uncertainty – with Applications to Highway Driving}, 
  year={2022},
  volume={},
  number={},
  pages={1152-1160},
  doi={10.1109/IV51971.2022.9827005}}

@INPROCEEDINGS{mpc_104,
  author={Oliveira, Rui and Nair, Siddharth H. and Wahlberg, Bo},
  booktitle={2023 IEEE Intelligent Vehicles Symposium (IV)}, 
  title={Interaction and Decision Making-aware Motion Planning using Branch Model Predictive Control}, 
  year={2023},
  volume={},
  number={},
  pages={1-8},
  doi={10.1109/IV55152.2023.10186633}}

@article{mpc_105,
  title={Interactive multi-modal motion planning with branch model predictive control},
  author={Chen, Yuxiao and Rosolia, Ugo and Ubellacker, Wyatt and Csomay-Shanklin, Noel and Ames, Aaron D},
  journal={IEEE Robotics and Automation Letters},
  volume={7},
  number={2},
  pages={5365--5372},
  year={2022},
  publisher={IEEE}
}

@article{wachter2006implementation_ipopt,
  title={On the implementation of an interior-point filter line-search algorithm for large-scale nonlinear programming},
  author={W{\"a}chter, Andreas and Biegler, Lorenz T},
  journal={Mathematical programming},
  volume={106},
  pages={25--57},
  year={2006},
  publisher={Springer}
}

@article{lucia2017rapid_do_mpc,
  title={Rapid development of modular and sustainable nonlinear model predictive control solutions},
  author={Lucia, Sergio and T{\u{a}}tulea-Codrean, Alexandru and Schoppmeyer, Christian and Engell, Sebastian},
  journal={Control Engineering Practice},
  volume={60},
  pages={51--62},
  year={2017},
  publisher={Elsevier}
}

@article{stellato2020_osqp,
  title={OSQP: An operator splitting solver for quadratic programs},
  author={Stellato, Bartolomeo and Banjac, Goran and Goulart, Paul and Bemporad, Alberto and Boyd, Stephen},
  journal={Mathematical Programming Computation},
  volume={12},
  number={4},
  pages={637--672},
  year={2020},
  publisher={Springer}
}

@article{tamp_tamp_1_garrett2021integrated,
  title={Integrated task and motion planning},
  author={Garrett, Caelan Reed and Chitnis, Rohan and Holladay, Rachel and Kim, Beomjoon and Silver, Tom and Kaelbling, Leslie Pack and Lozano-P{\'e}rez, Tom{\'a}s},
  journal={Annual review of control, robotics, and autonomous systems},
  volume={4},
  number={1},
  pages={265--293},
  year={2021},
  publisher={Annual Reviews}
}

@article{tamp_tamp_2_guo2023recent,
  title={Recent trends in task and motion planning for robotics: A survey},
  author={Guo, Huihui and Wu, Fan and Qin, Yunchuan and Li, Ruihui and Li, Keqin and Li, Kenli},
  journal={ACM Computing Surveys},
  volume={55},
  number={13s},
  pages={1--36},
  year={2023},
  publisher={ACM New York, NY}
}

@inproceedings{tamp_lgp_1_toussaint2015logic,
  title={Logic-Geometric Programming: An Optimization-Based Approach to Combined Task and Motion Planning.},
  author={Toussaint, Marc},
  booktitle={IJCAI},
  pages={1930--1936},
  year={2015}
}

@inproceedings{tamp_lgp_2_toussaint2017multi,
  title={Multi-bound tree search for logic-geometric programming in cooperative manipulation domains},
  author={Toussaint, Marc and Lopes, Manuel},
  booktitle={2017 IEEE International Conference on Robotics and Automation (ICRA)},
  pages={4044--4051},
  year={2017},
  organization={IEEE}
}

@inproceedings{tamp_lgp_3_18-toussaint-RSS,
  title = {Differentiable Physics and Stable Modes for Tool-Use and
  		  Manipulation Planning},
  author = {Toussaint, Marc and Allen, Kelsey R and Smith, Kevin A and Tenenbaum, Josh B},
  booktitle = {Proc{.} of Robotics: Science and Systems (R:SS)},
  note = {\emph{Best Paper Award}},
  year = {2018},
  youtube = {-L4tCIGXKBE}
}

@article{tamp_lgp_6_driess2021learning,
  title={Learning to solve sequential physical reasoning problems from a scene image},
  author={Driess, Danny and Ha, Jung-Su and Toussaint, Marc},
  journal={The International Journal of Robotics Research},
  volume={40},
  number={12-14},
  pages={1435--1466},
  year={2021},
  publisher={SAGE Publications Sage UK: London, England}
}

@article{tamp_po_0_kaelbling2013integrated,
  title={Integrated task and motion planning in belief space},
  author={Kaelbling, Leslie Pack and Lozano-P{\'e}rez, Tom{\'a}s},
  journal={The International Journal of Robotics Research},
  volume={32},
  number={9-10},
  pages={1194--1227},
  year={2013},
  publisher={Sage Publications Sage UK: London, England}
}

@inproceedings{tamp_po_1_garrett2020online,
  title={Online replanning in belief space for partially observable task and motion problems},
  author={Garrett, Caelan Reed and Paxton, Chris and Lozano-P{\'e}rez, Tom{\'a}s and Kaelbling, Leslie Pack and Fox, Dieter},
  booktitle={2020 IEEE International Conference on Robotics and Automation (ICRA)},
  pages={5678--5684},
  year={2020},
  organization={IEEE}
}

@inproceedings{tamp_po_3_adu2021probabilistic,
  title={Probabilistic inference in planning for partially observable long horizon problems},
  author={Adu-Bredu, Alphonsus and Devraj, Nikhil and Lin, Pin-Han and Zeng, Zhen and Jenkins, Odest Chadwicke},
  booktitle={2021 IEEE/RSJ International Conference on Intelligent Robots and Systems (IROS)},
  pages={3154--3161},
  year={2021},
  organization={IEEE}
}

@phdthesis{tamp_po_3bis_adu2023long,
  title={Long-Horizon Planning Under Uncertainty and Geometric Constraints for Mobile Manipulation by Autonomous Humanoid Robots},
  author={Adu-Bredu, Alphonsus Antwi},
  year={2023}
}

@inproceedings{tamp_po_4_hadfield2015modular,
  title={Modular task and motion planning in belief space},
  author={Hadfield-Menell, Dylan and Groshev, Edward and Chitnis, Rohan and Abbeel, Pieter},
  booktitle={2015 IEEE/RSJ International Conference on Intelligent Robots and Systems (IROS)},
  pages={4991--4998},
  year={2015},
  organization={IEEE}
}

@inproceedings{shah2020anytime,
  title={Anytime integrated task and motion policies for stochastic environments},
  author={Shah, Naman and Vasudevan, Deepak Kala and Kumar, Kislay and Kamojjhala, Pranav and Srivastava, Siddharth},
  booktitle={2020 IEEE International Conference on Robotics and Automation (ICRA)},
  pages={9285--9291},
  year={2020},
  organization={IEEE}
}

@article{determinization_keyder2008hmdpp,
  title={The HMDPP planner for planning with probabilities},
  author={Keyder, Emil and Geffner, Hector},
  year={2008}
}

@inproceedings{determinization_platt2010belief,
  title={Belief space planning assuming maximum likelihood observations.},
  author={Platt Jr, Robert and Tedrake, Russ and Kaelbling, Leslie Pack and Lozano-Perez, Tomas},
  booktitle={Robotics: Science and systems},
  volume={2},
  year={2010}
}

@article{task_planner_hoffmann2001ff,
  title={FF: The fast-forward planning system},
  author={Hoffmann, J{\"o}rg},
  journal={AI magazine},
  volume={22},
  number={3},
  pages={57--57},
  year={2001}
}

@inproceedings{task_planner_garrett2020pddlstream,
  title={Pddlstream: Integrating symbolic planners and blackbox samplers via optimistic adaptive planning},
  author={Garrett, Caelan Reed and Lozano-P{\'e}rez, Tom{\'a}s and Kaelbling, Leslie Pack},
  booktitle={Proceedings of the international conference on automated planning and scheduling},
  volume={30},
  pages={440--448},
  year={2020}
}

@article{task_planner_hansen2001lao,
  title={LAO: A heuristic search algorithm that finds solutions with loops},
  author={Hansen, Eric A and Zilberstein, Shlomo},
  journal={Artificial Intelligence},
  volume={129},
  number={1-2},
  pages={35--62},
  year={2001},
  publisher={Elsevier}
}

@article{motion_planning_Schulman2013FindingLO,
  title={Finding Locally Optimal, Collision-Free Trajectories with Sequential Convex Optimization},
  author={John Schulman and Jonathan Ho and Alex X. Lee and Ibrahim Awwal and Henry Bradlow and P. Abbeel},
  journal={Robotics: Science and Systems IX},
  year={2013},
  url={https://api.semanticscholar.org/CorpusID:2393365}
}

@article{pomdp_lauri2022partially,
  title={Partially observable markov decision processes in robotics: A survey},
  author={Lauri, Mikko and Hsu, David and Pajarinen, Joni},
  journal={IEEE Transactions on Robotics},
  volume={39},
  number={1},
  pages={21--40},
  year={2022},
  publisher={IEEE}
}

@article{pomdp_kurniawati2022partially,
  title={Partially observable markov decision processes and robotics},
  author={Kurniawati, Hanna},
  journal={Annual Review of Control, Robotics, and Autonomous Systems},
  volume={5},
  number={1},
  pages={253--277},
  year={2022},
  publisher={Annual Reviews}
}

@inproceedings{mpc_0_phiquepal2021control,
  title={Control-Tree Optimization: an approach to MPC under discrete Partial Observability},
  author={Phiquepal, Camille and Toussaint, Marc},
  booktitle={2021 IEEE International Conference on Robotics and Automation (ICRA)},
  pages={9666--9672},
  year={2021},
  organization={IEEE}
}

@inproceedings{tamp_0_phiquepal2019combined,
  title={Combined task and motion planning under partial observability: An optimization-based approach},
  author={Phiquepal, Camille and Toussaint, Marc},
  booktitle={2019 International Conference on Robotics and Automation (ICRA)},
  pages={9000--9006},
  year={2019},
  organization={IEEE}
}

@incollection{komo_17-toussaint-Newton,
  author = {Toussaint, Marc},
  title = {A tutorial on {N}ewton methods for constrained trajectory
  		  optimization and relations to {SLAM}, {G}aussian {P}rocess
  		  smoothing, optimal control, and probabilistic inference},
  booktitle = {Geometric and Numerical Foundations of Movements},
  editor = {Laumond, Jean-Paul},
  publisher = {Springer},
  year = {2017}
}

@misc{komo_solver,
    title={\href{https://arxiv.org/abs/1407.0414}{Newton methods for k-order Markov Constrained Motion Problems}},
    author={Marc Toussaint},
    year={2014},
    eprint={1407.0414},
    archivePrefix={arXiv},
    primaryClass={cs.RO}
}

@book{book_underactuated,
title = "Underactuated Robotics",
subtitle = "Algorithms for Walking, Running, Swimming, Flying, and Manipulation",
howpublished = "Course Notes for MIT 6.832",
author = "Tedrake, Russ",
year = 2023,
url = "https://underactuated.csail.mit.edu",
}

@Inbook{book_Diehl2013,
author="Diehl, Moritz",
editor="Baillieul, John
and Samad, Tariq",
title="Optimization Algorithms for Model Predictive Control",
bookTitle="Encyclopedia of Systems and Control",
year="2013",
publisher="Springer London",
address="London",
pages="1--11",
isbn="978-1-4471-5102-9",
doi="10.1007/978-1-4471-5102-9_9-1",
url="https://doi.org/10.1007/978-1-4471-5102-9_9-1"
}

@MISC{qmdp,
    author = {Michael L. Littman and Anthony R. Cassandra and Leslie Pack Kaelbling},
    title = {Learning policies for partially observable environments: Scaling up},
    year = {1995}
}

@misc{aula_1,
      title={A Novel Augmented Lagrangian Approach for Inequalities and Convergent Any-Time Non-Central Updates}, 
      author={Marc Toussaint},
      year={2014},
      eprint={1412.4329},
      archivePrefix={arXiv},
      primaryClass={math.OC}
}

@article{ADMM,
author = {Boyd, Stephen and Parikh, Neal and Chu, Eric and Peleato, Borja and Eckstein, Jonathan},
title = {Distributed Optimization and Statistical Learning via the Alternating Direction Method of Multipliers},
year = {2011},
issue_date = {January 2011},
publisher = {Now Publishers Inc.},
address = {Hanover, MA, USA},
volume = {3},
number = {1},
issn = {1935-8237},
url = {https://doi.org/10.1561/2200000016},
doi = {10.1561/2200000016},
journal = {Found. Trends Mach. Learn.},
month = jan,
pages = {1–122},
numpages = {122}
}

@article{pomcp_silver2010monte,
  title={Monte-Carlo planning in large POMDPs},
  author={Silver, David and Veness, Joel},
  journal={Advances in neural information processing systems},
  volume={23},
  year={2010}
}

@article{rmax_brafman2002r,
  title={R-max-a general polynomial time algorithm for near-optimal reinforcement learning},
  author={Brafman, Ronen I and Tennenholtz, Moshe},
  journal={Journal of Machine Learning Research},
  volume={3},
  number={Oct},
  pages={213--231},
  year={2002}
}

@book{boyd2004convex,
  title={Convex Optimization},
  author={Boyd, Stephen and Vandenberghe, Lieven},
  year={2004},
  publisher={Cambridge University Press}
}

@book{nocedal2006numerical,
  title={Numerical Optimization},
  author={Nocedal, Jorge and Wright, Stephen},
  year={2006},
  publisher={Springer Science \& Business Media}
}

@article{agha2014firm,
  title={FIRM: Sampling-based feedback motion-planning under motion uncertainty and imperfect measurements},
  author={Agha-Mohammadi, Ali-Akbar and Chakravorty, Suman and Amato, Nancy M},
  journal={The International Journal of Robotics Research},
  volume={33},
  number={2},
  pages={268--304},
  year={2014},
  publisher={SAGE Publications Sage UK: London, England}
}

@article{lee2020magic,
  title={MAGIC: Learning macro-actions for online POMDP planning},
  author={Lee, Yiyuan and Cai, Panpan and Hsu, David},
  journal={arXiv preprint arXiv:2011.03813},
  year={2020}
}

@inproceedings{agha2014health,
  title={Health aware stochastic planning for persistent package delivery missions using quadrotors},
  author={Agha-mohammadi, Ali-akbar and Ure, N Kemal and How, Jonathan P and Vian, John},
  booktitle={2014 IEEE/RSJ International Conference on Intelligent Robots and Systems},
  pages={3389--3396},
  year={2014},
  organization={IEEE}
}

@inproceedings{phan2020covernet,
  title={Covernet: Multimodal behavior prediction using trajectory sets},
  author={Phan-Minh, Tung and Grigore, Elena Corina and Boulton, Freddy A and Beijbom, Oscar and Wolff, Eric M},
  booktitle={Proceedings of the IEEE/CVF conference on computer vision and pattern recognition},
  pages={14074--14083},
  year={2020}
}

@article{kurniawati2011motion,
  title={Motion planning under uncertainty for robotic tasks with long time horizons},
  author={Kurniawati, Hanna and Du, Yanzhu and Hsu, David and Lee, Wee Sun},
  journal={The International Journal of Robotics Research},
  volume={30},
  number={3},
  pages={308--323},
  year={2011},
  publisher={SAGE Publications Sage UK: London, England}
}

@article{liang2024scaling,
  title={Scaling Long-Horizon Online POMDP Planning via Rapid State Space Sampling},
  author={Liang, Yuanchu and Kim, Edward and Thomason, Wil and Kingston, Zachary and Kurniawati, Hanna and Kavraki, Lydia E},
  journal={arXiv preprint arXiv:2411.07032},
  year={2024}
}

@article{kim2023reference,
  title={Reference-based POMDPs},
  author={Kim, Edward and Karunanayake, Yohan and Kurniawati, Hanna},
  journal={Advances in Neural Information Processing Systems},
  volume={36},
  pages={40659--40675},
  year={2023}
}

@article{somani2013despot,
  title={DESPOT: Online POMDP planning with regularization},
  author={Somani, Adhiraj and Ye, Nan and Hsu, David and Lee, Wee Sun},
  journal={Advances in neural information processing systems},
  volume={26},
  year={2013}
}

@article{zyner2019naturalistic,
  title={Naturalistic driver intention and path prediction using recurrent neural networks},
  author={Zyner, Alex and Worrall, Stewart and Nebot, Eduardo},
  journal={IEEE transactions on intelligent transportation systems},
  volume={21},
  number={4},
  pages={1584--1594},
  year={2019},
  publisher={IEEE}
}

@article{liang2026think,
  title={Think fast and far: Long-horizon online POMDP planning via rapid state sampling},
  author={Liang, Yuanchu and Kim, Edward and Knoll, J Arden and Thomason, Wil and Kingston, Zachary and Kavraki, Lydia E and Kurniawati, Hanna},
  journal={The International Journal of Robotics Research},
  pages={02783649261453546},
  year={2026},
  publisher={SAGE Publications Sage UK: London, England}
}

@inproceedings{pineau2003point,
  title={Point-based value iteration: An anytime algorithm for POMDPs},
  author={Pineau, Joelle and Gordon, Geoff and Thrun, Sebastian and others},
  booktitle={Ijcai},
  volume={3},
  pages={1025--1032},
  year={2003}
}

@inproceedings{kurniawati2008sarsop,
  title={Sarsop: Efficient point-based pomdp planning by approximating optimally reachable belief spaces.},
  author={Kurniawati, Hanna and Hsu, David and Lee, Wee Sun and others},
  booktitle={Robotics: Science and systems},
  volume={2008},
  year={2008},
  organization={Zurich, Switzerland}
}

\appendix
\onecolumn
\section{Background on the Augmented Lagrangian and ADMM Methods} \label{sec:background_constrained_opt}
\revise{This appendix provides background on the Augmented Lagrangian and ADMM methods, focusing on the aspects most relevant to this work.}

\revise{The first section presents the Augmented Lagrangian method for problems with inequality constraints, using the ``centered'' update rule described in~\citep{aula_1}. This treatment of inequality constraints corresponds to the ``unconstrained formulation'' in Section~17.4 of~\citep{nocedal2006numerical}.}

\revise{The second section introduces the consensus optimization variant of ADMM, as described in Chapter~7 of~\citep{ADMM}.}

\revise{ 
\subsection{Augmented Lagrangian Method (AuLa)} \label{sec:background_aula}
The Augmented Lagrangian method addresses constrained optimization problems of the standard form:
\begin{subequations}
\begin{align*}  
  \min_{z} \quad & c(z),\\
  \text{s.t.} \quad & g(z) \leq 0,\\
  				    & h(z) = 0,
\end{align*}
\end{subequations}
where $z \in \mathbb{R}^n$ is the optimization variable, $c: \mathbb{R}^n \rightarrow \mathbb{R}$ is an objective function, $g: \mathbb{R}^n \rightarrow \mathbb{R}^{d_g}$ is an inequality constraint function, and $h: \mathbb{R}^n \rightarrow \mathbb{R}^{d_h}$ is an equality constraint. The symbols $d_g$ and $d_h$ denote the dimensionality of the inequality and equality constraints respectively.
We assume the functions $c$, $g$ and $h$ to be differentiable but not necessarily convex or unimodal.}

\revise{In the context of trajectory optimization, the decision variable \( z \) typically represents a vector of controls, states, configurations, or a combination thereof, depending on how the problem is transcribed into a finite-dimensional optimization. Inequality constraints can represent, for example, collision avoidance, encoding the requirement that the distance to obstacles remains above a given threshold. Equality constraints can be used to enforce non-holonomic kinematics.
}

\medskip
\noindent \revise{\textbf{Unconstrained objective}: The method consists of optimizing a sequence of unconstrained objectives, known as Augmented Lagrangians, and defined as,
\begin{align*} 
  L_{\mu, \nu}(z, \lambda, \kappa) =
  c(z) + \lambda \ . \ g(z) + \frac{\mu}{2} \norm{[g(z) > 0]\odot g(z)}^2 \nonumber + \kappa \ . \ h(z) + \frac{\nu}{2} \norm{h(z)}^2, \nonumber\\
\end{align*}
where $\lambda \in \mathbb{R}^{d_g}$, and $\kappa \in \mathbb{R}^{d_h}$ are called Lagrange multipliers, and where $[g(z) > 0] \in \{0, 1\}^{d_g}$ is a binary vector indicating where the inequality constraint is violated, and the operator $\odot$ denotes the element-wise multiplication. The variables $\mu$ and $\nu$ are positive constants that weight the square penalty terms penalizing constraint violation. These penalty terms are commonly referred to as the ``augmentations''. They are introduced to improve the numerical stability and convergence behavior. The unconstrained objective function $L_{0, 0}$  obtained by omitting the augmentation terms (i.e., $\mu=\nu=0$), corresponds to the standard Lagrangian.}

\medskip
\noindent \revise{\textbf{Algorithm}: The Augmented Lagrangian algorithm proceeds by iterating the following steps:
\begin{subequations} 
\begin{align} 
  z^{k+1} &= \min_{z} L_{\mu, \nu} (z, \lambda^{k}, \kappa^{k}),\label{eq:back:L_minimization}\\
  \lambda^{k+1} &= \max(0, \lambda^k + \mu g(z^{k+1})),\label{eq:back:lambda_update}\\
  \kappa^{k+1} &= \kappa^k + \nu h(z^{k+1}). \label{eq:back:kappa_update}
\end{align}
\end{subequations} 
The first line~\eqref{eq:back:L_minimization} corresponds to the unconstrained minimization of the Augmented Lagrangian over $z$ with fixed $(\lambda^k, \kappa^k)$. Lines ~\eqref{eq:back:lambda_update} and ~\eqref{eq:back:kappa_update} adjust the Lagrange multipliers in response to constraint violations. These updates play a central role in the algorithm, as they penalize violations and guide the optimization process toward contraint satisfaction in subsequent iterations.}

\medskip
\noindent \revise{\textbf{Stopping criterion}: Iterations of the algorithm are performed until the following conditions are satisfied:
\begin{subequations} 
\begin{alignat}{3}
 &\norm{[g(z^k) > 0] \odot g(z^k)} && \leq \epsilon^{pri}, \label{eq:back:feasability_g}\\
 &\norm{h(z^k)} \ \ \ &&\leq \epsilon^{pri}, \label{eq:back:feasability_h}\\
 &\norm{z^{k} - z^{k-1}}\ \ \ && \leq \epsilon^{opt}, \label{eq:back:stability_z}
\end{alignat}
\end{subequations} 
where $\epsilon^{pri}$ and $\epsilon^{opt}$ are threshold values. Conditions~\eqref{eq:back:feasability_g} and ~\eqref{eq:back:feasability_h} ensure that the constraints violation are below a threshold. Line~\eqref{eq:back:stability_z} checks whether the solution has stabilized, indicating convergence.}

\medskip
\noindent \revise{\textbf{Convergence properties}:
Provided that the constraint functions satisfy certain regularity conditions---commonly referred to as constraint qualifications---such as Slater’s condition (in the convex case), Linear Independence Constraint Qualification (LICQ), or Mangasarian-Fromovitz Constraint Qualification (MFCQ) \citep{nocedal2006numerical}, the algorithm is guaranteed to converge to a solution. This corresponds to the global minimum for convex problems, while only local optimality can be ensured in the non-convex case. Specifically, it converges to a stationary point $(z^{\star}, \lambda^{\star}, \kappa^{\star})$ of the Lagrangian, at which the Karush-Kuhn-Tucker (KKT) conditions are satisfied:
\begin{subequations} 
\begin{align} 
  \nabla_{z} L_{0, 0}(z^{\star}, \lambda^{\star}, \kappa^{\star}) &= 0, && \text{(stationarity)} \label{eq:back:stationarity}\\
  g(z^{\star}) &\leq 0, && \text{(primal feasability)}\label{eq:back:primal_feasibility_ineq}\\
  h(z^{\star}) &= 0, && \text{(primal feasability)} \label{eq:back:primal_feasibility_eq}\\
  \lambda^{\star} &\geq 0, &&\text{(dual feasability)}\label{eq:back:dual_feasibility}\\
  \lambda^{\star}\ .\ g(z^{\star}) &= 0. &&\text{(complementary slackness)} \label{eq:back:complementary_slackness_ineq}
\end{align}
\end{subequations} 
The first line~\eqref{eq:back:stationarity} expresses that $(z^{\star}, \lambda^{\star}, \kappa^{\star})$ is a stationary point of the Lagrangian. \eqref{eq:back:primal_feasibility_ineq}, \eqref{eq:back:primal_feasibility_eq} and \eqref{eq:back:dual_feasibility} are feasibility conditions. The lines \eqref{eq:back:complementary_slackness_ineq} state that the Langrange multiplier $\lambda$ is zero where the constraints is inactive.}

\revise{The stationarity condition indicates that the sum of the gradients $\nabla_z c(z^{\star}) + \lambda^{\star} \nabla_z g(z^{\star}) + \kappa^{\star} \nabla_z h(z^{\star})$ cancel out. A geometric perspective provides a perhaps more intuitive way to understand the optimality conditions. By interpreting the functions $c$, $g$ and $h$ as potential fields creating ``forces'', as shown in Fig.~\ref{fig:back:kkt}, one can view the optimal solution as the point where those forces reach an equilibrium.}

\begin{figure}[!htb]
 \center{\includegraphics[width=0.6\textwidth]{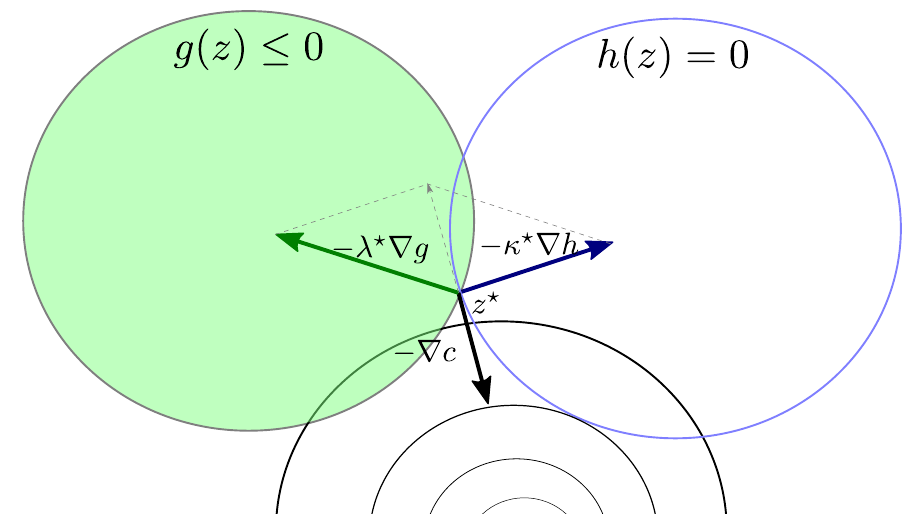}}
 \caption{\revise{Geometrical View of the KKT conditions: The constraint gradients weighted by the Lagrange multipliers cancel out the gradient of the cost function $c$. The green area indicates where the inequality constraints $g$ is satisfied, while the equality constraints $h$ is satisfied along the blue contour.}}
 \label{fig:back:kkt}
\end{figure}

\revise{\subsection{Consensus Form of the  Alternating Direction Method of Multipliers (ADMM)} \label{sec:background_admm}
The Consensus ADMM algorithm solves optimization problems of the form:
\begin{align*}
  \min_{\tilde{z}, z_i} \quad &{\sum_i{c_i(z_i)}},\\
 \text{s.t.} \quad & P z_i = \tilde{z}, 
\end{align*}  
where the variables $z_i \in \mathbb{R}^n$, for $i \in [1, N]$, along with the variable $\tilde{z} \in \mathbb{R}^m$ are the optimization variables. The matrix $P \in \{0, 1\}^{m\times n}$ is a selection matrix which extracts the part of the optimization variables $z_i$ subject to the consensus equality constraint. We have $m \leq n$ indicating that consensus is only required on a subset of the optimization variables. In addition, each row of $P$ contains exactly one entry equal to one.}

\revise{This type of formulation arises from the decomposition of an global optimization problem of the form $\min_z \sum_i c_i(z)$, where the functions $c_i$ operate on distinct but overlapping subsets of the global optimization variable $z$. In this paper, this structure results from the trajectory-tree formulation, as described in Section~\ref{sec:mpc_problem_formulation}, with distinct branches diverging from a common trunk. Splitting the optimization variable into distinct component $z_i$ and adding the equality contraint may initially seem counterintuitive. However, it presents the key advantage of enabling the optimization to be decomposed into smaller subproblems. It is particularly advantageous if the core subpart shared between all optimization variables is small, i.e. $m \ll n$.}

\medskip
\noindent \revise{\textbf{Unconstrained objective}: Similarly to the Augmented Lagrangian method, the ADMM method performs several interations of unconstrained minimizations based on the following objective:
\begin{align*} 
  L_{\rho}(z_1, \dots, z_i, \dots, z_N, \tilde{z}, \eta_1, \dots, \eta_i, \dots, \eta_N) =
\sum_i{c_i(z_i) + \eta_i \ . \ (Pz_i - \tilde{z}) + \frac{\rho}{2} \norm{Pz_i - \tilde{z}}^2}, \nonumber
\end{align*}
where $\eta_i \in \mathbb{R}^n$ are the Lagrange multipliers, and $\rho$ is a positive constant. This optimization objective arises directly from the application of the Augmented Lagrangian method to the consensus ADMM problem. It can be expressed as the sum of $N$ components $L_{\rho}(z_1, \dots , z_i, \dots, z_N, \tilde{z}, \eta_1, \dots, \eta_i, \dots, \eta_N) = \sum_i{L_{i, \rho}(z_i, \tilde{z}, \eta_i)}$, where each component is defined as:
\begin{align*}
 L_{i, \rho}(z_i, \tilde{z}, \eta_i) = c_i(z_i) + \eta_i \ . \ (Pz_i - \tilde{z}) + \frac{\rho}{2} \norm{Pz_i - \tilde{z}}^2.
\end{align*}
\medskip
\noindent \textbf{Algorithm}: The Consensus ADMM algorithm initializes the Lagrange multipliers to zero, i.e., $\eta_i = 0$ for $i \in [1, N]$, and then iterates the following steps:
\begin{subequations} 
\begin{align} 
  z_i^{k+1} &:= \min_{z_i} L_{i, \rho} (z_i, \tilde{z}^k , \eta_i^{k}),\label{eq:back:admm_unconstrained}\\
  \tilde{z}^{k+1} &:= \frac{1}{N}\sum_{i}{Pz^{k+1}_i}, \label{eq:back:admm_update_consensus}\\
  \eta_i^{k+1} &:= \eta_i^k + \rho (Pz_i^{k+1} - \tilde{z}^{k+1}) \label{eq:back:admm_update_dual}.
\end{align}
\end{subequations} 
The line~\eqref{eq:back:admm_unconstrained} corresponds to $N$ unconstrained minimization steps, one for each $i \in [1, N]$. These minimizations can be performed in parallel. In contrast, line~\eqref{eq:back:admm_update_consensus} is a centralized step which updates the consensus variable $\tilde{z}$ by averaging the results of the unconstrained minimizations. Finally, line~\eqref{eq:back:admm_update_dual} updates the Lagrange multipiers.}

\medskip
\noindent \revise{\textbf{Stopping criterion}: Iterations are performed until the following conditions are met:
\begin{subequations} \label{eq:termination-criterion}
\begin{alignat}{2}
 &\norm{Pz^k_{i} - \tilde{z}^k} && \leq \xi^{pri}, \label{eq:back:admm-primal} \\
 &\norm{\tilde{z}^{k} - \tilde{z}^{k-1}}\ \ \ && \leq \xi^{dual}, \label{eq:back:admm-admm-dual}
\end{alignat}
\end{subequations}
where $\xi^{pri}$ and $\xi^{dual}$ are threshold values. Condition~\eqref{eq:back:admm-primal} ensures that a consensus has been reached, while line~\eqref{eq:back:admm-admm-dual} verifies that the optimization of the consensus variable has stabilized.}

\medskip
\noindent \revise{\textbf{Convergence Properties}:
Similar to the Augmented Lagrangian method, convergence to the global minimum is guaranteed if the functions $c_i$ are convex. Otherwise, convergence may only reach a local minimum which satisfies the KKT conditions with respect to the  Lagrangian $L_{0}$.}

\section{Convergence and Optimality of the Distritbuted Augmented Lagrangian Algorithm (D-AuLa)}
\label{sec:proof}
This section provides a proof of convergence and optimality of D-AuLa under the following assumptions:
\begin{itemize}
\item Presence of equality constraints, but no inequality contraints. Inequality constraints are discussed in Section~\ref{sec:inequality_constraints}.
\item Convexity of the cost function and affine equality constraint.  Non-convex problems are discussed in Section~\ref{sec:nonconvex}.
\item Full overlap of the optimization problems. This implies that the sub-optimization variables have the same dimension. The proof can be extended to the case without full overlap as discussed in Section~\ref{sec:no_full_overlap}.
\end{itemize}

We formally define the optimization problem, under those assumptions. Let $z_i \in \mathbb{R}^{n}, i \in [1..N]$, $f_i:\mathbb{R}^{n} \rightarrow \mathbb{R}, i \in [1..N]$ be $N$ differentiable and convex cost functions, and $h_i:\mathbb{R}^{n} \rightarrow \mathbb{R}^m, , i \in [1..N]$ be $N$ affine vector-valued functions. We call the triple $(c_i, h_i, z_i)$ a subproblem.
\begin{definition} 
We define the \textit{Distributed Augmented Lagrangian} (D-AuLa) method as:

\begin{enumerate}[(1)]
\item Optimization objective:
\begin{subequations} \label{eq:distributed_aula_objective}
\begin{align}   
  \min_{\tilde{z}, z_i} \quad &{
  \sum_i{c_i(z_i)} \label{eq:consensus_admm_cost}
  },\\
 \text{s.t.} \quad
 & h_i(z_i) = 0, \label{eq:consensus_h_constraint}\\
 & z_i = \tilde{z}, \label{eq:consensus_z_admm_constraint}
\end{align}
\end{subequations}

where $\tilde{z} \in \mathbb{R}^n$ is an auxilary variable use to define an equality constraints stating that all $z_i$ should reach a consensus.
\item Augmented Lagrangian of a subproblem: 
\begin{align} 
  L_{i, \nu, \rho}(\tilde{z}, z_i, \kappa_i, \eta_i) =\ &c_i(z_i) \label{eq:dec_lagrangian} \\
    &+ \kappa_i \ . \ h_i(z_i) + \frac{\nu}{2} \norm{h_i(z_i)}^2 \nonumber\\
    &+ \eta_i \ . \ (z_i - \tilde{z}) + \frac{\rho}{2} \norm{z_i - \tilde{z}}^2, \nonumber
\end{align}
where $\kappa_i \in \mathbb{R}^{m_i}$, $\eta_i \in \mathbb{R}^n$ are the dual variables corresponding to the equality constraints~\eqref{eq:consensus_h_constraint} and consensus constraint ~\eqref{eq:consensus_z_admm_constraint} respectively. The constants $\nu$ and $\rho$ are positive real constants.

\item Algorithm:
\begin{subequations} \label{eq:dec_algo}
\begin{align}  
{z^{k+1}_{i}} &:= \min_{z_i^k} L_{i, \nu, \rho}(\tilde{z}^k, z_i^k, \kappa_i^k, \eta_i^k) \label{eq:dec-algo-newton}\\
\kappa_i^{k+1} &:= \kappa_i^{k} + \nu h_i(z_i^{k+1}) \label{eq:dec-algo-dual-eq} \\
\tilde{z}^{k+1} &:= \frac{1}{N}\sum_i{z_i^{k+1}} \label{eq:dec-z_update} \\
\eta_i^{k+1} &:= \eta_i^{k} + \rho (z_i^{k+1} - \tilde{z}^{k+1}) \label{eq:dec-dual-admm}
\end{align}
\end{subequations}

\end{enumerate}

\end{definition} 

At each iteration, the steps~\eqref{eq:dec-algo-newton}, \eqref{eq:dec-algo-dual-eq} and \eqref{eq:dec-dual-admm} are applied for each of subproblem (for each $i \in [1, N]$). These sub-steps are parallelizable since they are fully independent from each other. The step~\eqref{eq:dec-z_update} is a centralized step averaging the results of the unconstrained minimizations~\eqref{eq:dec-algo-newton}. Iterations are performed until convergence.

\subsection{Proof Overview}
The proof proceeds in two steps:
\begin{itemize}
\item We first define an alternative problem formulatiom \textit{Constrained ADMM} (C-ADMM), and show its convergence and optimality in Section~\ref{sec:constrained_admm_proof}. This algorithm is closer to the ADMM standard form. It decomposes an optimization problem in only two subproblems, and performs the unconstrained minimization steps in sequence. It differs from the standard form, since one of the subproblems has an additional equality constrains. This builds upon the proof provided in \citep{ADMM} for the standard form, and extends it to cover the additional equality constraint. 
\item Second, we show that the D-AuLa algorithm can be rewritten as a special case of the C-ADMM.
\end{itemize}
Taken together, those two steps prove convergence and optimality of D-AuLa. 
\newline

For readers seeking a high-level understanding of the connection between D-AuLa and the standard ADMM framework, we recommend  Section~\ref{sec:equivalence}. Those interested in the main arguments underlying the C-ADMM convergence proof can find sufficient details in Section~\ref{sec:proof_assuming_inequalities}. The Sections~\ref{sec:inequalities_proof_1}, ~\ref{sec:inequalities_proof_2} and ~\ref{sec:inequalities_proof_3} are the detailed steps and algebraic manipulations for readers interested in the proof in its full depth.

\subsection{Constrained ADMM Algorithm (C-ADMM)} \label{constrained_admm}
Let $x \in \mathbb{R}^n$, $y \in \mathbb{R}^n$, $f:\mathbb{R}^{n} \rightarrow \mathbb{R}$ be a differentiable and convex cost function, $g:\mathbb{R}^{n} \rightarrow \mathbb{R}$ be a function proper, closed and convex, and $h:\mathbb{R}^{n} \rightarrow \mathbb{R}^m$ be an affine vector-valued function.

\begin{definition}
We define the \textit{Constrained ADMM} (C-ADMM) method as:
\begin{enumerate}[(1)]
\item Optimization objective:
\begin{subequations}
\begin{align}  
  \min_{x, y} \quad & f(x) + g(y),\\
 \text{s.t.}  \quad & h(x) = 0, \label{eq:constrained_admm_eq_constraint} \\
                    & x = y. \label{eq:constrained_admm_consensus_constraint}
\end{align}
\end{subequations}
\item Augmented Lagrangian:
\begin{align}  
  L_{\nu, \rho}(x, y, \kappa, \eta) &= 
  f(x) + g(y)\\
    &+ \kappa \ . \ h(x) + \frac{\nu}{2} \norm{h(x)}^2 \nonumber\\
    &+ \eta \ . \ (x - y) + \frac{\rho}{2} \norm{x - y}^2 \nonumber,
\end{align}
where $\kappa \in \mathbb{R}^m$, $\eta \in \mathbb{R}^n$ are the dual variables corresponding to the equality constraint~\eqref{eq:constrained_admm_eq_constraint} and consensus constraint~\eqref{eq:constrained_admm_consensus_constraint} respectively. The constants $\nu$ and $\rho$ are positive real constants.
\item Algorithm:
\begin{subequations}
\begin{align}  
  x^{k+1} &:= \min_{x} L_{\nu, \rho}(x, y^{k}, \kappa^k, \eta^k)\\
  y^{k+1} &:= \min_{y} L_{\nu, \rho}(x^{k+1}, y, \kappa^k, \eta^k)\\
  \kappa^{k+1} &:= \kappa^{k} + \nu h(x^{k+1}) \label{eq:constrained_admm_y_min}\\
  \eta^{k+1} &:= \eta^{k} + \rho(x^{k+1} - y^{k+1})
\end{align}
\end{subequations}
\end{enumerate}
\end{definition}

\noindent We note that the optimization objective is not symmetric:
\begin{itemize}
\item The equality constraint applies only on $x$.
\item The function $f$ needs to be differentiable, whereas $g$ need not be.
\end{itemize} 
The reason for this dissymmetry will become clearer in Section~\ref{sec:equivalence}: $f$ and $h$ will be used to aggregate the cost and constraint functions over the different subproblems. The function $g$ will potentially not be continuous. It will be used as an indicator function taking values in $\{0, +\infty\}$ to enforce structure on $x$ and $y$.

\begin{theorem} \label{th:theorem_1}
There exists an optimal solution $(x^{\star},y^{\star})$ and corresponding Lagrange multipliers $(\kappa^{\star},\eta^{\star})$ such that $(x^{\star},y^{\star}, \kappa^{\star},\eta^{\star})$ is a saddle point of the Lagrangian $L_{0,0}$. Furthermore, the iterates generated by the C-ADMM algorithm satisfy:
\begin{align} 
 x^k \rightarrow x^{\star},\  y^k \rightarrow y^{\star},\  \kappa^k \rightarrow \kappa^{\star},\ \eta^k \rightarrow \eta^{\star},
\end{align}
ensuring convergence to the optimal solution.
\end{theorem}

\subsection{Proof of Theorem~\ref{th:theorem_1}} \label{sec:constrained_admm_proof}
The existence of a unique optimum $(x^{\star}, y^{\star})$, along with Lagrange multipliers $(\kappa^{\star}, \eta^{\star})$ forming a saddle point of the Lagrangian follows directly from the assumption of convex cost function and affine equality constraints~\citep[Chapter 5]{boyd2004convex}. The rest of the proof focuses on showing that the algorithm iterates converge towards this optimal solution.

To this end, we first introduce notations and three inequalities that will be used throughout the proof. We note $r^k$ the deviation between the two optimization variables:
\begin{subequations}
\begin{align}  
  r^k = x^k - y^k \nonumber
\end{align}
\end{subequations}

\noindent We note $p^k$ the sum of the two costs functions:
\begin{equation}
  p^k = f(x^k) + g(y^k) \nonumber
\end{equation}

\begin{definition} \label{def:vk} We define $V^k$ as a weighted sum of the squared distance from the optimum for the dual variables $\kappa$ and $\eta$ and the second primal variable $y$:
\begin{equation}  
  V^k = \frac{1}{\nu}\norm{\kappa^k-\kappa^{\star}}^2
+ \frac{1}{\rho}\norm{\eta^k-\eta^{\star}}^2
+ \rho \norm{y^k - y^{\star}}^2
\end{equation}
\end{definition}

The value of $V^k$ is not known during the algorithm execution since it depends on the unknown variables $x^{\star}, y^{\star}, \kappa^{\star}, \eta^{\star}$, but this value will be used to articulate the proof. This definition of $V^k$ extends the one provided in \citep{ADMM} with the additional term~$\frac{1}{\nu}\norm{\kappa^k-\kappa^{\star}}^2$.

\begin{proposition}
The three following inequalities hold throughout the algorithm execution:
\begin{enumerate}[(1)]
\item Optimum deviation lower bound:
\begin{align} \label{eq:a_3}
p^{\star} - p^{k+1} \leq \kappa^{\star} \ .\ h(x^{k+1}) + \eta^{\star} \ . \ r^{k+1}  
\end{align}

\item Optimum deviation upper bound:
\begin{align}  \label{eq:a_2}
p^{k+1} - p^{\star} \leq - \kappa^{k+1} \ .\ h(x^{k+1}) - \eta^{k+1} \ .\ r^{k+1} + \rho (y^{k+1} - y^{k}) \ .\ (-r^{k+1} - (y^{k+1} - y^{\star})) 
\end{align}

\item Value decay:
\begin{align}  \label{eq:a_1}
V^{k+1} - V^k\leq - \nu \norm{h(x^{k+1})}^2 - \rho \norm{r^{k+1}}^2 - \rho \norm{y^{k+1} - y^k}^2
\end{align}
\end{enumerate}
\end{proposition}
\noindent These inequalities are proven in Sections~\ref{sec:inequalities_proof_1}, ~\ref{sec:inequalities_proof_2} and ~\ref{sec:inequalities_proof_3}. They extend the inequalities refered to as (A.3), (A.2) and (A.1) in \citep{ADMM}.

\subsubsection{Proof of Theorem~\ref{th:theorem_1} Assuming the Validity of the Inequalities~\eqref{eq:a_3}, \eqref{eq:a_2} and \eqref{eq:a_1}} \label{sec:proof_assuming_inequalities}

\begin{proof}[Proof of Convergence]

Summing the inequality~\eqref{eq:a_1} for an infinite number of iterations gives the following inequality:
\begin{equation*}
V^{\infty} - V^0 \leq - \nu \sum_{k=0}^{\infty}{\norm{h(x^{k+1})}^2} - \rho \sum_{k=0}^{\infty}{\norm{r^{k+1}}^2 + \norm{y^{k+1} - y^{k}}^2}, \label{eq:residual_sum_1}
\end{equation*}
which can be rewritten as,
\begin{align}  
\nu \sum_{k=0}^{\infty}{\norm{h(x^{k+1})}^2} + \rho \sum_{k=0}^{\infty}{\norm{r^{k+1}}^2 + \norm{y^{k+1} - y^{k}}^2} \leq V^0 - V^{\infty}. \label{eq:residual_sum_2}
\end{align}

By definition, $V^k$ is always a positive, since it is the sum of the norm of vectors weighed by positive numbers. In addition, the Value decay inequality~\eqref{eq:a_1} implies that $V^k \leq V^0$. Therefore $V^0 - V^{\infty}$ is finite and positive.

The left side of inequality~\eqref{eq:residual_sum_2} is the sum of an infinity of positive terms. This sum is bounded by $V^0 - V^{\infty}$ which implies that the summed elements tend to zero, i.e. $h(x^{k+1}) \rightarrow 0 $, $r^k \rightarrow 0$ and $y^{k+1} - y^{k} \rightarrow 0$ when $k \rightarrow \infty$.
This respectively means that the equality constraint~\eqref{eq:constrained_admm_eq_constraint} is enforced, that the two local variables $x$ and $y$ reach a consensus~\eqref{eq:consensus_z_admm_constraint}, and that the algorithm becomes stationary. \hfill \qedsymbol
\end{proof}

\begin{proof}[Proof of Optimality]
Inequality~\eqref{eq:a_3} gives a lower-bound of $p^{k+1} - p^{\star}$. We showed in the convergence proof, that $h(x^{k+1}) \rightarrow 0$ and $r^{k+1} \rightarrow 0$. It therefore implies that the lower bound tends to zero. 

Inequality~\eqref{eq:a_2} gives an upper-bound of $p^{k+1} - p^{\star}$. We showed earlier that $h(x^{k+1}) \rightarrow 0$,  $r^{k+1} \rightarrow 0$ and $y^{k+1} - y^{k} \rightarrow 0$. In addition, the fact that $0 \leq V^k \leq V^0$ implies that $\kappa^{k+1}$, $\eta^{k+1}$, and $y^{k+1} - y^{\star}$ are bounded. The upper bound is therefore a sum of three terms which are each a product of a finite term and a quantity tending to zero. The upper-bound therefore tends to zero.

Since both the upper and lower bounds of $p^{k+1} - p^{\star}$ tend to zero, it implies that $p^k \rightarrow p^{\star}$ which is the objective convergence. \hfill \qedsymbol
\end{proof}

\subsubsection{Proof of Inequality~\eqref{eq:a_3}} \label{sec:inequalities_proof_1}

\begin{proof}
Since $(x^{\star}, y^{\star}, \kappa^{\star}, \eta^{\star})$ is a saddle point of the Lagrangian $L_{0, 0}$, we have:
\begin{align}  
L_{0, 0}(x^{\star}, y^{\star}, \kappa^{\star}, \eta^{\star}) &\leq L_{0, 0}(x^{k+1}, y^{k+1}, \kappa^{\star}, \eta^{\star}). \label{eq:need_saddle_point}
\end{align}
By definition, $h(x^{\star}) = 0$, $x^{\star} = y^{\star}$, and $\rho = \nu = 0$ such that the Lagrangian on the left side reduces to the cost terms $f(x^{\star}) + g(x^{\star})$. On the right side the Lagrangian is composed to the cost term $f(x^{k+1}) + g(x^{k+1})$ and the Lagrange terms for the equality and consensus constraints, such that it gives,
\begin{align*}  
f(x^{\star}) + g(x^{\star}) &\leq f(x^{k+1}) + g(x^{k+1}) + \kappa^{\star}\ .\ h(x^{k+1}) + \eta^{\star}\ .\ (x^{k+1} - y^{k+1}),
\end{align*}
which leads to inequality~\eqref{eq:a_3} when replacing $f(x) + g(y)$ by $p$ and $x-y$ by $r$ according to their definitions:
\begin{align*}  
p^{\star} - p^{k+1} &\leq \kappa^{\star} \ .\ h(z_0^{k+1}) + \eta^{\star} \ .\ r^{k+1}.   
\end{align*} \hfill \qedsymbol
\end{proof}

\subsubsection{Proof of Inequality ~\eqref{eq:a_2}} \label{sec:inequalities_proof_2}
\begin{proof} The proof proceeds in three steps. First, we derive the optimality conditions for $x^{k+1}$ and $y^{k+1}$, which follow from their definition as minimizers of the Augmented Lagrangian. Then, we combine these conditions to obtain inequality~\eqref{eq:a_2}.
\paragraph{Optimality with respect to the first optimization variable:}
By definition, $x^{k+1}$ minimizes $L_{\nu, \rho}(x, y^k, \kappa^k, \eta^{k})$. Since  $f$ and $h$ are differentiable, the Augmented Lagrangian is differentiable. A necessary optimality condition is that the gradient is zero in $x^{k+1}$, i.e.
\begin{subequations}
\begin{gather} 
\nabla_x L_{\nu, \rho} (x^{k+1}, y^k, \kappa^k, \eta^k) = 0, \nonumber \\ 
\nabla f (x^{k+1}) + \kappa^k \cdot \mathbb{J}_h (x^{k+1}) + \nu h(x^{k+1}) \cdot \mathbb{J}_h(x^{k+1}) + \eta^k + \rho (x^{k+1} - y^{k}) = 0. \nonumber
\end{gather}
\end{subequations}

We develop further the second term using $\kappa^{k} = \kappa^{k+1} - \nu h(x^{k+1})$, which has for effect to cancel the third term. Similarly, we use $\eta^{k} = \eta^{k+1} - \rho (x^{k+1} - y^{k+1})$ for the third term, which can then be combined with the fourth term and gives,
\begin{align} 
\nabla f(x^{k+1}) +
 \kappa^{k+1} \cdot \mathbb{J}_h (x^{k+1}) +
 \eta^{k+1} + \rho ( y^{k+1} - y^{k}) = 0, \label{eq:needs_convexity_0}
\end{align}
\noindent which implies, by integration, that $x^{k+1}$ minimizes,
\begin{align} 
f(x) + \kappa^{k+1} \cdot h(x) + (\eta^{k+1} + \rho (y^{k+1} - y^{k})) \cdot x. \label{eq:needs_convexity}
\end{align}

\noindent It follows that its value at $x^{k+1}$ is less than or equal to its value at $x^{\star}$ which yields the following inequality,
\begin{align} \label{eq:proof_A_2_first_term}
f(x^{k+1}) + \kappa^{k+1}h(x^{k+1}) + (\eta^{k+1} + \rho (y^{k+1} - y^{k})) \cdot x^{k+1} \leq  f(x^{\star}) + (\eta^{k+1} + \rho (y^{k+1} - y^{k})) \cdot y^{\star}.
\end{align}
Note: In the expression above, we use the fact that $h(x^{\star}) = 0$, by definition.

\paragraph{Optimality with respect to the second optimization variable:}
By definition, $y^{k+1}$ minimizes $L_{\nu, \rho}(x^{k+1}, y, \kappa^k, \eta_{k})$. Since $g$ is closed, proper and convex, the Augmented Lagrangian is subdifferentiable, and a necessary and sufficient optimality condition is therefore that 0 belongs to the subderivatives of the Augmented Langrangian in $y^{k+1}$:
\begin{align*} 
0 \in \partial_y L_{\nu, \rho}(x^{k+1}, y^{k+1}, \kappa^k, \eta^k).  \nonumber
\end{align*}
The derivative of the Lagrange term and the square penalty with respect to $y$ have an analytical expression, such that the subdifferential notation remains only on $g$ giving:
\begin{align*} 
0 \in \partial g(y^{k+1}) - \eta^k - \rho (x^{k+1} - y^{k+1}). \nonumber 
\end{align*}
We now use $\eta^{k} = \eta^{k+1} - \rho (x^{k+1} - y^{k+1})$ for the second term, which cancels the third term and gives:
\begin{align*} 
0 \in \frac{\partial}{\partial z_1} g(z_1^{k+1}) - \eta^{k+1}.
\end{align*}
This implies that $y^{k+1}$ minimizes:
\begin{align} 
g(y) - \eta^{k+1} \cdot y\ . \label{eq:g_intermediary}
\end{align}
Comparing the evaluations in $y^{k+1}$ and $y^{\star}$ leads to the following inequality:
\begin{align} \label{eq:proof_A_2_second_term}
g(y^{k+1}) - \eta^{k+1} \cdot \ y^{k+1} \leq g(y^{\star}) - \eta^{k+1} \cdot \ y^{\star}.
\end{align}

\paragraph{Summing the two inequalities:}
Summing up the inequalities~\eqref{eq:proof_A_2_first_term} and \eqref{eq:proof_A_2_second_term} gives:
\begin{subequations}
\begin{align}
f(x^{k+1}) + \kappa^{k+1} \cdot h(x^{k+1}) + (\eta^{k+1} + \rho (y^{k+1} - y^{k})) \cdot x^{k+1} \nonumber \\ +\ g(y^{k+1}) - \eta^{k+1} \cdot y^{k+1} \nonumber \\ \leq \nonumber \\ f(x^{\star}) + (\eta^{k+1} + \rho (y^{k+1} - y^{k})) \cdot x^{\star} \nonumber\\+\ g(y^{\star}) - \eta^{k+1}\ . \ y^{\star}. \nonumber
\end{align}
\end{subequations}
We separate the cost terms with $f$ and $g$ (moved to the left side) from the other terms which are moved to the right side. Using the notation with $p$, this gives:
\begin{align*}
p^{k+1} - p^{\star} \\
 \leq \\ - \kappa^{k+1} \cdot h(x^{k+1}) - (\eta^{k+1} + \rho (y^{k+1} - y^{k})) \cdot x^{k+1} \\ + (\eta^{k+1} + \rho (y^{k+1} - y^{k})) \cdot x^{\star} + \eta^{k+1} \cdot y^{k+1} - \eta^{k+1} \cdot y^{\star}.
\end{align*}
Since $x^{\star} = y^{\star}$, the terms with $\eta^{k+1} \cdot z_1^{k+1}$ cancel out yielding:
\begin{align*}
p^{k+1} - p^{\star}
 &\leq - \kappa^{k+1} \cdot h(x^{k+1}) - (\eta^{k+1} + \rho (y^{k+1} - y^{k})) \cdot x^{k+1} + \rho (y^{k+1} - y^{k}) \cdot y^{\star} + \eta^{k+1} \cdot y^{k+1}. \\
\intertext{Regrouping the terms around and using the notation $r^{k+1} = x^{k+1} - y^{k+1}$ leads to:}
p^{k+1} - p^{\star}
 &\leq - \kappa^{k+1} \cdot h(x^{k+1}) - \eta^{k+1} \cdot r^{k+1} - \rho (y^{k+1} - y^{k}) \cdot x^{k+1} + \rho (y^{k+1} - y^{k}) \cdot y^{\star}. \\
\intertext{We now factorize the last terms around $\rho (y^{k+1} - y^{k})$ to get:}
p^{k+1} - p^{\star}
 &\leq - \kappa^{k+1} \cdot h(x^{k+1}) - \eta^{k+1} \cdot r^{k+1} + \rho (y^{k+1} - y^{k}) \cdot (-x^{k+1} + y^{\star}). \\
\intertext{The last term $-x^{k+1} + y^{\star}$ can be trivially rewritten into $-x^{k+1} + y^{k+1} - y^{k+1} + y^{\star}$, which is $-r^k - (y^{k+1} - y^{\star})$, which results in,}
 p^{k+1} - p^{\star}
 &\leq - \kappa^{k+1} \cdot h(x^{k+1}) - \eta^{k+1}\ .\ r^{k+1} + \rho (y^{k+1} - y^{k}) . (-r^{k+1} - (y^{k+1} - y^{\star})),
\end{align*}
which is the inequality~\eqref{eq:a_2}. \hfill \qedsymbol
\end{proof}

\subsubsection{Proof of Inequality~\eqref{eq:a_1}} \label{sec:inequalities_proof_3}
\begin{proof}
Adding the inequalities~\eqref{eq:a_3} and \eqref{eq:a_2} gives:
\begin{align}
&0 \leq -\kappa^{k+1} \cdot h(x^{k+1}) - \eta^{k+1} \cdot r^{k+1} + \rho (y^{k+1} - y^{k}) \cdot (-r^{k+1} - (y^{k+1} - y^{\star})) + \kappa^{\star} \cdot h(x^{k+1}) + \eta^{\star} \cdot r^{k+1}. \nonumber \\
\intertext{Flipping the inequality and factorizing the terms involving $h(x^{k+1})$ leads to:}
&(\kappa^{k+1} - \kappa^{\star}) \cdot h(x^{k+1}) + \eta^{k+1} \cdot r^{k+1} - \rho (y^{k+1} - y^{k}) \cdot (-r^{k+1} - (y^{k+1} - y^{\star})) - \eta^{\star} \cdot r^{k+1} \leq 0. \nonumber \\
\intertext{We now regroup the terms with $\eta$ to get:}
&(\kappa^{k+1} - \kappa^{\star}) \cdot h(x^{k+1}) + (\eta^{k+1} - \eta^{\star}) \cdot r^{k+1} - \rho (y^{k+1} - y^{k}) \cdot (-r^{k+1} - (y^{k+1} - y^{\star})) \leq 0. \nonumber \\
\intertext{Multiplying by 2 yields:}
&2(\kappa^{k+1} - \kappa^{\star}) \cdot h(x^{k+1}) + 2(\eta^{k+1} - \eta^{\star}) \cdot r^{k+1} - 2\rho (y^{k+1} - y^{k}) \cdot (-r^{k+1} - (y^{k+1} - y^{\star})) \leq 0. \label{eq:a2_plus_a3_to_rewrite}
\end{align}

This inequality consists of three terms, each of which will be systematically reformulated in the subsequent sections. The main goal is to form the terms $\norm{\kappa^{k} - \kappa^{\star}}^2$, $\norm{\eta^{k} - \eta^{\star}}^2$, $\norm{y^{k} - \eta^{\star}}^2$ which are the compounds of $V^k$.

\paragraph{Rewriting the first term of inequality~\eqref{eq:a2_plus_a3_to_rewrite}:} \label{sec:rewrite_first_term}
The term to rewrite is:
\begin{align*}
2 (\kappa^{k+1} - \kappa^{\star}) \cdot h(x^{k+1}).
\end{align*}
We use the fact that $\kappa^{k+1} = \kappa^{k} + \nu h(x^{k+1})$ which gives:
\begin{align*}
2 (\kappa^{k} + \nu h(z^{k+1}) - \kappa^{\star}) \cdot h(x^{k+1}).
\end{align*}
Developing and splitting the term involving $h(x^{k+1})$ into two identical parts gives:
\begin{align*}
2 (\kappa^{k} - \kappa^{\star}) .\ h(x^{k+1}) + \nu \norm{h(x^{k+1})}^2 + \nu \norm{h(x^{k+1})}^2.
\end{align*}
Substituting $h(x^{k+1}) = \frac{1}{\nu}(\kappa^{k+1} - \kappa^{k})$ in the first two terms results in:
\begin{align*}
\frac{2}{\nu}(\kappa^{k} - \kappa^{\star}) \cdot (\kappa^{k+1} - \kappa^{k}) + \frac{1}{\nu} \norm{\kappa^{k+1} - \kappa^{k}}^2 + \nu \norm{h(x^{k+1})}^2.
\end{align*}
We rewrite further the first term using $\kappa^{k+1} - \kappa^{k} = (\kappa^{k+1} - \kappa^{\star}) + (\kappa^{\star} - \kappa^{k})$ to expand the first term:
\begin{align*}
\frac{2}{\nu}(\kappa^{k} - \kappa^{\star}) \cdot ((\kappa^{k+1} - \kappa^{\star}) + (\kappa^{\star} - \kappa^{k})) + \frac{1}{\nu} \norm{\kappa^{k+1} - \kappa^{k}}^2 + \nu \norm{h(x^{k+1})}^2.
\intertext{Developing the first term leads to:}
-\frac{2}{\nu}\norm{\kappa^{k} - \kappa^{\star}}^2 + \frac{2}{\nu} (\kappa^{k} - \kappa^{\star}) \cdot (\kappa^{k+1} - \kappa^{\star}) + \frac{1}{\nu} \norm{\kappa^{k+1} - \kappa^{k}}^2 + \nu \norm{h(x^{k+1})}^2.
\end{align*}
Rewriting the third term using $\kappa^{k+1} - \kappa^{k} = (\kappa^{k+1} - \kappa^{\star}) + (\kappa^{\star} - \kappa^{k})$ yields:
\begin{align*}
-\frac{2}{\nu}\norm{\kappa^{k} - \kappa^{\star}}^2 + \frac{2}{\nu} (\kappa^{k} - \kappa^{\star}) \cdot (\kappa^{k+1} - \kappa^{\star}) + \frac{1}{\nu} \norm{(\kappa^{k+1} - \kappa^{\star}) + (\kappa^{\star} - \kappa^{k})}^2 + \nu \norm{h(x^{k+1})}^2.
\end{align*}
Developing the norm gives:
\begin{align*}
-\frac{2}{\nu}\norm{\kappa^{k} - \kappa^{\star}}^2 + \frac{2}{\nu} (\kappa^{k} - \kappa^{\star}) \cdot & (\kappa^{k+1} - \kappa^{\star}) + \frac{1}{\nu} \norm{\kappa^{k+1} - \kappa^{\star}}^2\\& + \frac{1}{\nu}\norm{\kappa^{\star} - \kappa^{k}}^2 + \frac{2}{\nu}(\kappa^{k+1} - \kappa^{\star}) \cdot (\kappa^{\star} - \kappa^{k}) + \nu \norm{h(x^{k+1})}^2.
\end{align*}
The terms $\frac{2}{\nu} (\kappa^{k} - \kappa^{\star}) \cdot (\kappa^{k+1} - \kappa^{\star})$ cancel out and the terms with $\norm{\kappa^{k} - \kappa^{\star}}^2$ can be regrouped, resulting in:
\begin{align}
(1 / \nu) \left( \norm{\kappa^{k+1} - \kappa^{\star}}^2 - \norm{\kappa^{\star} - \kappa^{k}}^2 \right) + \nu \norm{h(x^{k+1})}^2.
\end{align}

\paragraph{Rewriting the second term:}
The term to rewrite is:
\begin{align*}
2 (\eta^{k+1} - \eta^{\star}) \ .\ r^{k+1}.
\end{align*}
This is the exact equivalent of the first term rewritten in Section~\ref{sec:rewrite_first_term} with $\eta^{k+1}$, $\eta^{\star}$ and $r^{k+1}$ instead of $\kappa^{k+1}$, $\kappa^{\star}$ and $h(x^{k+1})$. A similar argument, (using the fact that $\eta^{k+1} = \eta^{k} + \rho r^{k+1}$) leads to:
\begin{align}
(1 / \rho) \left( \norm{\eta^{k+1} - \eta^{\star}}^2 - \norm{\eta^{\star} - \eta^{k}}^2 \right) + \rho \norm{r^{k+1}}^2. \label{eq:rewritten_second_term}
\end{align} 

\paragraph{Rewriting the third term of inequality~\eqref{eq:a2_plus_a3_to_rewrite}:}
We now rewrite the third term~\eqref{eq:a2_plus_a3_to_rewrite} to which we add the term $\rho \norm{r^{k+1}}^2$ coming from the rewritten second term~\eqref{eq:rewritten_second_term}. The term to rewrite is therefore:
\begin{align*}
\rho \norm{r^{k+1}}^2 + 2 \rho (y^{k+1} - y^{k}) \cdot r^{k+1} + 2 \rho (y^{k+1} - y^{k})\ .\ (y^{k+1} - y^{\star}).
\end{align*}
We use $y^{k+1} - y^{\star} = (y^{k+1} - y^{k}) + (y^k - y^{\star})$ in the last term and develop it to obtain:
\begin{align*}
\rho \norm{r^{k+1}}^2 + 2 \rho (y^{k+1} - y^{k}) \cdot r^{k+1} + 2 \rho \norm{y^{k+1} - y^{k}}^2 + 2 \rho (y^{k+1} - y^{k}) \cdot ((y^k - y^{\star})).
\end{align*}
Observing that $\rho \norm{r^{k+1}}^2 + 2 \rho (y^{k+1} - y^{k}) \cdot r^{k+1} = \rho \norm{r^{k+1} + (y^{k+1} - y^k)}^2 - \rho \norm{y^{k+1} - y^{k}}^2$, we rewrite the first two terms and combine it with the third term to get:
\begin{align*}
\rho \norm{r^{k+1} + (y^{k+1} - y^k)}^2 + \rho \norm{y^{k+1} - y^{k}}^2 + 2 \rho (y^{k+1} - y^{k}) \cdot ((y^k - y^{\star})).
\end{align*}
We now rewrite the last two terms using $y^{k+1} - y^k = (y^{k+1} - y^{\star}) - (y^{k} - y^{\star})$,
\begin{align*}
\rho \norm{r^{k+1} + (y^{k+1} - y^k)}^2 + \rho \norm{(y^{k+1} - y^{\star}) - (y^{k} - y^{\star})}^2 + 2 \rho ((y^{k+1} - y^{\star}) - (y^{k} - y^{\star})) \cdot (y^k - y^{\star}),
\end{align*}
and develop further which yields:
\begin{align*}
\rho \norm{r^{k+1} + (y^{k+1} - y^k)}^2 + \rho \norm{y^{k+1} - y^{\star}}^2 + \rho \norm{ y^{k} - y^{\star}}^2 - 2\rho (y^{k+1} - y^{\star}) \cdot (y^{k} - y^{\star}) \\+ 2 \rho ((y^{k+1} - y^{\star}) - (y^{k} - y^{\star})) \cdot ((y^k - y^{\star})).
\end{align*}
The terms starting with $2\rho$ partially cancel out resulting in the following,
\begin{align*}
\rho \norm{r^{k+1} + (y^{k+1} - y^k)}^2 + \rho \norm{(y^{k+1} - y^{\star})}^2 + \rho \norm{ y^{k} - y^{\star}}^2 - 2 \rho \norm{y^k - y^{\star}}^2,
\end{align*}
and finally,
\begin{align}
\rho \norm{r^{k+1} + (y^{k+1} - y^k)}^2 + \rho \left( \norm{(y^{k+1} - y^{\star})}^2 -  \norm{y^k - y^{\star}}^2 \right).
\end{align}

\paragraph{Bringing the rewritten terms together:}
Adding the first, second and third terms and bringing back the inequality gives:
\begin{align*}
\frac{1}{\nu} \left( \norm{\kappa^{k+1} - \kappa^{\star}}^2 - \norm{\kappa^{\star} - \kappa^{k}}^2 \right) & + \nu \norm{h(x^{k+1})}^2 \\ &+\frac{1}{\rho} \left( \norm{\eta^{k+1} - \eta^{\star}}^2 - \norm{\eta^{\star} - \eta^{k}}^2 \right) \\ &+ \rho \norm{r^{k+1} + (y^{k+1} - y^k)}^2 + \rho \left( \norm{(y^{k+1} - y^{\star})}^2 -  \norm{y^k - y^{\star}}^2 \right) \leq 0.
\end{align*}
The terms in large parenthese are exactly the compounds of $V^{k+1} - V^k$ such that one can rewrite the expression into:
\begin{align*}
V^{k+1} - V^k + \nu \norm{h(x^{k+1})}^2 + \rho \norm{r^{k+1} + (y^{k+1} - y^k)}^2 \leq 0,
\end{align*}
and equivalently,
\begin{align*}
V^{k+1} - V^k \leq -\nu \norm{h(x^{k+1})}^2 - \rho \norm{r^{k+1} + (y^{k+1} - y^k)}^2,
\end{align*}
which yields the following expression when developing the last term:
\begin{align} 
V^{k+1} - V^k \leq - \nu \norm{h(x^{k+1})}^2 - \rho \norm{r^{k+1}}^2 -\rho \norm{y^{k+1} - y^k}^2 - 2\rho r^{k+1} \cdot (y^{k+1} - y^k). \label{eq:almost_a_1} 
\end{align}
This is almost the inequality~\eqref{eq:a_1}, except for the last term $2\rho r^{k+1} \cdot (y^{k+1} - y^k)$ which is not present in~\eqref{eq:a_1}. To show~\eqref{eq:a_1} it is therefore sufficient to  that $\rho r^{k+1} \cdot (y^{k+1} - y^k) \geq 0$. To that end, we use the fact that $y^{k+1}$ minimizes $g(y) - \eta^{k+1} \cdot y$ and that $y^{k}$ minimizes $g(y) - \eta^{k} \cdot  y$ as it has been established in~\eqref{eq:g_intermediary}. One can therefore write the two following inequalities:
\begin{align*}
g(y^{k+1}) - \eta^{k+1} \cdot y^{k+1} &\leq g(y^{k}) - \eta^{k+1} \cdot y^{k}\\
g(y^{k}) - \eta^{k} \cdot y^{k} &\leq g(y^{k+1}) - \eta^{k} \cdot y^{k+1}.
\end{align*}
When summing these two inequalities, the cost terms involving $g$ cancel out resulting in:
\begin{align*}
- \eta^{k+1}\ . \ z_1^{k+1} - \eta^{k}\ . \ z_1^{k} &\leq - \eta^{k+1}\ . \ z_1^{k} - \eta^{k}\ . \ z_1^{k+1},
\end{align*}
which is equivalent to,
\begin{align*}
(\eta^{k+1} - \eta^{k}) \ . \ (z_1^{k+1} - z_1^{k}) \geq 0.
\end{align*}

Given the fact that $\eta^{k+1} - \eta^k = \rho r^{k+1}$, it implies that $\rho r^{k+1} \cdot (y^{k+1} - y^k) \geq 0$, which combined with~\eqref{eq:almost_a_1} therefore implies~\eqref{eq:a_1}:

\begin{align} 
V^{k+1} - V^k \leq - \nu \norm{h(x^{k+1})}^2 - \rho \norm{r^{k+1}}^2 -\rho \norm{y^{k+1} - y^k}^2.
\end{align} \hfill \qedsymbol
\end{proof}

\subsection{Distributed Augmented Lagrangian as a special case of Constrained ADMM} \label{sec:equivalence}

We now show that the D-AuLa algorithm is a particular case of the C-ADMM algorithm. To this end, we define the variable $z$ and the functions $c$, $h$ which aggregate the subproblems as follows,
\begin{align}  
  z &= (z_0, .., z_{N-1}), \\
  c(z) &= \sum_i c_i(z_i), \\
  h(z) &= (h_0(z_0), ..., h_{N-1}(z_{N-1})).
\end{align}
The cost functions are summed, while the optimization variable and the constraints functions are stacked ($c:\mathbb{R}^{n \times N} \rightarrow \mathbb{R}$, $z \in \mathbb{R}^{n \times N}$ and $h:\mathbb{R}^{n \times N} \rightarrow \mathbb{R}^{m \times N}$).
Let $\chi$ be the characteristic function of the subset of $\mathbb{R}^{n \times N}$ where $z$ is an aggregation of $N$~times the same coumpound, i.e:
\begin{align}  
  \chi(z) = \chi(z_0, ..., z_{N-1}) = \begin{cases}
  0, & \text{if } z_0 = z_1 = .. = z_{N-1}, \\
  + \infty, &\text{otherwise}.
\end{cases}
\end{align}
 
\noindent We apply the C-ADMM algorithm to the following problem:
\begin{align}  
  \min_{z, y} \quad & c(z) + \chi(y) \label{eq:equivalent_sequential},\\
  \text{s.t.} \quad & h(z) = 0, \nonumber \\
                    & z = y, \nonumber
\end{align} 
which we show is equivalent to the D-AuLa algorithm.


\begin{remark}
The pre-requisites for applying the C-ADMM algorithm are fulfilled: $c$ is differentiable and convex, as a sum of differentiable and convex functions. Similarly, $h$ is affine, as an aggregate of affine functions. In addition, it can be shown easily that $\chi$ is closed, proper and convex (convexity and properness are trivial, closeness comes from the fact the underlying set ${\{(z_0, ..., z_N) \in \mathbb{R}^{n \times N} \mid z_0 = ... = z_N}\}$ is itself closed).
\end{remark}


\begin{proposition} \label{th:theorem_chi_zero}
When applying the C-ADMM algorithm, $\chi(y^k) = 0$ after the first iteration ($k \geq 1$).
\end{proposition}

\begin{proof}
This follows from the definition of $y^{k+1}$ as the result of the Augmented Lagrangian minimization, $y^{k+1} = \min_{y} L_{\nu, \rho}(z^{k+1}, y, \kappa^k, \eta^k)$ which includes the function $\chi(y)$. If there would exists $y^{\star} = (y_0^{\star}, ..., y_{N-1}^{\star})$ minimizing $L_{\nu, \rho}(z^{k+1}, y, \kappa^k, \eta^k)$, such that $\chi(y^{\star}) = +\infty$, then $L_{\nu, \rho}(z^{k+1}, y^{\star}, \kappa^k, \eta^k) = +\infty$. The Augmented Lagrangian in $y_2^{\star} = (y_0^{\star}, .., y_0^{\star})$ built as a repetition of the first compounds of $y^{\star}$ would, however, be finite, since $\chi(y_2^{\star})=0$, and therefore lower than its value in $y^{\star}$, which contradicts the definition of $y^{\star}$.
\end{proof}

$\chi(y^k) = 0$ implies that we can define $y$ fully by its first $m$ coumpounds. One can therefore introduce the variable $\tilde{z} \in \mathbb{R}^n$ such that $y = (\tilde{z}, ..., \tilde{z})$. With this notation in place, the connection between the two algorithms can be established:

\begin{theorem} \label{th:theorem_eq}
Applying the D-AuLa algorithm on problem~\eqref{eq:distributed_aula_objective} is equivalent to applying the C-ADMM algorithm on~\eqref{eq:equivalent_sequential} with an initialization that satisfies $\chi(y_0) = 0$ and $\sum_i{\eta_i^0} = 0$.

\begin{alignat}{2}
&z^{k+1} := \min_{z} L_{\nu, \rho}(z, y^{k}, \kappa^k, \eta^k) \ \ &&\Longleftrightarrow\ \  {z^{k+1}_{i}} := \min_{z_i} L_{i, \nu, \rho}(z_i, \tilde{z}^k, \kappa_i^k, \eta_i^k) \label{th:eq_lagrangian_minimization} \\
&y^{k+1} := \min_{y} L_{\nu, \rho}(z^{k+1}, y, \kappa^k, \eta^k)\ \ &&\Longleftrightarrow\ \  \tilde{z}^{k+1} := \frac{1}{N}\sum_i{z_i^{k+1}} \label{th:eq_y_update} \\
&\kappa^{k+1} := \kappa^{k} + \nu h(z^{k+1})\ \  &&\Longleftrightarrow\ \  \kappa_i^{k+1} := \kappa_i^{k} + \nu h_i(z_i^{k+1}) \label{th:eq_kappa_update}\\
&\eta^{k+1} := \eta^{k} + \rho(z^{k+1} - y^{k+1})\ \  &&\Longleftrightarrow\ \  \eta_i^{k+1} := \eta_i^{k} + \rho (z_i^{k+1} - \tilde{z}^{k+1}) \label{th:eq_eta_update}
\end{alignat}
\end{theorem}

\begin{proof}
\item
The equivalence of the Lagrange multiplier updates (\ref{th:eq_kappa_update}) and (\ref{th:eq_eta_update}) is trivial, and results directly from the definitions of $\kappa_i$, $\eta_i$ and $\tilde{z}$. The next two paragraphs establish equivalence for (\ref{th:eq_lagrangian_minimization}) and \added{(\ref{th:eq_y_update})}\deleted{(\ref{th:eq_eta_update})}.
\paragraph{Augmented Lagrangian minimization (\ref{th:eq_lagrangian_minimization}):} Since $\chi(y^k) = 0$, the C-ADMM Augmented Lagrangian minimization (\ref{th:eq_lagrangian_minimization}) step can be simplified to:
\begin{align*}
z^{k+1} := \min_z c(z) + \kappa^k \cdot h(z) + \frac{\nu}{2} \norm{h(z)}^2 + \eta^k \cdot (z - y^k) + \frac{\rho}{2} \norm{z - y^k}^2.
\end{align*}
Using the definitions of $c$, $h$, $z$ one can fully separate this expression as a sum of $N$~terms:
\begin{align*}
(z_0, .., z_{N-1})^{k+1} := \min_{z_0, .., z_{N-1}} \sum_i {c_i(z_i) + \kappa_i^k \cdot h_i(z_i) + \frac{\nu}{2} \norm{h_i(z_i)}^2 + \eta_i^k \cdot (z_i - y^k) + \frac{\rho}{2} \norm{z_i - y^k}^2}.
\end{align*}
The terms of the sum are independent from each other, since each depends only on $z_i$, the other variables are constants throughout the minimization. Consequently, each term can be minimized independently, such that one can rewrite it as:
\begin{align*}
z_i^{k+1} &:= \min_{z_i} { \left( c_i(z_i) + \kappa_i^k \cdot h_i(z_i) + \frac{\nu}{2} \norm{h_i(z_i)}^2 + \eta_i^k \cdot (z_i - y^k) + \frac{\rho}{2} \norm{z_i - y^k}^2 \right) },
\end{align*}
which are the $N$ Augmented Lagragian minimizations in D-AuLa, see step~\eqref{eq:dec-algo-newton}:
\begin{align*}
z_i^{k+1} &:= \min_{z_i} L_{i, \nu, \rho}(z_i, y^k, \kappa^k, \eta^k).
\end{align*}

\paragraph{Consensus variable update~\eqref{th:eq_y_update}:}
Step~\eqref{th:eq_y_update} of the C-ADMM algorithm is a minimization over $y$ of the Augmented Lagrangian:
\begin{align*}
y^{k+1} := \min_y \left( f(z^{k+1}) + \chi(y) + \kappa^k \cdot h(z^{k+1}) + \frac{\nu}{2} \norm{h(z^{k+1})}^2 + \eta^k \cdot (z^{k+1} - y) + \frac{\rho}{2} \norm{z^{k+1} - y}^2 \right).
\end{align*}
One can remove $\chi(y)$ (since it is zero), as well as the terms that are not depending on $y$, i.e. $f(z^{k+1})$, $h(z^{k+1})$ and $\eta^k \cdot z^{k+1}$:
\begin{align*}
y^{k+1} := \min_y \left(- \eta^k \cdot y + \frac{\rho}{2} \norm{z^{k+1} - y}^2 \right).
\end{align*}

We rewrite this expression as a minimization over $\tilde{z}$, by using the unstacked variables $\eta^k_i$, $\tilde{z}$, $z_i^{k+1}$ and developing the dot product and norm expression, yielding:

\begin{align*}
\tilde{z}^{k+1} := \min_{\tilde{z}} \left( - \sum_i{\eta_i^k \cdot \tilde{z}} + \frac{\rho}{2} \sum_i \norm{z_i^{k+1} - \tilde{z}}^2 \right).
\end{align*}
Since the above expression is convex with respect to $\tilde{z}$, a necessary and sufficient condition for $\tilde{z}^{k+1}$ to be a minimum is that the gradient is zero in $\tilde{z}^{k+1}$:
\begin{align}
- \sum_i{\eta_i^k} + \frac{\rho}{2} \sum_i 2 (z_i^{k+1} - \tilde{z}^{k+1}) = 0, \label{eq:optimality_condition_second_minimization}
\end{align}
which can be rewritten by moving $\tilde{z}^{k+1}$ out of the sum, and gathering the terms $z_i^{k+1}$ and $\eta_i^k$ together:
\begin{align*}
- \rho N \tilde{z}^{k+1} + \rho \sum_i (z_i^{k+1} - \frac{\eta_i^{k}}{\rho}) = 0,
\end{align*}
which leads to:
\begin{align}
\tilde{z}^{k+1} = \frac{1}{N} \sum_i (z_i^{k+1} - \frac{\eta_i^{k}}{\rho}).
\end{align}

This expression can be simplified further by using the fact that $\sum_i{\eta_i^k} = 0$ after the first iteration. This is proven by injecting the update rule $\eta_i^{k+1} = \eta_i^k + \rho (z_i^{k+1} - 
\tilde{z}^{k+1})$ in~\eqref{eq:optimality_condition_second_minimization}. The update rule therefore becomes,
\begin{align*}
\tilde{z}^{k+1} = \frac{1}{N} \sum_i z_i^{k+1},
\end{align*}
which is the consensus variable update step of D-AuLa~\eqref{eq:dec-z_update}. \hfill \qedsymbol
\end{proof}


\subsection{Non-Convex Problems} \label{sec:nonconvex}
The proof relies on the strong duality of the optimization problem, which guarantees the existence of a \textit{global} saddle point of the Lagrangian. Concretely, this is used in the proof for ensuring the validity of Eq.~\eqref{eq:need_saddle_point}. In addition, convexity of the cost and constraints are needed when integrating Eq.~\eqref{eq:needs_convexity_0} leading to Eq.~\eqref{eq:needs_convexity}.

If the cost are not convex or the equality constraints are not affine, strong duality may not hold, and the guarantee to converge to a \textit{global} optimum is lost. The algorithm is not guaranteed to converge. When it does converge, it may only reach a local optimum. In practice, the algorithm typically converges to a stationary point of the Lagrangian satisfying the Karush-Kuhn-Tucker (KKT) conditions. 
In the non-convex case, D-AuLa is therefore a local method whose behavior depends on the initialization and penalty parameters, as in the Augmented Lagrangian (AuLa) and Alternating Direction Method of Multipliers (ADMM) methods.

%




\subsection{Extension to the Case with Inequality Constraints} \label{sec:inequality_constraints}
Extending the ADMM proof provided in \citep{ADMM} with additional inequality constraints is more difficult than its extension to equality constraints. Indeed, the proof is articulated around the value,
\begin{equation*}
V^k = (1/\nu)\norm{\kappa^k-\kappa^{\star}}^2
+ (1/\rho)\norm{\eta^k-\eta^{\star}}^2
+ \rho \norm{y^k - y^{\star}}^2,
\end{equation*}
which is monotonically decreasing. Extending this expression to inequality constraints, using the same scheme, would result in,
\begin{equation*}
V_{ineq}^k = (1/\mu)\norm{\lambda^k-\lambda^{\star}}^2 
+ (1/\nu)\norm{\kappa^k-\kappa^{\star}}^2
+ (1/\rho)\norm{\eta^k-\eta^{\star}}^2
+ \rho \norm{y^k - y^{\star}}^2,
\end{equation*}
which is not necessarily monotonically decreasing, depending on how the constraints activities evolves. This is due to the fact that the square penalty in the Augmented Lagrangian which is,
\begin{equation*}
\frac{\mu}{2} \norm{[g(x) > 0] \odot g(x)}^2,
\end{equation*}
depends on the constraint's activity. Furthermore, the update rule for $\lambda$ which is,
\begin{equation*}
\lambda^{k+1} := \max(0, \lambda^{k} + \mu g(x^{k+1})),
\end{equation*}
has two different cases corresponding to the $\max$ operation. 
Defining a monotonically decreasing $V^k_{ineq}$ is therefore not straightforward, since this property has to be kept across eventual case switches.

The proof can however be extended under the additional assumption that the activity of the inequality constraints stabilizes after a finite number of iterations.
This assumption is commonly satisfied in practice, particularly in trajectory optimization, where it is observed that the constraints activity changes predominantly during the initial iterations.


%

\subsection{Extension to the Case without Full Overlap of the Subproblems}\label{sec:no_full_overlap}

The proof can be extended to the case without full overlap of the subproblems. It involves generalizing the relation between the two variables $x$ and $y$  of C-ADMM~\eqref{eq:constrained_admm_consensus_constraint} to enforce equality only on specific parts. It can be achieved by replacing the constraint $x=y$ with a more general formulation $Ax + By = C$. The convergence proof provided in \citep{ADMM} uses this generalized constraint.

\section{TAMP Policies}
\label{sec:appendix_tamp_policies}
This appendix provides example of planned policies for the problems Baxter-C and Franka-C$\times$A' of the experiments of Section~\ref{sec:tamp_experiments}.
\subsection{Planned Policy for Baxter-C}
\begin{figure}[H]
    \centering
       \includegraphics[width=1.0\linewidth]{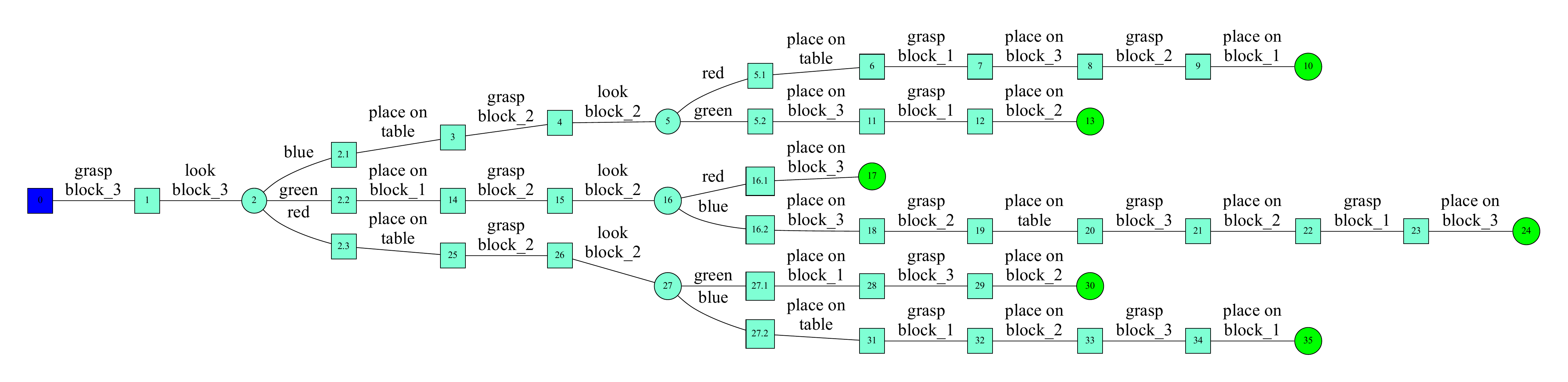}
  \caption{Policy for problem Baxter-C. With 3 unknown blocks, the robot must observe 2 times, resulting in 6 possible contingencies.}
  \label{fig:policy_baxter_c} 
\end{figure}
\subsection{Planned Policy for Franka-C$\times$A'}

\begin{figure}[ht]
    \centering
       \includegraphics[width=1.0\linewidth]{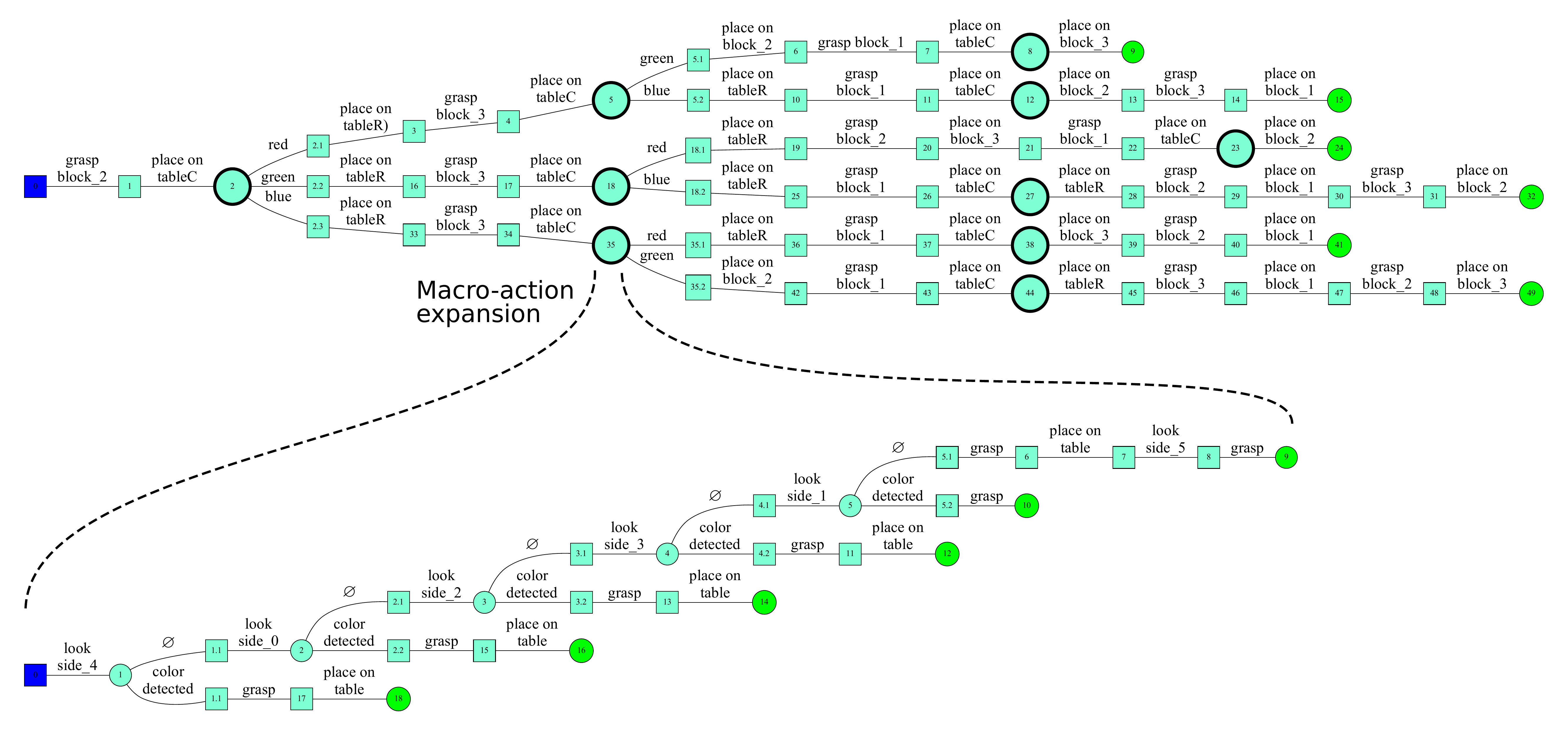}
  \caption{Policy for problem Franka-C$\times$A'. The high level policy on the top of the image contains macro-actions (indicated with a bold circle). The macro-actions expand into low level exploration policies with 5 branching points.}
  \label{fig:policy_franka_ca} 
\end{figure}

}



\end{document}